\documentclass[12pt]{article}
\usepackage{fullpage}

\date{}

\title{\bfseries Combinatorial Bandits with Relative Feedback}

\author{
Aadirupa Saha\thanks{Indian Institute of Science, Bangalore, India. {\tt aadirupa@iisc.ac.in}}, \and Aditya Gopalan \thanks{Indian Institute of Science, Bangalore, India. {\tt aditya@iisc.ac.in} }
}

\usepackage{microtype}
\usepackage{cancel}
\usepackage{graphicx}
\usepackage{booktabs} 
\usepackage{amssymb}
\usepackage{color}
\usepackage{amsmath}
\usepackage{amsthm}
\usepackage{thmtools}
\usepackage{thm-restate}
\usepackage{algorithm}
\usepackage{algorithmic}

\usepackage{natbib}
\usepackage{hyperref}
\usepackage{nameref}

\allowdisplaybreaks

\newtheorem{thm}{Theorem}
\newtheorem{lem}[thm]{Lemma}

\newtheorem{defn}[thm]{Definition}

\newtheorem{rem}{Remark}

\newcommand{\R}{{\mathbb R}}

\newcommand{\N}{{\mathbb N}}
\renewcommand{\P}{{\mathbf P}}
\newcommand{\bhP}{\mathbf {\hat P}}
\newcommand{\W}{{\mathbf W}}
\newcommand{\E}{{\mathbf E}}

\newcommand{\1}{{\mathbf 1}}

\newcommand{\cA}{{\mathcal A}}

\newcommand{\cB}{{\mathcal B}}

\newcommand{\cC}{{\mathcal C}}

\newcommand{\cF}{{\mathcal F}}
\newcommand{\cG}{{\mathcal G}}

\newcommand{\bN}{{\mathbf N}}

\newcommand{\X}{{\mathcal X}}

\newcommand{\sS}{{S_{(k)}}}

\newcommand{\hp}{{\hat p}}
\newcommand{\p}{{\mathbf p}}
\newcommand{\q}{{\mathbf q}}

\newcommand{\sm}{\setminus}
\newcommand{\objbest}{{Winner-regret}}
\newcommand{\objk}{{Top-$k$-regret}}
\def \algmm{{\it MaxMin-UCB}}
\def \algbld{{\it build\_S}}
\def \algkmm{{\it Rec-MaxMin-UCB}}
\def \mnl{{MNL($n,\btheta$)}}
\def \wf{{ Winner Feedback}}
\def \tf{{ Top-$m$-ranking Feedback}}
\def \tk{{ Top-$k$-ranking Feedback}}

\def \rnd{{\bf Random-Exploration}}
\def \prog{{\bf Progress}}
\def \sat{{\bf Saturation}}

\def \bi{{\it Best-Item}}
\def \bk{{\it Top-$k$ Best-Items}}

\def \nr{{\it No-regret}}
\def \rb{{\it Rank-Breaking}}

\newcommand{\ceil}[1]{\Big \lceil{#1} \Big\rceil}

\newcommand{\btheta}{\boldsymbol \theta}
\newcommand{\bSigma}{\boldsymbol \Sigma}

\newcommand{\bnu}{{\boldsymbol \nu}}
\newcommand{\bmu}{{\boldsymbol \mu}}

\newcommand{\bsigma}{\boldsymbol \sigma}

\newcommand{\red}[1]{\textcolor{red}{#1}}

\begin{document}

\maketitle

\vspace{-10pt}
\begin{abstract}
\vspace{-10pt}
We consider combinatorial online learning with subset choices when only relative feedback information from subsets is available, instead of bandit or semi-bandit feedback which is absolute. Specifically, we study two regret minimisation problems over subsets of a finite ground set $[n]$, with subset-wise relative preference information feedback according to the Multinomial logit choice model. In the first setting, the learner can play subsets of size bounded by a maximum size and receives top-$m$ rank-ordered feedback, while in the second setting the learner can play subsets of a fixed size $k$ with a full subset ranking observed as feedback. For both settings, we devise instance-dependent and order-optimal regret algorithms with regret $O(\frac{n}{m} \ln T)$ and $O(\frac{n}{k} \ln T)$, respectively. We derive fundamental limits on the regret performance of online learning with subset-wise preferences, proving the tightness of our regret guarantees. Our results also show the value of eliciting more general top-$m$ rank-ordered feedback over single winner feedback ($m=1$). Our theoretical results are corroborated with empirical evaluations.
\end{abstract}

\vspace*{-10pt}
\section{Introduction}
\label{sec:intro}
\vspace{-10pt}
Online learning over subsets with absolute or {\em cardinal} utility feedback is well-understood in terms of  statistically efficient algorithms for bandits or semi-bandits with large, combinatorial subset action spaces \citep{chen2013combinatorial,kveton2015tight}. In such settings the learner aims to find the subset with highest value, and upon testing a subset observes either noisy rewards from its constituents or an aggregate reward. In many natural settings, however, information obtained about the utilities of alternatives chosen is inherently relative or {\em ordinal}, e.g., recommender systems \citep{hofmann2013fast, radlinski2008does}, crowdsourcing \citep{chen2013pairwise}, multi-player game ranking \citep{graepel2006ranking}, market research and social surveys \citep{ben1994combining, alwin1985measurement, hensher1994stated}, and in other systems where humans are often more inclined to express comparative preferences.

The framework of dueling bandits \citep{Yue+09, Zoghi+13} represents a promising attempt to model online optimisation with {\em pairwise} preference feedback. However, our understanding of the more general and realistic online learning setting of combinatorial subset choices and {\em subset-wise} feedback is relatively less developed than the case of observing absolute, subset-independent reward information.


In this work, we consider a generalisation of the dueling bandit problem where the learner, instead of choosing only two arms, selects a subset of (up to) $k \geq 2$ many arms in each round. The learner subsequently observes as feedback a rank-ordered list of $m \geq 1$ items from the subset, generated probabilistically according to an underlying subset-wise preference model -- in this work the Plackett-Luce distribution on rankings based on the multinomial logit (MNL) choice model \citep{Az+12} -- in which each arm has an unknown positive value. Simultaneously, the learner earns as reward the average value of the subset played in the round. The goal of the learner is to play subsets to minimise its cumulative regret with respect to the subset with highest value. 

Achieving low regret with subset-wise preference feedback is relevant in settings where deviating from choosing an optimal subset of alternatives comes with a cost (driven by considerations like revenue) even during the learning phase, but where the feedback information provides purely relative feedback. For instance, consider a beverage company that experimentally develops several variants of a drink (arms or alternatives), a best-selling subset of which it wants to learn to put up in the open market by trial and error. Each time a subset of items is put up, in parallel the company  elicits relative preference feedback about the subset from, say, a team of expert tasters or through crowdsourcing. The  value of a subset can be modelled as the average value of items in it, which is however not directly observable, it being function of the open market response to the offered subset. The challenge thus lies in optimizing the subset selection over time by observing only relative preferences (made precise by the notion of \objk, Section \ref{subsec:regretdefns}).

A challenging feature of this problem, with subset plays and relative feedback, is the combinatorially large action and feedback space, much like those in combinatorial bandits   \citep{cesa12,combes15}. The key question here is whether (and if so, how) structure in the subset choice model -- defined compactly by only a few parameters (as many as the number of arms) -- can be exploited to give algorithms whose regret does not explode combinatorially. 
The contributions of this paper are:

%

\textbf{(1).} We consider the problem of regret minimisation when subsets of items $\{1, \ldots, n\}$ of size at most $k$ can be played, top $m \leq k$ rank-ordered feedback is received according to the MNL model, and the value of a subset is the mean MNL-parameter value of the items in the subset. We propose an upper confidence bound (UCB)-based algorithm, with a new max-min subset-building rule and a lightweight space requirement of tracking $O(n^2)$ pairwise item estimates, showing that it enjoys instance-dependent regret in $T$ rounds of $O(\frac{n}{m} \ln T)$. This is shown to be order-optimal by exhibiting a lower bound of $\Omega(\frac{n}{m} \ln T)$ on the regret for any \nr\ algorithm. Our results imply that the optimal regret does not vary with the maximum subset size ($k$) that can be played, but improves multiplicatively with the length of top $m$-rank-ordered feedback received per round (Sec. \ref{sec:res_bi}). 
%
%
	
\textbf{(2).} We consider a related regret minimisation setting in which subsets of size exactly $k$ must be played, after which a ranking of the $k$ items is received as feedback, and where the zero-regret subset consists of the $k$ items with the highest MNL-parameter values. In this case, our analysis reveals a fundamental lower bound on regret of $\Omega(\frac{n-k}{k\Delta_{(k)}} \ln T)$, where the problem complexity now depends on the parameter difference between the $k^{th}$ and $(k+1)^{th}$ best item of the MNL model. We follow this up with a subset-playing algorithm (Alg. \ref{alg:kmm}) for this problem -- a recursive variant of the earlier UCB-based algorithm -- with a matching, optimal regret guarantee of $O\Big(\frac{(n-k)\ln T}{k\Delta_{(k)}}\Big)$ (Sec. \ref{sec:res_tk}).

We also provide extensive numerical evaluations supporting our theoretical findings. Due to space constraints, a discussion on  related work appears in the Appendix.

\vspace*{-10pt}
\section{Preliminaries and Problem Statement}
\label{sec:prelims}
\vspace*{-10pt}
{\bf Notation.} We denote by $[n]$ the set $\{1,2,...,n\}$. For any subset $S \subseteq [n]$, we let $|S|$ denote the cardinality of $S$. 
When there is no confusion about the context, we often represent (an unordered) subset $S$ as a vector (or ordered subset) $S$ of size $|S|$ according to, say, a fixed global ordering of all the items $[n]$. In this case, $S(i)$ denotes the item (member) at the $i$th position in subset $S$.   
For any ordered set $S$, $S(i:j)$ denotes the set of items from position $i$ to $j$, $i<j$, $\forall i,j \in [|S|]$.
$\bSigma_S = \{\sigma \mid \sigma$ is a permutation over items of $ S\}$, where for any permutation $\sigma \in \Sigma_{S}$, $\sigma(i)$ denotes the element at the $i$-{th} position in $\sigma, i \in [|S|]$. We also denote by $\bSigma_S^m$ the set of permutations of any $m$-subset of $S$, for any $m \in [k]$, i.e. $\Sigma_S^m := \{ \Sigma_{S'} \mid S' \subseteq S, \, |S'| = m \}$. 
$\1(\varphi)$ is generically used to denote an indicator variable that takes the value $1$ if the predicate $\varphi$ is true, and $0$ otherwise. 
$Pr(A)$ is used to denote the probability of event $A$, in a probability space that is clear from the context.

\begin{defn}[Multinomial logit probability model]
\label{def:mnl_thet}
A Multinomial logit (MNL) probability model \mnl, specified by positive parameters $(\theta_1, \ldots, \theta_n)$, is a collection of probability distributions $\{Pr(\cdot | S): S \subset [n], S \neq \emptyset \}$, where for each non-empty subset $S \subseteq [n]$, $Pr(i| S) = \frac{\theta_i \1(i \in S)}{\sum_{j \in S} \theta_j}$ $\forall 1 \leq i \leq n$. The indices $1, \ldots, n$ are referred to as `items'  or `arms' .

%
\end{defn}
\vspace*{-7pt}
\noindent \textbf{(i). \bi:} Given an \mnl\, instance, we define the \bi\, $a^* \in [n]$, to be the item with highest MNL parameter if such a unique item exists, i.e. $a^* := \arg\max_{i \in [n]} \theta_i$. 

\noindent \textbf{(ii). \bk:} Given any instance of \mnl\, we define the \bk\, $\sS \subseteq [n]$, to be the set of $k$ distinct items with highest MNL parameters if such a unique set exists, i.e. for any pair of items $i \in \sS$ and $j \in [n] \sm \sS$, $\theta_i > \theta_j$, such that $|\sS| = k$. For this problem, we assume $\theta_1 \ge \theta_2 \ge \ldots \theta_k > \theta_{k+1} \ge \ldots \ge \theta_n$, implying $\sS = [k]$. We also denote $\Delta_{(k)} = \theta_k - \theta_{k+1}$.\\
\vspace*{-20pt}
 


\vspace{-0pt}
\subsection{Feedback models}
\label{sec:feed_mod}
\vspace{-5pt}
An {\em online learning algorithm} interacts with a \mnl\, probability model over $n$ items  as follows. At each round $t = 1, 2, \ldots$, the algorithm plays a subset $S_t \subseteq [n]$ of (distinct) items, with $|S_t| \leq k$, upon which it receives stochastic feedback defined as: 


1. \textbf{\wf:} 
In this case, the environment returns a single item $J$ drawn independently from probability distribution $Pr(\cdot | S)$, i.e., $Pr(J = j|S) = \frac{{\theta_j}}{\sum_{\ell \in S} \theta_\ell} \, \forall j \in S$.

2. \textbf{\tf\, ($ 1 \leq m \leq k-1$):} Here, the environment returns an ordered list of $m$ items sampled without replacement from the \mnl\, probability model on $S$. More formally, the environment returns a partial ranking $\bsigma \in \bSigma_{S}^m$, drawn from the probability distribution
$
\label{eq:prob_rnk1}
Pr(\bsigma = \sigma|S) = \prod_{i = 1}^{m}\frac{{\theta_{\sigma^{-1}(i)}}}{\sum_{j \in S\sm\sigma^{-1}(1:i-1)}\theta_{j}}, \; \sigma \in \bSigma_S^m.
$ 
This can also be seen as picking an item $\bsigma^{-1}(1) \in S$ according to {\it{\wf}} from $S$, then picking $\bsigma^{-1}(2)$ from $S \setminus \{\bsigma^{-1}(1)\}$, and so on, until all elements from $S$ are exhausted. When $m = 1$, \tf\, is the same as \wf.
To incorporate sets with $|S|< m$, we set $m = \min(|S|,m)$.
Clearly this model reduces to {\wf} for $m = 1$, and a full rank ordering of the set $S$ when $m=|S|-1$.

\vspace{-5pt}
\subsection{Decisions (Subsets) and Regret}
\label{subsec:regretdefns}
\vspace{-5pt} 
We consider two regret minimisation problems in terms of their decision spaces and notions of regret:


\textbf{(1). \objbest:} 
This is motivated by learning to identify the \bi\, $a^*$. At any round $t$, the learner can play sets of size $1,\ldots, k$,  
but is penalised for playing any item other than $a^*$. 
Formally, we define the learner's instantaneous regret at round $t$ as $r_t^1 = \sum_{i \in S_t}\frac{(\theta_{a^*} - \theta_i)}{|S_t|}$, and %
its cumulative regret from $T$ rounds  as
$ 
R_T^1 = \sum_{t = 1}^{T} r_t^1 = \sum_{t=1}^T\Big(\sum_{i \in S_t}\frac{(\theta_{a^*} - \theta_i)}{|S_t|}\Big),
$ 

The learner aims to play sets $S_t$ to keep the regret as {\em low} as possible, i.e., to play only the singleton set $S_t = \{a^*\}$ over time, as that is the only set with $0$ regret.
The instantaneous \objbest\, can be interpreted as a shortfall in {value} of the played set $S_t$ with respect to $\{a^*\}$, where the value of a set $S$ is simply the mean parameter value $\frac{\sum_{i \in S}\theta_i}{|S|}$ of its items.

\begin{rem}
Assuming $\theta_{a^*} = 1$ (we can do this without loss of generality since the MNL model is positive scale invariant, see Defn. \ref{def:mnl_thet}), it is easy to note that for any item $i \in [n]\sm\{a^*\}$ $p_{a^*i} := Pr(a^*|\{a^*,i\}) = \frac{\theta_{a^*}}{\theta_{a^*} + \theta_i} \ge \frac{1}{2} + \frac{\theta_{a^*} - \theta_i}{4}$ (as $\theta_i < \theta_{a^*}, \, \forall i$). Consequently, the \objbest\, as defined above, can be further bounded above (up to constant factors) as
$ 
\tilde R_T^1 = \sum_{t=1}^T\sum_{i \in S_t}\frac{(p_{a^*i} - \frac{1}{2})}{|S_t|},
$ 
which, for $k=2$, is standard dueling bandit regret \citep{Yue+12,Zoghi+14RUCB,DTS}.
\end{rem}

%
%

\begin{rem}
An alternative notion of instantaneous regret is the shortfall in the preference probability of the best item $a^*$ in the selected set $S_t$, i.e., $
\tilde r_t^1 = \sum_{i \in S_t}\Big( Pr(a^*|S_t \cup \{a^*\}) - Pr(i|S_t \cup\{a^*\})\Big)$.
However, if all the MNL parameters are bounded, i.e., $\theta_i \in [a,b], \, \forall i \in [n]$, then 
$
\frac{1}{b}\bigg(\sum_{i \in S_t} \frac{(\theta_{a^*} - \theta_i)}{|S_t|+1}\bigg) \le \tilde r_t^1 \le \frac{1}{a}\bigg(\sum_{i \in S_t} \frac{(\theta_{a^*} - \theta_i)}{|S_t|+1}\bigg)$,
implying that these two notions of regret, $r_t^1$ and $\tilde r_t^1$, are only constant factors apart.
\end{rem}

\textbf{(2). \objk:} 
This setting is motivated by learning to identify the set of \bk\, $\sS$ of the \mnl\, model. Correspondingly, we assume that the learner can play sets of $k$ distinct items at each round $t \in [T]$. The instantaneous regret of the learner, in this case, in the $t$-th round is defined to be $r_t^k = \left(\frac{\theta_{\sS} - \sum_{i \in S_t}\theta_i}{k}\right)$, where $\theta_{\sS} = {\sum_{i \in \sS}\theta_i }$. Consequently, the cumulative regret of the learner at the end of round $T$ becomes
$ 
\label{eq:reg_k}
R_T^k = \sum_{t = 1}^{T} r_t^k = \sum_{t=1}^T\left(\frac{\theta_{\sS} - \sum_{i \in S_t}\theta_i}{k}\right).
$ 
As with the \objbest, the \objk\, also admits a natural interpretation as the shortfall in {\em value} of the set $S_t$ with respect to the set $\sS$, with value of a set being the mean $\theta$ parameter of its arms.

\vspace{-10pt}
\section{Minimising \objbest}
\label{sec:res_bi}
\vspace{-7pt}
We first consider the problem of minimising \objbest. We start by analysing a regret lower bound for the problem, followed by designing an optimal algorithm with matching upper bound. 
\vspace{-10pt}

\subsection{Fundamental lower bound on \objbest}
\label{sec:lb_bi}
\vspace{-5pt}
Along the lines of \cite{lai1985asymptotically}, we define the following consistency property of any reasonable online  learning algorithm in order to state a fundamental lower bound on regret performance.

\begin{defn}[\nr\, algorithm]
\label{def:con}
An online learning algorithm $\cA$ is defined to be a \nr\, algorithm for \objbest\, if for each problem instance \mnl\,, the expected number of times $\cA$ plays any suboptimal set $S \subseteq [n]$ is sublinear in $T$, i.e.,  $\forall S \neq \arg\max_{i} \theta_i:$ $\E_{\btheta}[N_S(T)] = o(T^\alpha)$, for some $\alpha \in [0,1]$ (potentially depending on $\btheta$), where $N_S(T) :=  \sum_{t=1}^T \1(S_t = S)$ is the number plays of set $S$ in $T$ rounds. $\E_{\btheta}[\cdot]$ denotes expectation under the algorithm and \mnl\, model.
\end{defn}


\begin{restatable}[\objbest\, Lower Bound]{thm}{lbwiwf}
\label{thm:lb_wiwf}
For any \nr\, learning algorithm $\cA$ for \objbest\, that uses \wf, and for any problem instance \mnl\, s.t. $a^* = \underset{i \in [n]}{\arg\max}\theta_i$, the expected regret incurred by $\cA$ satisfies
$
\underset{T \to \infty }{\lim \inf}\, \E_{\btheta}\Big[\frac{R_T^{1}(\cA)}{\ln T}\Big] \ge \frac{\theta_{a^*}}{\Big(\underset{i \in [n]\sm\{a^*\}}{\min}\frac{\theta_{a^*}}{\theta_i} - 1\Big)}(n-1)
$. 
\end{restatable}
\vspace{-10pt}
\textbf{Note:} This is a problem-dependent lower bound with $\theta_{a^*}{\Big(\underset{i \in [n]\sm\{a^*\}}{\min}\frac{\theta_{a^*}}{\theta_i} - 1\Big)^{-1}}$ denoting a complexity or hardness term (`gap') for regret performance under any \emph{`reasonable learning algorithm'}.

\begin{rem}
The result suggests the regret rate with only \wf\, cannot improve with $k$, uniformly across all problem instances. Rather strikingly, there is no reduction in hardness (measured in terms of regret rate) in learning the \bi\, using \wf\, from large ($k$-size) subsets as compared to using pairwise (dueling) feedback ($k = 2$). It could be tempting to expect an improved learning rate with subset-wise feedback as the number of items being tested per iteration is more ($k \ge 2$), so information-theoretically one may expect to `learn more' about the underlying model per subset query. On the contrary, it turns out that it is intuitively `harder' for a good (i.e., near-optimal) item to prove its competitiveness in just a single winner draw against a large population of its $k-1$ other competitors, as compared to winning over just a single competitor for $k=2$ case. 
\end{rem}

\noindent \textbf{Proof sketch.}
The proof of the result is based on the change of measure technique for bandit regret lower bounds presented by, say,  \citet{Garivier+16}, that uses the information divergence between two nearby instances \mnl\, (the original instance) and MNL$(n, \btheta')$ (an alternative instance) to quantify the hardness of learning the best arm in either environment. %
In our case, each bandit instance corresponds to an instance of the \mnl\, problem with the arm set containing all subsets of $[n]$ of size upto $k$: $\cA = \{S \subseteq [n] ~|~ |S| \in [k]\}$. The key of the proof relies on carefully crafting a true instance, with optimal arm $a^* = 1$, and a family of `slightly perturbed' alternative instances $\{\bnu^a: a \neq 1\}$, each with optimal arm $a \neq 1$, chosen as: $
\textbf{(1). True Instance: } \text{MNL}(n,\btheta^1): \theta_1^1 > \theta_2^1 = \ldots = \theta_n^1 = \theta  ~~(\text{for some } \theta \in \R_+),
$ 
, and for each suboptimal item $a \in [n]\setminus \{1\}$, the $
\textbf{(2). Altered instances: }    \text{MNL}(n,\btheta^a): \theta_a^a = \theta_1^1 + \epsilon = \theta + (\Lambda + \epsilon); ~\theta_i^a = \theta_i^1, ~~\forall i \in [n]\sm \{a\}
$ 
for some $\epsilon > 0$.
The result of Thm. \ref{thm:lb_wiwf} now follows by applying Lemma \ref{lem:gar16} on pairs of problem instances $(\nu, \nu'^{(a)})$ with suitable upper bounds on the divergences. %
(Complete proof given in Appendix \ref{app:lb_wiwf}).
$\hfill \square$

\textbf{Note:} We also show an alternate version of the regret lower bound of $\Omega\Big( \frac{n}{\big(\underset{i \in [n]\sm\{a^*\}}{\min} p_{a^*,i} - 0.5 \big)}\ln T \Big)$ in terms of pairwise preference-based instance complexities (details are moved to Appendix \ref{app:alt_wilb}).

\noindent {\bf Improved regret lower bound with \tf.} In contrast to the situation with only winner feedback, the following (more general) result shows a reduced lower bound when \tf\, is available in each play of a subset, opening up the possibility of improved learning (regret) performance when ranked-order feedback is available.

\begin{restatable}[Regret Lower Bound: \objbest\, with \tf]{thm}{lbwitf}
\label{thm:lb_witf}
For any \nr\, algorithm $\cA$ for the  \objbest\, problem with \tf, there exists a problem instance \mnl\, such that the expected \objbest\,  incurred by $\cA$ satisfies
$ 
\underset{T \to \infty }{\lim \inf}\,\E_{\btheta}\Big[\frac{R_T^{1}(\cA)}{\ln T}\Big] \ge \frac{\theta_{a^*}}{\Big(\underset{i \in [n]\sm\{a^*\}}{\min}\frac{\theta_{a^*}}{\theta_i} - 1\Big)}\frac{(n-1)}{m}, 
$ 
where as in Thm. \ref{thm:lb_wiwf}, $\E_{\btheta}[\cdot]$ denotes expectation under the algorithm and the MNL model \mnl, and recall $a^* := \arg\max_{i \in [n]}\theta_i$.
\end{restatable}

\noindent \textbf{Proof sketch.} 
The main observation made here is that the KL divergences for \tf\, are $m$ times compared to the case of \wf, which we show using chain rule for KL divergences \citep{cover2012elements}: 
$
 KL(p^1_S, p^a_S) = KL(p^1_S(\sigma_1), p^a_S(\sigma_1)) +  \sum_{i=2}^{m}KL(p^1_S(\sigma_i \mid \sigma(1:i-1)), p^a_S(\sigma_i \mid \sigma(1:i-1)))$,
where $\sigma_i = \sigma(i)$ and $KL( P(Y \mid X),Q(Y \mid X)): = \sum_{x}Pr\Big( X = x\Big)\big[ KL( P(Y \mid X = x),Q(Y \mid X = x))\big]$ denotes the conditional KL-divergence. 
Using this, along with the upper bound on KL divergences for \wf\, (derived for Thm. \ref{thm:lb_wiwf}), we show that $KL(p^1_S,p^a_S) \le \frac{ m\Delta_a'^2}{\theta_S^1(\theta_1^1 + \epsilon)}, \, \forall a \in [n]\sm\{1\}$ \big(where $\theta_S = \sum_{i \in S}\theta_i$ and $\Delta_a' = O(\theta_{a^*} - \theta_a)$\big), which precisely gives the $\frac{1}{m}$-factor reduction in the lower bound compared to \wf. The bound now can be derived following a similar technique used for Thm. \ref{thm:lb_wiwf} (details in Appendix \ref{app:lb_witf}).
$\hfill \square$.

\begin{rem}
\label{rem:lb_witf}
Thm. \ref{thm:lb_witf} shows a $\Omega\big(\frac{n\ln T}{m}\big)$ lower bound on regret, containing the instance-dependent constant term $\frac{\theta_{a^*}}{\Big(\underset{i \in [n]\sm\{a^*\}}{\min}\frac{\theta_{a^*}}{\theta_i} - 1\Big)}$ which exposes the hardness of the regret minimisation problem in terms of the {`gap'} between the best $a^*$ and the second best item $\min_{i \in [n]\sm\{a^*\}}\theta_i$: $(\theta_{a^*}-\max_{i \in [n]\sm\{a^*\}}\theta_j)$. The $\frac{1}{m}$ factor improvement in learning rate with \tf\, can be intuitively interpreted as follows: revealing an $m$-ranking of a $k$-set is worth about $\ln \left({k \choose m} m!\right) = O(m \ln k)$ bits of information, which is about $m$ times as large compared to revealing a single winner. 
\end{rem}

\vspace{-10pt}
\subsection{An order-optimal algorithm for \objbest}
\label{sec:alg_wiwf}
\vspace{-10pt}
We here show that above fundamental lower bounds on \objbest\, are, in fact, achievable with carefully designed online learning algorithms.
We design an upper-confidence bound (UCB)-based algorithm for \objbest\, with \tf\ model based on the following ideas:

\textbf{(1). Playing sets of only $(m+1)$ sizes:} It is enough for the algorithm to play subsets of size either $(m+1)$ (to fully exploit the \tf) or $1$ (singleton sets), and not play a singleton unless there is a high degree of confidence about the single item being the best item. 

\textbf{(2). Parameter estimation from pairwise preferences:} It is possible to play the subset-wise game just by maintaining pairwise preference estimates of all $n$ items of the \mnl\, model using the idea of \rb--the idea of extracting pairwise comparisons from (partial) rankings and applying estimators on the obtained pairs treating each comparison independently  over the received subset-wise feedback---this is possible owning to the independence of irrelevant attributes (IIA) property of the MNL model (Defn. \ref{def:iia}). 


\textbf{(3). A new UCB-based max-min set building rule for playing large sets (\algbld):} Main novelty of \algmm\, lies in its underlying set building subroutine (Alg. \ref{alg:bld}, Appendix \ref{app:alg_wiwf} ), that constructs $S_t$ by applying a recursive max-min strategy on the UCB estimates of empirical pairwise preferences.

\vspace*{5pt}
\textbf{Algorithm description.} \algmm\, maintains an pairwise preference matrix $\bhP \in [0,1]^{n \times n}$, whose $(i,j)$-th entry $\hp_{ij}$ records the empirical probability of $i$ having beaten $j$ in a pairwise duel, and a corresponding upper confidence bound $u_{ij}$ for each pair $(i,j)$.
At any round $t$, it plays a subset $S_t \subseteq [n], \, |S_t| \in [k]$ using the \emph{Max-Min} set building rule \algbld\, (see Alg. \ref{alg:bld}), receives \tf\, $\sigma_t \in \Sigma_{S_t}^m$ from $S_t$, and updates the $\hp_{ij}$ entries of pairs in $S_t$ by applying \rb\, (Line $10$). 
The set building rule \algbld\, is at the heart of \algmm\, which builds the subset $S_t$ from a set of potential Condorcet winners ($\cC_t$) of round $t$: By recursively picking the strongest opponents of the already selected items using a max-min selection strategy on $u_{ij}$. The complete algorithm is presented in Alg. \ref{alg:mm}, Appendix \ref{app:alg_wiwf}.
 
\noindent 
The following result establishes that \algmm\, enjoys $O(\frac{n}{m}\ln T)$ regret with high probability. 

\begin{restatable}[\algmm: High Probability Regret bound]{thm}{whpmm}
\label{thm:whp_reg_mm}
Fix a time horizon $T$ and $\delta \in (0,1)$, $\alpha > \frac{1}{2}$. With probability at least $(1-\delta)$, the regret of \algmm~ for \objbest\, with \tf \, satisfies
$
R_{T}^1 \le \bigg(2\Big[ \frac{2\alpha n^2}{(2\alpha-1)\delta} \Big]^{\frac{1}{2\alpha-1}} + 2D\ln 2D\bigg)\hat \Delta_{\max} + \frac{ \ln T}{m+1}\sum_{i = 2}^{n}(D_{\max}  \hat \Delta_i), 
$ 
where $\forall i \in [n] \sm \{a^*\}$, $\hat \Delta_i = (\theta_{a^*} - \theta_i)$, $\Delta_i = \frac{\theta_{a^*}-\theta_i}{2(\theta_{a^*}+\theta_i)}$, $\hat \Delta_{\max} = \max_{i \in [n] \sm \{a^*\}}\hat \Delta_i$  $D_{1i} = \frac{4\alpha}{\Delta_i^2}$, $D := \sum_{i < j}D_{ij}$, $D_{\max} = \max_{i \in [n] \sm \{a^*\}}D_{1i}$. 
\end{restatable}

\noindent \textbf{Proof sketch.}
The proof hinges on analysing the entire run of \algmm\, by breaking it up into $3$ phases: (1). \rnd\, (2). \prog, and (3). \sat.

\vspace{1pt}
\textbf{(1).} \rnd: This phase runs from time $1$ to ${f(\delta)} = \Big[ \frac{2\alpha n^2}{(2\alpha-1)\delta} \Big]^{\frac{1}{2\alpha-1}}$, for any $\delta \in (0,1)$, such that for any $t> f(\delta)$, the upper confidence bounds $u_{ij}$ are guaranteed to be correct for the true values $p_{ij}$ for all pairs $(i,j) \in [n]\times[n]$ (i.e. $p_{ij} \le u_{ij}$), with high probability $(1-\delta)$. 

\textbf{(2).} \prog: After $t > f(\delta)$, the algorithm can be viewed as starting to explore the `confusing items', appearing in $\cC_t$, as potential candidates for the \bi\, $a^*$, and trying to capture $a^*$ in a holding set $\cB_t$. At any time, the set $\cB_t$ is either empty or a singleton by construction, and once $a^* \in \cB_t$ it stays their forever (with high probability). The \prog\, phase ensures that the algorithm explores fast enough so that within a constant number of rounds, $\cB_t$ captures $\{a^*\}$.

\textbf{(3).} \sat: This is the last phase from time $T_0(\delta)+1$ to $T$. As the name suggests, \algmm\, shows relatively stable behavior here, mostly playing $S_t = \{a^*\}$ and incurring almost no regret. 



Although Thm. \ref{thm:whp_reg_mm} shows a $(1-\delta)$-high probability regret bound for \algmm\, it is important to note that the algorithm itself does not require to take the probability of failure $(\delta)$ as input. As a consequence, by simply integrating the bound obtained in  Thm. \ref{thm:whp_reg_mm} over the entire range of $\delta \in [0,1]$, we get an expected regret bound of \algmm\, for \objbest\, with \tf:

\vspace*{2pt}
\begin{restatable}[]{thm}{expmm}
\label{thm:exp_reg_mm}
The expected regret of \algmm~ for \objbest\, with \tf \, is:
$
\E[R_{T}^1] \le \Bigg( 2\Big[ \frac{2\alpha n^2}{(2\alpha-1)} \Big]^{\frac{1}{2\alpha-1}}\frac{2\alpha-1}{\alpha-1} + 2D\ln 2D \Bigg)\hat \Delta_{\max} + \frac{ \ln T}{m+1}\sum_{i = 2}^{n}(D_{\max}  \hat \Delta_i) 
$, in $T$ rounds. 
\end{restatable}

\vspace*{-5pt}
\begin{rem}
\label{rem:ub_witf}
This is an upper bound on expected regret of the same order as that in the lower bound of Thm. \ref{thm:lb_wiwf}, which shows that the algorithm is essentially regret-optimal.
From Thm. \ref{thm:exp_reg_mm}, note that the first two terms $\Bigg( 2\Big[ \frac{2\alpha n^2}{(2\alpha-1)} \Big]^{\frac{1}{2\alpha-1}}\frac{2\alpha-1}{\alpha-1} + 2D\ln 2D \Bigg)\hat \Delta_{\max}$ of $\E[R_{T}^1]$ are essentially  instance specific constants, its only the third term which makes expected regret $O\Big( \frac{n\ln T}{m} \Big)$ which is in fact optimal in terms of its dependencies on $n$ and $T$ (since it matches the $\Omega\Big (\frac{n \ln T}{m}\Big)$ lower bound of Thm. \ref{thm:lb_witf}). Moreover the problem dependent complexity terms $(D_{\max}  \hat \Delta_i) = \frac{16\alpha(\theta_{a^*} - \theta_i)(\theta_{a^*}+\max_{j \in [n]\sm\{a^*\}}\theta_j)^2}{(\theta_{a^*}-\max_{j \in [n]\sm\{a^*\}}\theta_j)^2} \le \frac{64\alpha(\theta_{a^*} - \theta_i)(\theta_{a^*})}{(\theta_{a^*}-\max_{j \in [n]\sm\{a^*\}}\theta_j)^2} = O\Big(\frac{\theta_{a^*}}{(\theta_{a^*}-\max_{j \in [n]\sm\{a^*\}}\theta_j)}\Big)$, also brings out the inverse dependency on the `gap-term' $(\theta_{a^*}-\max_{j \in [n]\sm\{a^*\}}\theta_j)$ as discussed in Rem. \ref{rem:lb_witf}.
\end{rem}



\vspace{-10pt}
\section{Minimising \objk}
\label{sec:res_tk}
\vspace{-10pt}
In this section, we study the problem of minimising \objk\, with \tk. As before, we first derive a regret lower bound, for this learning setting, 
of the form $\Omega\big( \frac{n-k}{k\Delta_{(k)}}\ln T \big)$ (recall $\Delta_{(k)}$ from Sec. \ref{sec:prelims}).
We next propose an UCB based algorithm (Alg. \ref{alg:kmm}) for the same, along with a matching upper bound regret analysis (Thm. \ref{thm:whp_reg_kmm},\ref{thm:exp_reg_kmm}) showing optimality of our proposed algorithm. 

\vspace{-7pt}
\subsection{Regret lower bound for \objk\, with \tk}
\label{sec:lb_tktf}
\vspace{-4pt}



\begin{restatable}[Regret Lower Bound: \objk\, with \tf]{thm}{lbtktf}
\label{thm:lb_tktf}
For any \nr\, learning algorithm $\cA$ for \objk\, that uses \tk, and for any problem instance \mnl, the expected regret incurred by $\cA$ when run on it satisfies
$ 
\underset{T \to \infty }{\lim \inf}\,\E_{\btheta}\Big[\frac{R_T^{k}(\cA)}{\ln T}\Big] \ge \frac{\theta_1\theta_{k+1}}{\Delta_{(k)}}\frac{(n-k)}{k}, 
$ 
where $\E_{\btheta}[\cdot]$ denotes expectation under the algorithm and \mnl\, model. 
\end{restatable}

\vspace{-6pt}
\noindent \textbf{Proof sketch.}
Similar to \ref{thm:lb_witf}, the proof again relies on carefully constructing a true instance, with optimal set of \bk\, $\sS = [k]$, and a family of slightly perturbed alternative instances $\{\bnu^a: a \in [n]\sm \sS\}$, for each suboptimal arm $a \in [n]\sm \sS\}$, which we design as:
$ 
\textbf{(1). True Instance: } \text{MNL}(n,\btheta^1): \theta_1^1 = \theta_2^1 = \ldots = \theta_{k-1}^1 = \theta + 2\epsilon; \, \theta_n^1 = \theta + \epsilon;\,  \theta_{k+1}^1 = \theta_{k+2}^1 = \ldots \theta_{n-1}^1 = \theta,
$ 
for some $\theta \in \R_+$ and $\epsilon > 0$. Clearly \bk\,  of MNL$(n,\btheta^1)$ is $\sS[1] = [k-1]\cup \{n\}$. \textbf{(2). Altered Instances: } For every $n-k$ suboptimal items $a \notin \sS[1]$, now consider an altered instance
$
\textbf{Instance a, } \text{denoted by MNL}(n,\btheta^a), \text{ such that } \theta_a^a = \theta + 2\epsilon; ~\theta_i^a = \theta_i^1, ~~\forall i \in [n]\sm \{a\}.
$ 
The result of Thm. \ref{thm:lb_tktf} now can be obtained by following an exactly same procedure as described for the proof of Thm. \ref{thm:lb_witf}. The complete details is given in Appendix \ref{app:lb_tktf}.
$\hfill \square$

\begin{rem}
\label{rem:lb_tktf}
The regret lower bound of Thm. \ref{thm:lb_tktf} is  $\Omega(\frac{(n-k)\ln T}{k})$, with an instance-dependent term $\frac{\theta_1\theta_{k+1}}{(\theta_k-\theta_{k+1})}$ which shows for recovering the \bk, the problem complexity is  governed by the \emph{`gap'} between the $k^{th}$ and $(k+1)^{th}$ best item $\Delta_{(k)} = (\theta_k-\theta_{k+1})$, as consistent with intuition.
\end{rem}
\vspace{-7pt}
\subsection{An order-optimal algorithm with low \objk\, with \tk}
\label{sec:algo_tktf}
\vspace{-5pt}
\noindent 
\textbf{Main idea: A recursive set-building rule:} As with the  \algmm\, algorithm (Alg. \ref{alg:mm}), we maintain pairwise UCB estimates ($u_{ij}$) of empirical pairwise preferences $\hp_{ij}$ via \rb. However the {main difference here lies in the set building rule}, as here it is required to play sets of size exactly $k$.
  The core idea here is to recursively try to capture the set of \bk\, in an ordered set $\cB_t$, and, once the set is assumed to be found with confidence (formally $|\cB_t|=k$), to keep playing $\cB_t$ unless some other potential good item emerges,  which is then played replacing the weakest element $(\cB_t(k))$ of $\cB_t$.
The algorithm is described in Alg. \ref{alg:kmm}, Appendix \ref{app:alg_tktf}.


\begin{restatable}[\algkmm: High Probability Regret bound]{thm}{whpkmm}
\label{thm:whp_reg_kmm}
Given a fixed time horizon $T$ and $\delta \in (0,1)$, with high probability $(1-\delta)$, the regret incurred by \algkmm~ for \objk\, admits the bound 
$ 
R_T^{k} \le \bigg( 2\Big[\frac{2\alpha n^2}{(2\alpha-1)\delta} \Big]^{\frac{1}{2\alpha-1}} + 2\bar D^{(k)}\ln \big( 2\bar D^{(k)}\big) \bigg) \Delta'_{\max} + \frac{4\alpha\ln T}{k}\bigg(\sum_{b = k+1}^{n}\frac{(\theta_k-\theta_b)}{{\hat D}^2} \bigg),
$ 
where $D^{(k)}$ is an instance dependent constant (see Lem. \ref{lem:t_hat_deltak}, Appendix), $\Delta'_{\max} = \frac{\big(\sum_{i =1}^{k}\theta_i - \sum_{i=n-k+1}^{n}\theta_i\big)}{k}$,
and ${\hat D} = \min_{g \in [k-1]}(p_{kg}-p_{bg})$. 
\end{restatable}

\noindent \textbf{Proof sketch.}
Similar to Thm. \ref{thm:whp_reg_mm}, we prove the above bound dividing the entire run of algorithm \algkmm\, into three phases and applying an recursive argument:

\textbf{(1).} \rnd: Same as Thm. \ref{thm:whp_reg_mm}, in this case also this phase runs from time $1$ to ${f(\delta)} = \Big[ \frac{2\alpha n^2}{(2\alpha-1)\delta} \Big]^{\frac{1}{2\alpha-1}}$, for any $\delta \in (0,1)$, after which, for any $t> f(\delta)$, one can guarantee $p_{ij} \le u_{ij}$ for all pairs $(i,j) \in [n]\times [n]$, with high probability at least $(1-\delta)$. (Lem. \ref{lem:cdelta})

\textbf{(2).} \prog: 
The analysis of this phase is quite different from that of Thm. \ref{thm:whp_reg_mm}:
After $t > f(\delta)$, the algorithm starts exploring the items in the set of \bk\, in a \emph{recursive} manner--It first tries to capture (one of) the \bi s in $\cB_t(1)$. Once that slot is secured, it goes on for searching the second \bi\, from remaining pool of items and try capturing it in $\cB_t(2)$ and so on upto $\cB_t(k)$.
By definition, the phase ends at, say $t = T_0(\delta)$, when $\cB_{t} = \sS$. Moreover the update rule of \algkmm\ (along with Lem. \ref{lem:cdelta}) ensures that $\cB_{t} = \sS\, \forall t > T_0(\delta)$. The novelty of our analysis lies in showing that $T_0(\delta)$ is bounded by just a instance dependent complexity term which does not scale with $t$ (Lem. \ref{lem:t_hat_deltak}), and hence the regret incurred in this phase is also constant. 

\textbf{(3).} \sat: In the last phase from time $T_0(\delta)+1$ to $T$ \algkmm\, has already captured $\sS$ in $\cB_t$, and $\cB_t = \sS$ henceforth. 
Hence the algorithm mostly plays $S_t = \sS$ without incurring any regret. Only if any item outside $\cB_t$ enters into the list of potential \bk\,, it takes a very conservative approach of replacing the `weakest of $\cB_t$ by that element to make sure whether it indeed lies in or outside $\sS$. However we are able to show that any such suboptimal item $i \notin \sS$ can not occur for more than $O(\frac{\ln T}{\hat D^2})$ times (Lem. \ref{lem:kmm_reg_postT0}), combining which over all $[n]\sm[k]$ suboptimal items finally leads to the desired regret. \emph{The complete details are moved to Appendix  \ref{app:whp_reg_kmm}.}
$\hfill \square$

From Theorem \ref{thm:whp_reg_kmm}, we can also derive an expected regret bound for \algkmm\, in $T$ rounds is: 

\begin{restatable}[]{thm}{expkmm}
\label{thm:exp_reg_kmm}
The expected regret incurred by \algmm~ for \objk\,is:
\vspace*{-3pt}
\begin{align*}
\E[R_{T}^1] \le \bigg( 2\Big[ \frac{2\alpha n^2}{(2\alpha-1)} \Big]^{\frac{1}{2\alpha-1}}\frac{2\alpha-1}{\alpha-1} + 2\bar D^{(k)}\ln \big( 2\bar D^{(k)} \big) \bigg) \Delta'_{\max} + \frac{4\alpha\ln T}{k}\bigg(\sum_{b = k+1}^{n}\frac{(\theta_k-\theta_b)}{{\hat D}^2} \bigg).
\end{align*} 
\end{restatable}

\vspace*{-14pt}
\begin{rem}
In Thm. \ref{thm:exp_reg_kmm}, the first two terms $\bigg( 2\Big[\frac{2\alpha n^2}{(2\alpha-1)\delta} \Big]^{\frac{1}{2\alpha-1}} + 2\bar D^{(k)}\ln \big( 2\bar D^{(k)} \big) \bigg) \Delta'_{\max}$ of $\E[R_{T}^k]$ are just some \mnl\ model dependent constants which do not contribute to the learning rate of \algkmm, and the third term is $O\Big( \frac{(n-k)\ln T}{k} \Big)$ which varies optimally in terms of on $n$, $k$, $T$ matching the $\Omega\Big (\frac{(n-k) \ln T}{k}\Big)$ lower bound of Thm. \ref{thm:lb_tktf}). Also Rem. \ref{rem:lb_tktf} indicates an inverse dependency on the \emph{`gap-complexity'} $(\theta_k - \theta_{k+1})$, which also shows up in above bound through the component $\frac{(\theta_k-\theta_b)}{{\hat D}^2}$: Let $g^* \in [k-1]$ is the minimizer of $\hat D$, then $\frac{(\theta_k-\theta_b)}{{\hat D}^2} = \frac{(\theta_{g^*} + \theta_k)(\theta_b + \theta_{g^*})}{\theta_{g^*}^2(\theta_{k}-\theta_{b})} \le \frac{4}{\theta_{g^*}(\theta_k - \theta_{k+1})} $, where the upper bounding follows as $\theta_{g^*} \ge \theta_k > \theta_b$ for any $b \in [n]\sm[k]$, and $\theta_b \le \theta_{k+1}$ for any $b$.
\end{rem}

%


\vspace{-10pt}
\section{Experiments}
\label{sec:expts}
\vspace{-10pt}
In this section we present the empirical evaluations of our proposed algorithm \algmm\ (abbreviated as \textbf{MM}) on different synthetic datasets, and also compare them with different algorithms.
All results are reported as average across $50$ runs along with the standard deviations. 
For this we use $7$ different \mnl\, environments as described below: 

\textbf{\mnl\, Environments.}
1. {\it g1}, 2. {\it g4}, 3. {\it arith}, 4. {\it geo}, 5. {\it har} all with $n = 16$, and two larger models 6. {\it arithb}, and 7. {\it geob} with $n=50$ items in both. Details are moved to Appendix \ref{app:expts}.

We compare our proposed methods with the following two baselines which closely applies to our problem setup. {Note, as discussed in Sec. \ref{sec:intro}, none of the existing work exactly addresses our problem}.
\textbf{Algorithms.} 
\textbf{1. BD:} The {\it Battling-Duel} algorithm of \cite{SG18} with {\it RUCB} aalgorithm \cite{Zoghi+14RUCB} as the dueling bandit blackbox, and 
\textbf{2. Sp-TS:} The {\it Self-Sparring} algorithm of \cite{Sui+17} with Thompson Sampling \cite{TS12}, and
\textbf{3. MM:} Our proposed method {\it \algmm} for \objbest\  (Alg. \ref{alg:mm}).

\textbf{Comparing \objbest\, with \tf\, (Fig. \ref{fig:reg_wi}):} 
We first compare the regret performances for $k=10$ and $m=5$. 
From Fig. \ref{fig:reg_wi}, it clearly follows that in all cases \algmm\ uniformly outperforms the other two algorithms taking the advantage of \tf\, which the other two fail to make use of as they both allow repetitions in the played subsets which can not exploit the rank-ordered feedback to the full extent. Furthermore, the thompson sampling based \emph{Sp-TS} in general exhibits a much higher variance compared to the rest due to its bayesian nature. Also as expected, \emph{g1} and \emph{g4} being comparatively easier instances, i.e. with larger `gap' $\hat \Delta_{\max}$ (see Thm. \ref{thm:lb_wiwf}, \ref{thm:lb_witf},\ref{thm:whp_reg_mm}, \ref{thm:exp_reg_mm} etc. for a formal justification), our algorithm converges much faster on these models.
\vspace{-2pt}
\begin{figure*}[h]
\vspace{-5pt}
\hspace{-10pt}
\includegraphics[trim={3.2cm 0cm 0cm 0},clip,scale=0.1,width=0.245\textwidth]{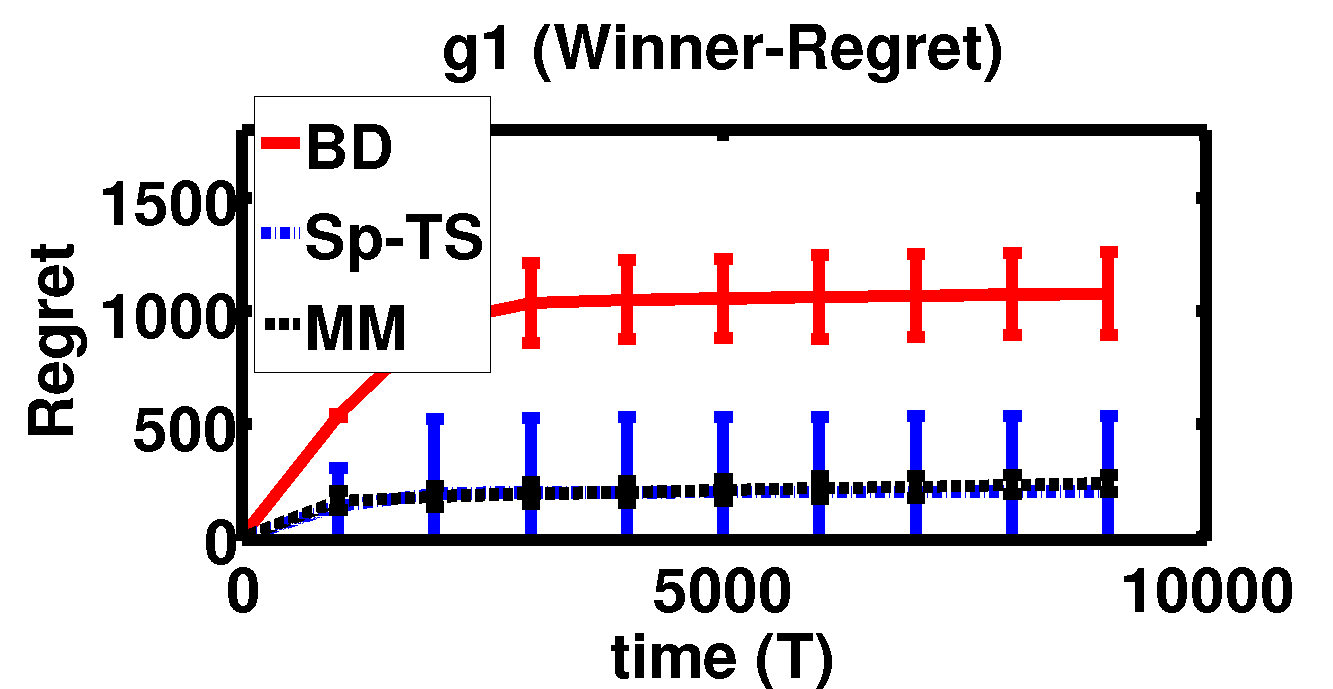}
\includegraphics[trim={3.2cm 0cm 0cm 0},clip,scale=0.1,width=0.245\textwidth]{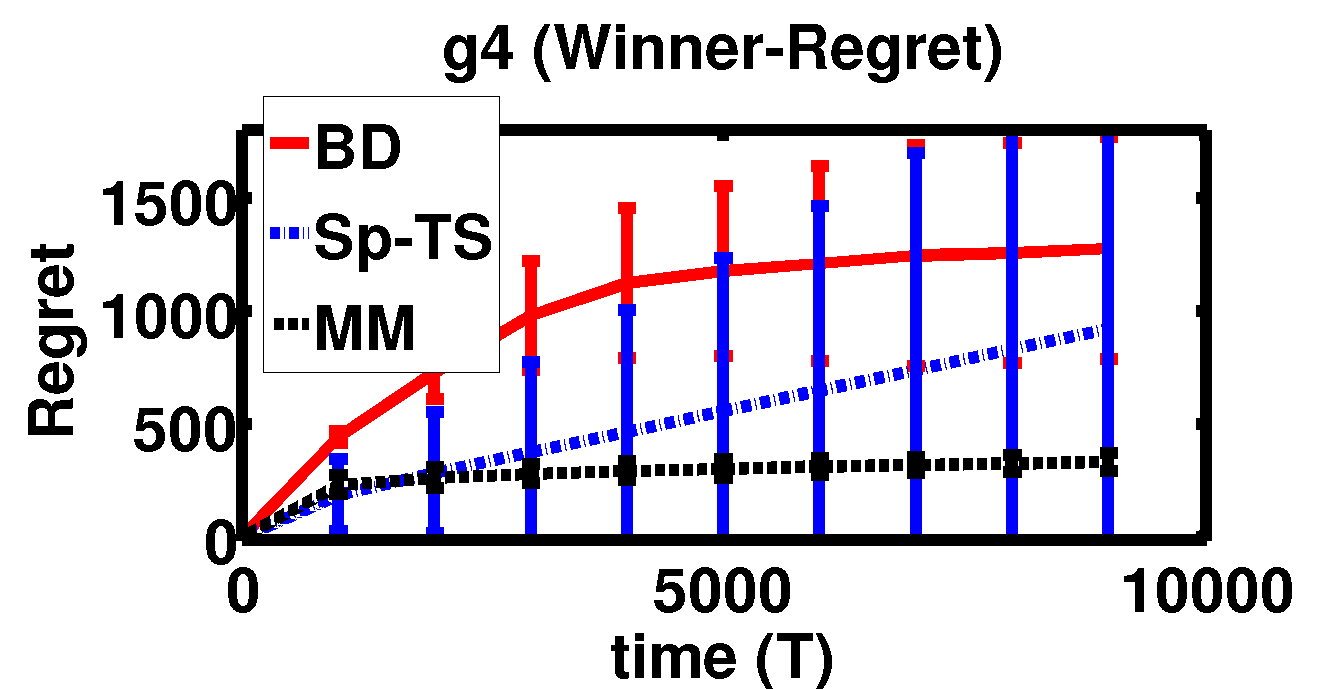}
\vspace{0pt}
\includegraphics[trim={3.2cm 0cm 0cm 0},clip,scale=0.1,width=0.245\textwidth]{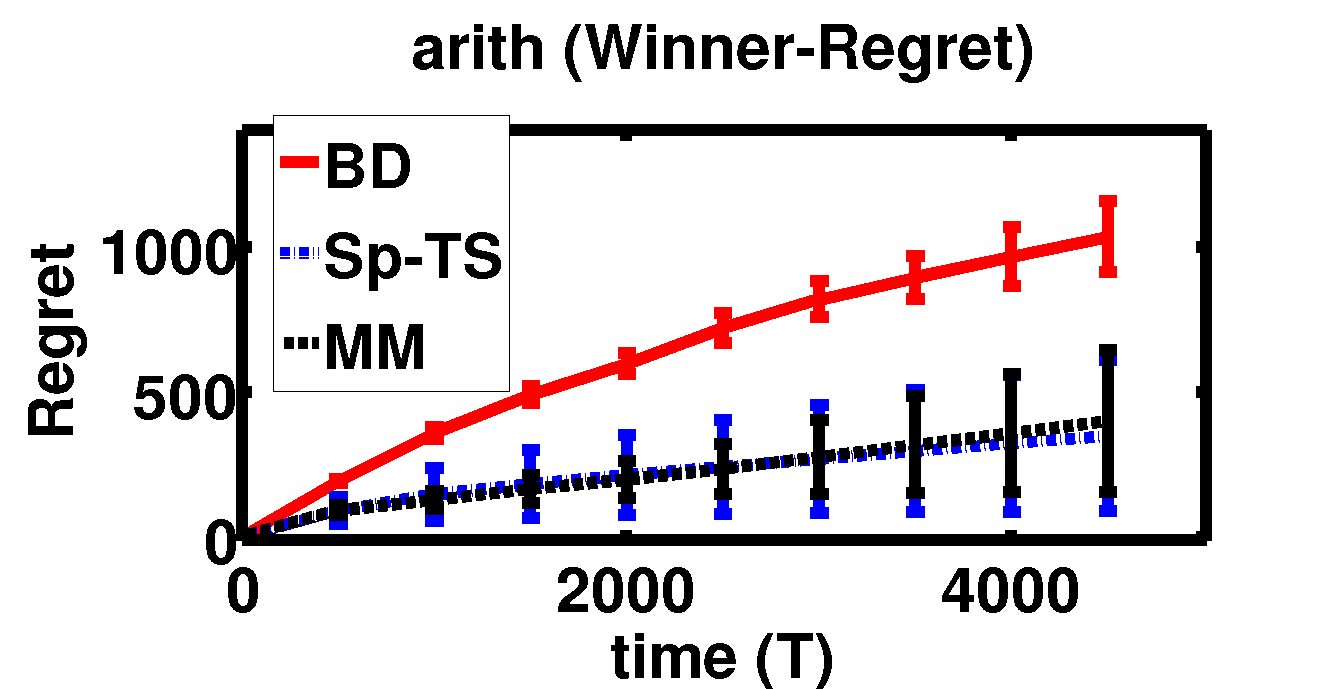}
\includegraphics[trim={3.2cm 0cm 0cm 0},clip,scale=0.1,width=0.245\textwidth]{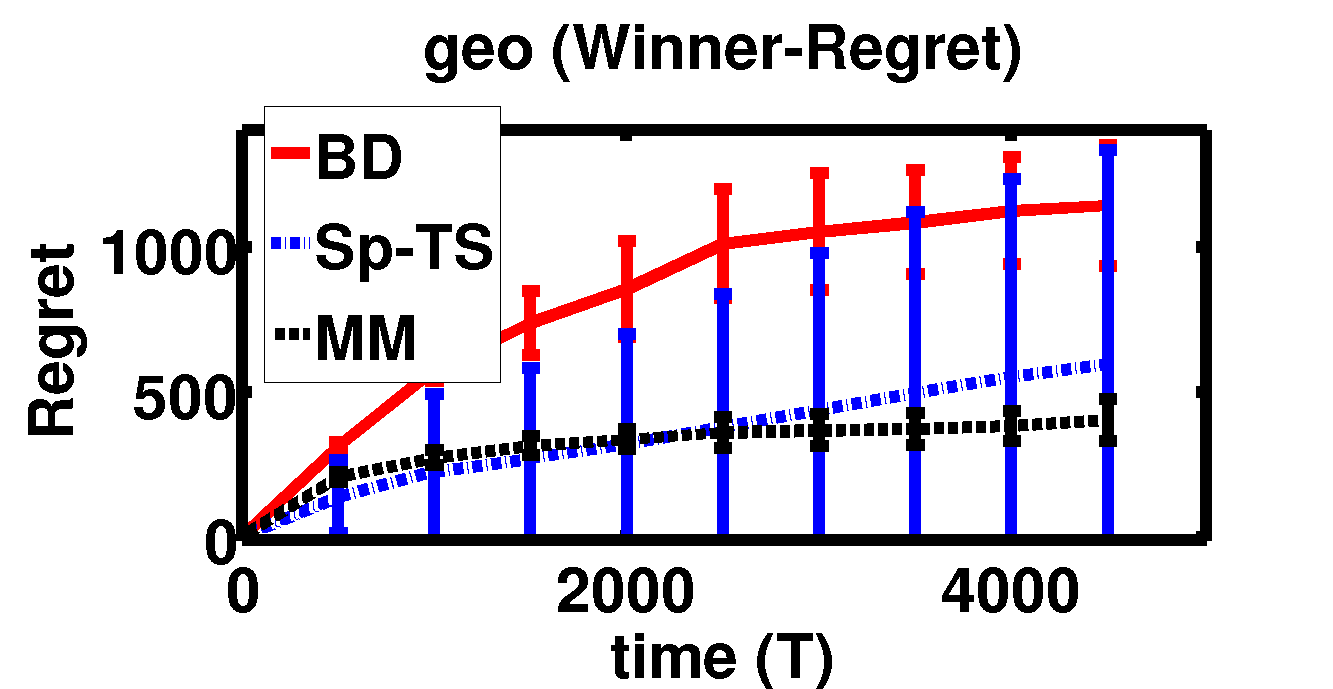}
\caption{Comparative performances on \objbest\, for $k = 10$, $m = 5$}
\label{fig:reg_wi}
\hspace{0pt}
\vspace{-15pt}
\end{figure*}
\vspace{-10pt}

\textbf{Comparing \objk\, performances for \tk\, (Fig. \ref{fig:reg_tk}):} 
We are not aware of any existing algorithm for \objk\, objective with \tk. We thus use a modified version of \emph{Sp-TS} algorithm \cite{Sui+17} described above for the purpose--it simply draws $k$-items without repetition and uses \rb\, updates to maintain the Beta posteriors. Here again, we see that our method \algkmm \ (\emph{Rec-MM}) uniformly outperforms \emph{Sp-TS} in all cases, and as before \emph{Sp-TS} shows a higher variability as well. Interestingly, our algorithm converges the fastest on \emph{g4}, it being the easiest model with largest `gap' $\Delta_{(k)}$ between the $k^{th}$ and $(k+1)^{th}$ best item (see Thm. \ref{thm:lb_tktf},\ref{thm:whp_reg_kmm},\ref{thm:exp_reg_kmm} etc.), and takes longest time for \emph{har} since it has the smallest $\Delta_{(k)}$.
\begin{figure*}[h]
\vspace{-0pt}
\hspace{-10pt}
\includegraphics[trim={3.2cm 0cm 0cm 0},clip,scale=0.1,width=0.245\textwidth]{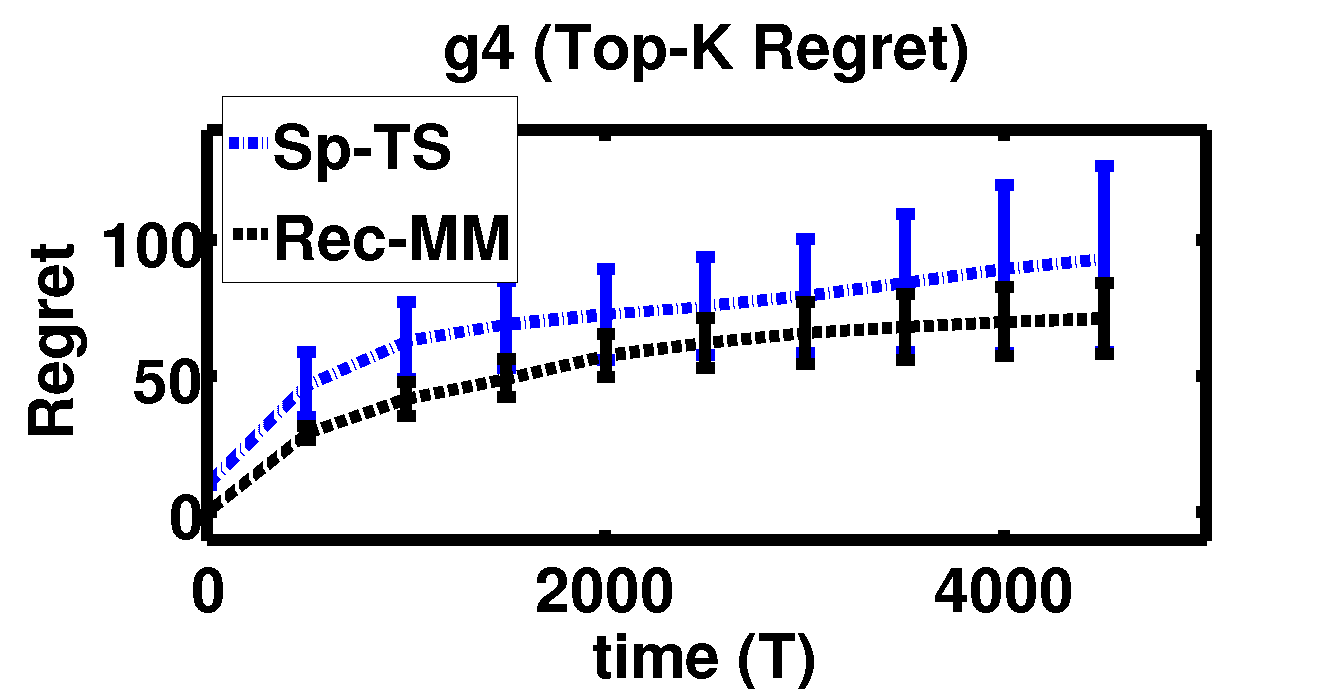}
\includegraphics[trim={3.2cm 0cm 0cm 0},clip,scale=0.1,width=0.245\textwidth]{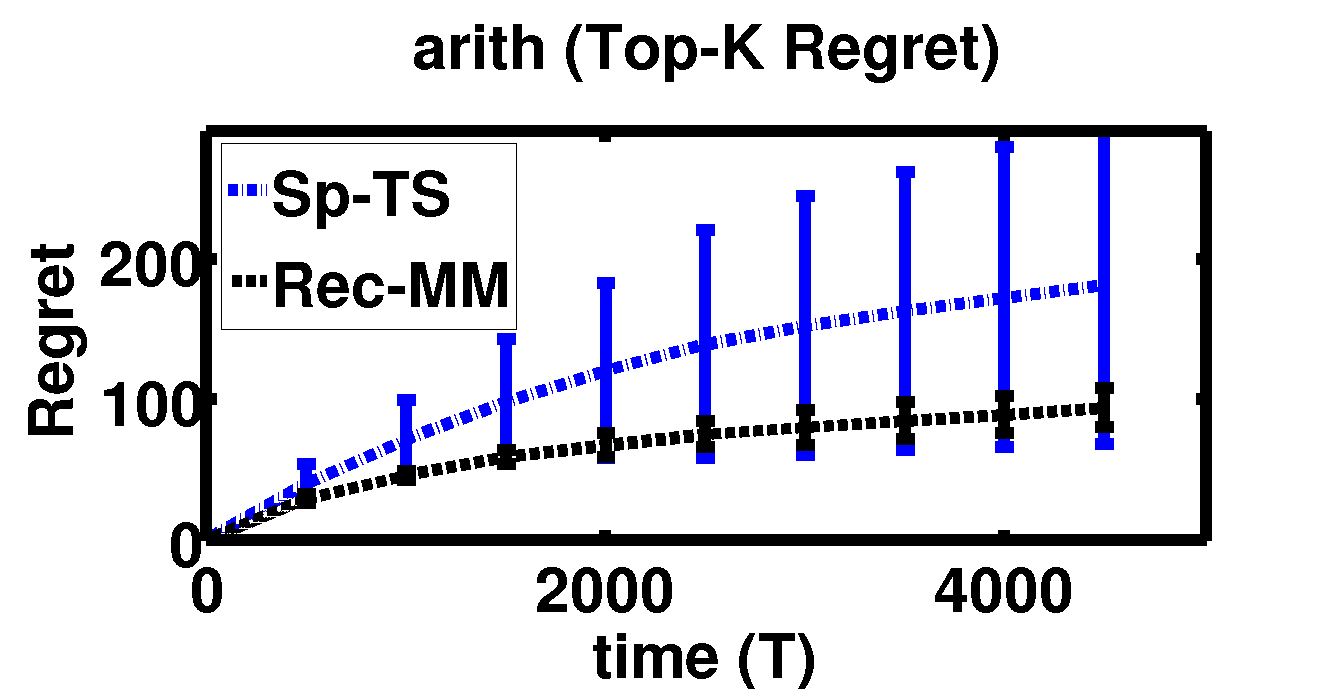}
\vspace{0pt}
\includegraphics[trim={3.2cm 0cm 0cm 0},clip,scale=0.1,width=0.245\textwidth]{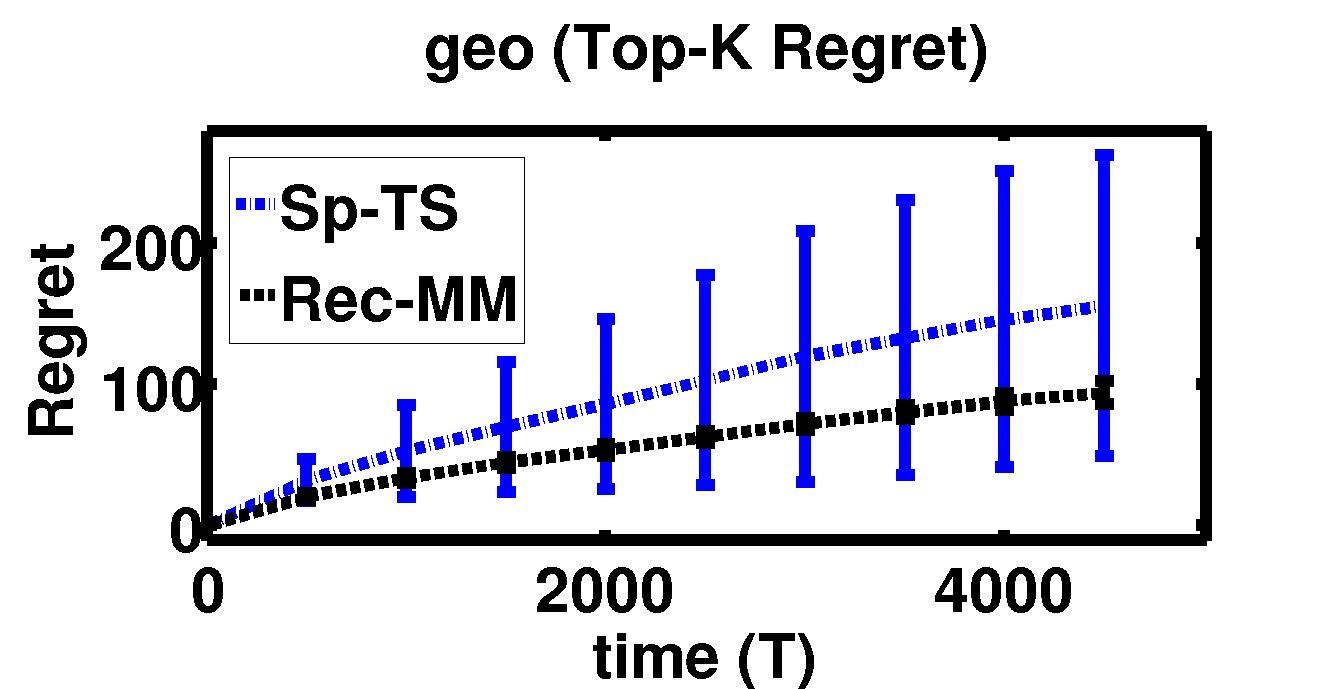}
\includegraphics[trim={3.2cm 0cm 0cm 0},clip,scale=0.1,width=0.245\textwidth]{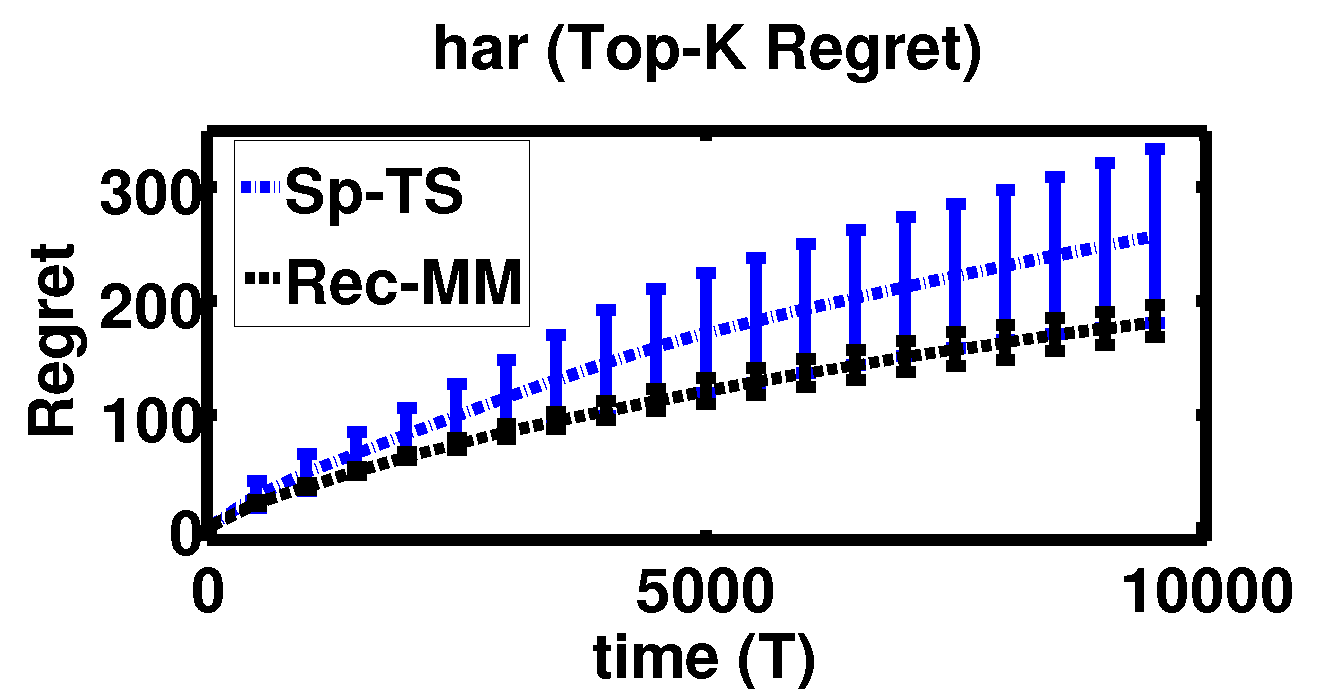}
\caption{Comparative performances on \objk\, for $k = 10$}
\label{fig:reg_tk}
\hspace{0pt}
\end{figure*}
\vspace{4pt}

\textbf{Effect of varying $m$ with fixed $k$ (Fig. \ref{fig:reg_m}):} We also studied our algorithm \algmm\,, with varying size rank-ordered feedback $(m)$, keeping the subsetsize $(k)$ fixed, both for \objbest\, and \objk\, objective, on the larger models {\it arithb} and {\it geob} which has $n=50$ items. As expected, in both cases, regret scales down with increasing $m$ (justifying the bounds in Thm. \ref{thm:whp_reg_mm},\ref{thm:exp_reg_mm}),\ref{thm:whp_reg_kmm},\ref{thm:exp_reg_kmm}).
\vspace{-2pt}
\begin{figure*}[h]
\vspace{-5pt}
\hspace{-10pt}
\includegraphics[trim={3.2cm 0cm 0cm 0},clip,scale=0.1,width=0.245\textwidth]{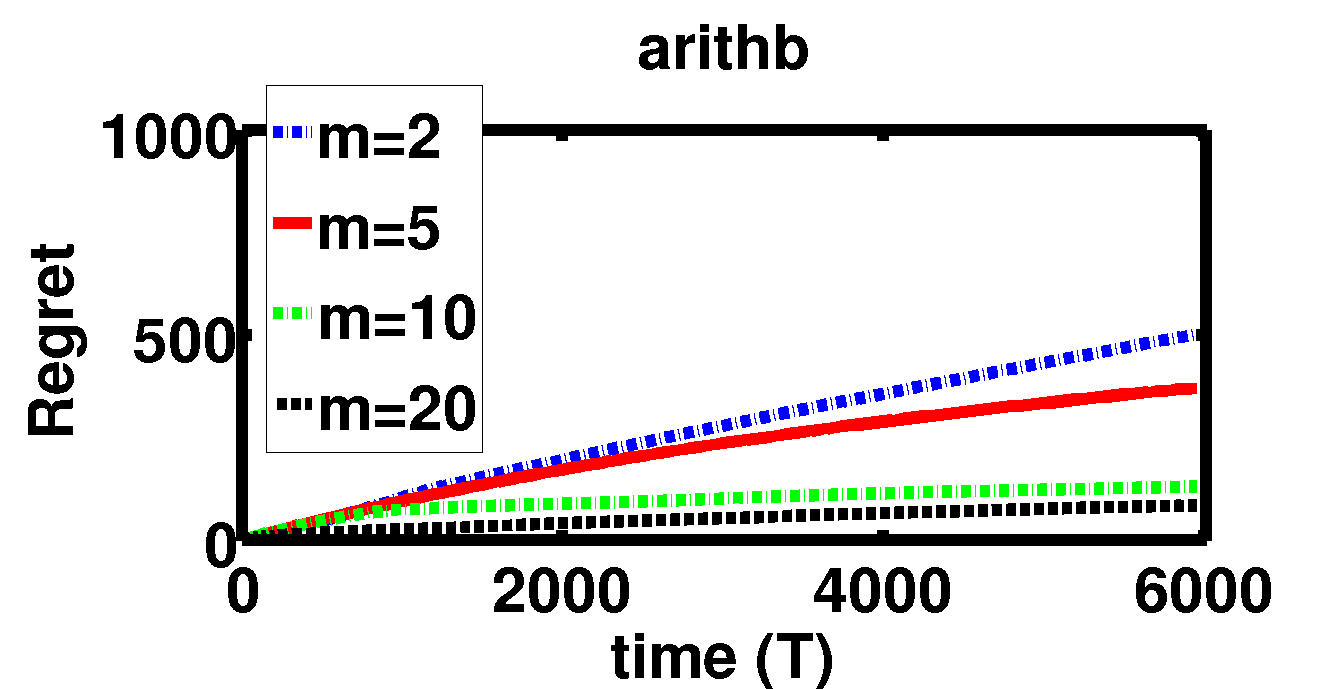}
\includegraphics[trim={3.2cm 0cm 0cm 0},clip,scale=0.1,width=0.245\textwidth]{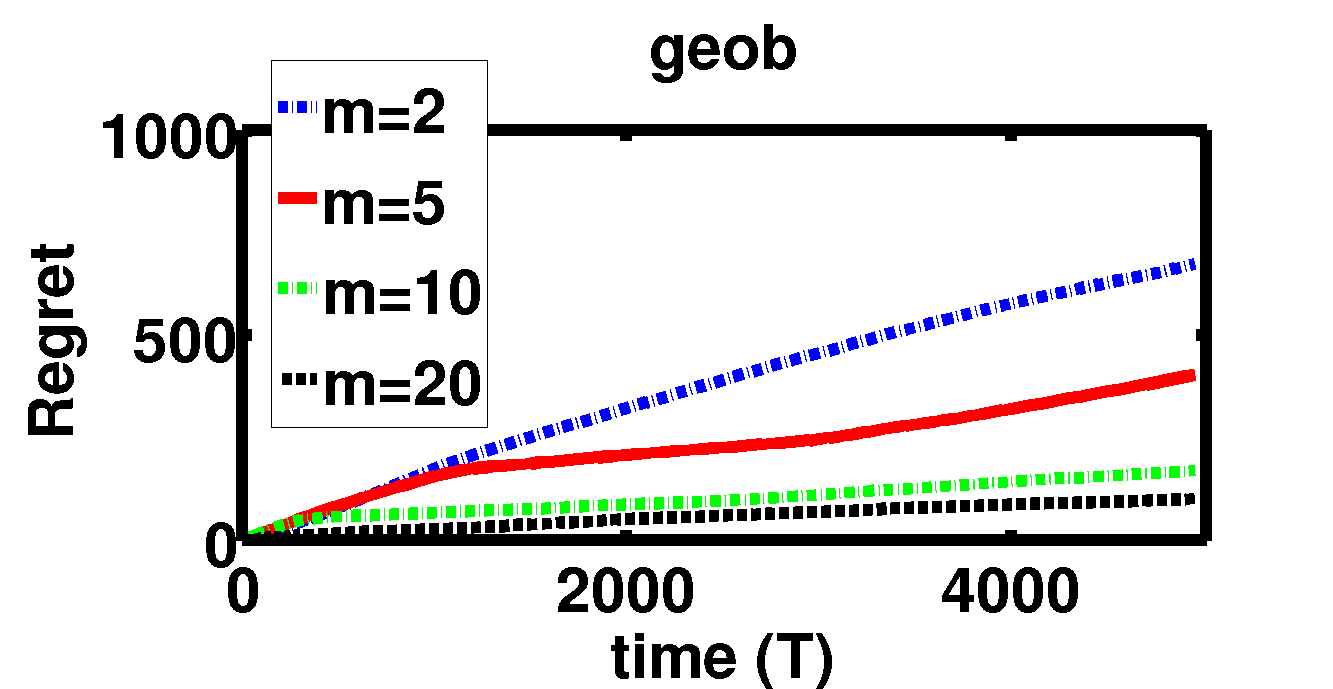}
\vspace{0pt}
\includegraphics[trim={3.2cm 0cm 0cm 0},clip,scale=0.1,width=0.245\textwidth]{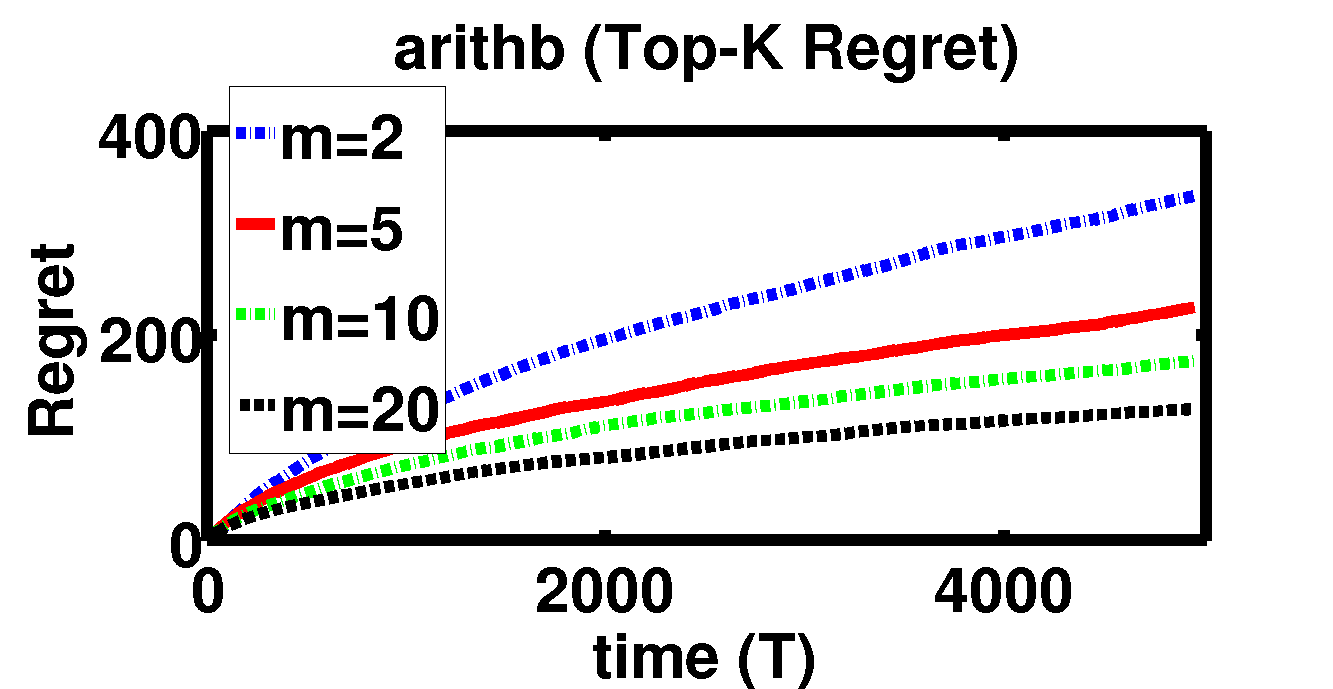}
\includegraphics[trim={3.2cm 0cm 0cm 0},clip,scale=0.1,width=0.245\textwidth]{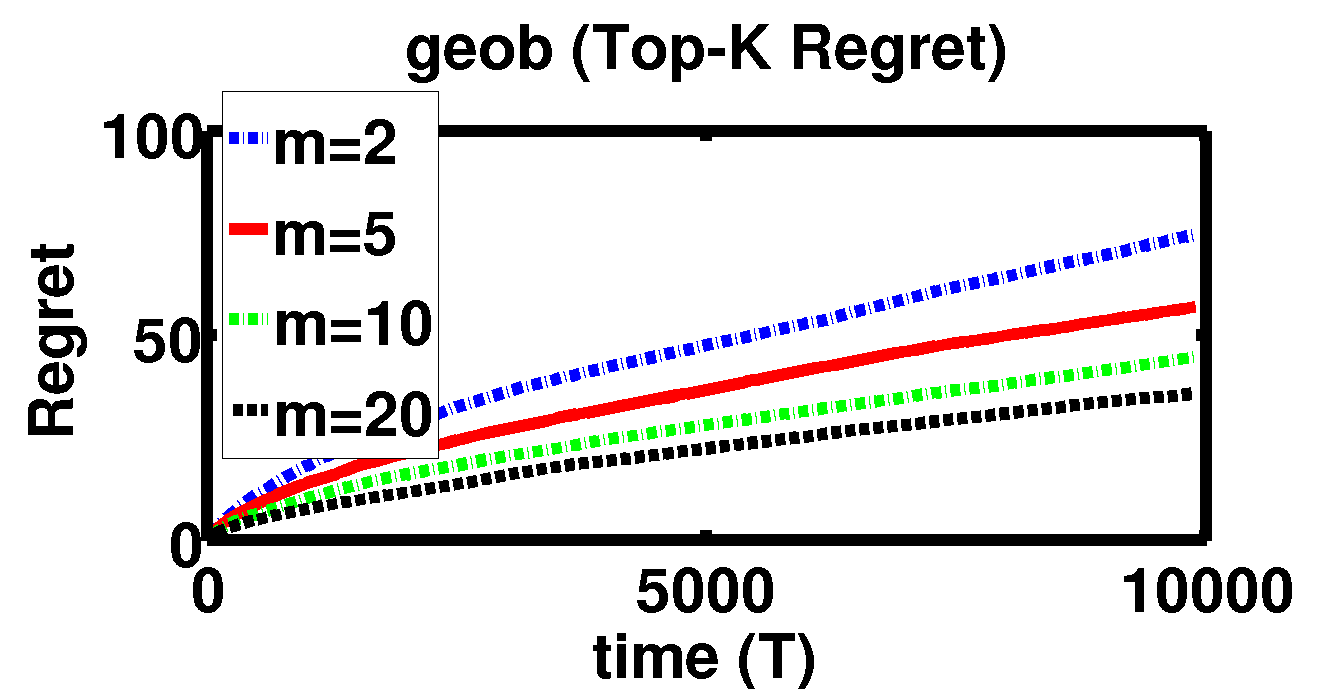}
\caption{Regret with varying $m$ with fixed $k = 40$ (on our proposed algorithm \algmm)}
\label{fig:reg_m}
\hspace{0pt}
\vspace{-10pt}
\end{figure*}

\vspace*{-20pt}
\section{Conclusion and Future Work}
\label{sec:conclusion}
\vspace*{-10pt}
Although we have analysed low-regret algorithms for learning with subset-wise preferences, there are several avenues for investigation that open up with these results. The case of learning with contextual subset-wise models is an important and practically relevant problem, as is the problem of considering mixed cardinal and ordinal feedback structures in online learning. Other directions of interest could be studying the budgeted version where there are costs associated with the amount of preference information that may be elicited in each round, or analysing the current problem on a variety of subset choice models, e.g. multinomial probit, Mallows, or even adversarial preference models etc.

\subsubsection*{Acknowledgements}
The authors are grateful to the anonymous reviewers for valuable feedback. This work is supported by a Qualcomm Innovation Fellowship 2019, and the Indigenous 5G Test Bed project grant from the Dept. of Telecommunications, Government of India. Aadirupa Saha thanks Arun Rajkumar for the valuable discussions, and the Tata Trusts and ACM-India/IARCS Travel Grants for travel support.

\newpage

\bibliographystyle{plainnat}
\bibliography{pref-comb-bandits}

\begin{thebibliography}{47}
\providecommand{\natexlab}[1]{#1}
\providecommand{\url}[1]{\texttt{#1}}
\expandafter\ifx\csname urlstyle\endcsname\relax
  \providecommand{\doi}[1]{doi: #1}\else
  \providecommand{\doi}{doi: \begingroup \urlstyle{rm}\Url}\fi

\bibitem[Agrawal and Goyal(2012)]{TS12}
Shipra Agrawal and Navin Goyal.
\newblock Analysis of {T}hompson sampling for the multi-armed bandit problem.
\newblock In \emph{Conference on Learning Theory}, pages 39--1, 2012.

\bibitem[Agrawal et~al.(2016)Agrawal, Avadhanula, Goyal, and Zeevi]{Agrawal+16}
Shipra Agrawal, Vashist Avadhanula, Vineet Goyal, and Assaf Zeevi.
\newblock A near-optimal exploration-exploitation approach for assortment
  selection.
\newblock In \emph{Proceedings of the 2016 ACM Conference on Economics and
  Computation}, pages 599--600. ACM, 2016.

\bibitem[Agrawal et~al.(2017)Agrawal, Avadhanula, Goyal, and
  Zeevi]{Agrawal+17TS}
Shipra Agrawal, Vashist Avadhanula, Vineet Goyal, and Assaf Zeevi.
\newblock Thompson sampling for the mnl-bandit.
\newblock \emph{Machine Learning Research}, 65:\penalty0 1--3, 2017.

\bibitem[Agrawal et~al.(2019)Agrawal, Avadhanula, Goyal, and
  Zeevi]{Agrawal+17Dyn}
Shipra Agrawal, Vashist Avadhanula, Vineet Goyal, and Assaf Zeevi.
\newblock Mnl-bandit: A dynamic learning approach to assortment selection.
\newblock \emph{Operations Research}, 67\penalty0 (5):\penalty0 1453--1485,
  2019.

\bibitem[Ailon et~al.(2014)Ailon, Karnin, and Joachims]{Ailon+14}
Nir Ailon, Zohar~Shay Karnin, and Thorsten Joachims.
\newblock Reducing dueling bandits to cardinal bandits.
\newblock In \emph{ICML}, volume~32, pages 856--864, 2014.

\bibitem[Alwin and Krosnick(1985)]{alwin1985measurement}
Duane~F Alwin and Jon~A Krosnick.
\newblock The measurement of values in surveys: A comparison of ratings and
  rankings.
\newblock \emph{Public Opinion Quarterly}, 49\penalty0 (4):\penalty0 535--552,
  1985.

\bibitem[Azari et~al.(2012)Azari, Parkes, and Xia]{Az+12}
Hossein Azari, David Parkes, and Lirong Xia.
\newblock Random utility theory for social choice.
\newblock In \emph{Advances in Neural Information Processing Systems}, pages
  126--134, 2012.

\bibitem[Bart{\'o}k et~al.(2011)Bart{\'o}k, P{\'a}l, and
  Szepesv{\'a}ri]{bartok2011minimax}
G{\'a}bor Bart{\'o}k, D{\'a}vid P{\'a}l, and Csaba Szepesv{\'a}ri.
\newblock Minimax regret of finite partial-monitoring games in stochastic
  environments.
\newblock In \emph{Proceedings of the 24th Annual Conference on Learning
  Theory}, pages 133--154, 2011.

\bibitem[Ben-Akiva et~al.(1994)Ben-Akiva, Bradley, Morikawa, Benjamin, Novak,
  Oppewal, and Rao]{ben1994combining}
Moshe Ben-Akiva, Mark Bradley, Takayuki Morikawa, Julian Benjamin, Thomas
  Novak, Harmen Oppewal, and Vithala Rao.
\newblock Combining revealed and stated preferences data.
\newblock \emph{Marketing Letters}, 5\penalty0 (4):\penalty0 335--349, 1994.

\bibitem[Benson et~al.(2016)Benson, Kumar, and Tomkins]{IIA-relevance16}
Austin~R Benson, Ravi Kumar, and Andrew Tomkins.
\newblock On the relevance of irrelevant alternatives.
\newblock In \emph{Proceedings of the 25th International Conference on World
  Wide Web}, pages 963--973. International World Wide Web Conferences Steering
  Committee, 2016.

\bibitem[Brost et~al.(2016)Brost, Seldin, Cox, and Lioma]{Brost+16}
Brian Brost, Yevgeny Seldin, Ingemar~J. Cox, and Christina Lioma.
\newblock Multi-dueling bandits and their application to online ranker
  evaluation.
\newblock \emph{CoRR}, abs/1608.06253, 2016.

\bibitem[Busa-Fekete and H{\"u}llermeier(2014)]{Busa14survey}
R{\'o}bert Busa-Fekete and Eyke H{\"u}llermeier.
\newblock A survey of preference-based online learning with bandit algorithms.
\newblock In \emph{International Conference on Algorithmic Learning Theory},
  pages 18--39. Springer, 2014.

\bibitem[Busa-Fekete et~al.(2014)Busa-Fekete, H{\"u}llermeier, and
  Sz{\"o}r{\'e}nyi]{Busa_mallows}
R{\'o}bert Busa-Fekete, Eyke H{\"u}llermeier, and Bal{\'a}zs Sz{\"o}r{\'e}nyi.
\newblock Preference-based rank elicitation using statistical models: The case
  of mallows.
\newblock In \emph{Proceedings of The 31st International Conference on Machine
  Learning}, volume~32, 2014.

\bibitem[Cesa-Bianchi and Lugosi(2012)]{cesa12}
Nicolo Cesa-Bianchi and G{\'a}bor Lugosi.
\newblock Combinatorial bandits.
\newblock \emph{Journal of Computer and System Sciences}, 78\penalty0
  (5):\penalty0 1404--1422, 2012.

\bibitem[Chen et~al.(2013{\natexlab{a}})Chen, Wang, and
  Yuan]{chen2013combinatorial}
Wei Chen, Yajun Wang, and Yang Yuan.
\newblock Combinatorial multi-armed bandit: General framework and applications.
\newblock In \emph{International Conference on Machine Learning}, pages
  151--159, 2013{\natexlab{a}}.

\bibitem[Chen et~al.(2013{\natexlab{b}})Chen, Bennett, Collins-Thompson, and
  Horvitz]{chen2013pairwise}
Xi~Chen, Paul~N Bennett, Kevyn Collins-Thompson, and Eric Horvitz.
\newblock Pairwise ranking aggregation in a crowdsourced setting.
\newblock In \emph{Proceedings of the sixth ACM international conference on Web
  search and data mining}, pages 193--202. ACM, 2013{\natexlab{b}}.

\bibitem[Chen et~al.(2018)Chen, Li, and Mao]{ChenSoda+18}
Xi~Chen, Yuanzhi Li, and Jieming Mao.
\newblock A nearly instance optimal algorithm for top-k ranking under the
  multinomial logit model.
\newblock In \emph{Proceedings of the Twenty-Ninth Annual ACM-SIAM Symposium on
  Discrete Algorithms}, pages 2504--2522. SIAM, 2018.

\bibitem[Chen and Suh(2015)]{SueIcml+15}
Yuxin Chen and Changho Suh.
\newblock Spectral mle: Top-k rank aggregation from pairwise comparisons.
\newblock In \emph{International Conference on Machine Learning}, pages
  371--380, 2015.

\bibitem[Combes et~al.(2015)Combes, Shahi, Proutiere, et~al.]{combes15}
Richard Combes, Mohammad Sadegh Talebi~Mazraeh Shahi, Alexandre Proutiere,
  et~al.
\newblock Combinatorial bandits revisited.
\newblock In \emph{Advances in Neural Information Processing Systems}, pages
  2116--2124, 2015.

\bibitem[Cover and Thomas(2012)]{cover2012elements}
Thomas~M Cover and Joy~A Thomas.
\newblock \emph{Elements of information theory}.
\newblock John Wiley \& Sons, 2012.

\bibitem[Garivier et~al.(2018)Garivier, M{\'e}nard, and Stoltz]{Garivier+16}
Aur{\'e}lien Garivier, Pierre M{\'e}nard, and Gilles Stoltz.
\newblock Explore first, exploit next: The true shape of regret in bandit
  problems.
\newblock \emph{Mathematics of Operations Research}, 44\penalty0 (2):\penalty0
  377--399, 2018.

\bibitem[Graepel and Herbrich(2006)]{graepel2006ranking}
Thore Graepel and Ralf Herbrich.
\newblock Ranking and matchmaking.
\newblock \emph{Game Developer Magazine}, 25:\penalty0 34, 2006.

\bibitem[Hajek et~al.(2014)Hajek, Oh, and Xu]{Hajek+14}
Bruce Hajek, Sewoong Oh, and Jiaming Xu.
\newblock Minimax-optimal inference from partial rankings.
\newblock In \emph{Advances in Neural Information Processing Systems}, pages
  1475--1483, 2014.

\bibitem[Hensher(1994)]{hensher1994stated}
David~A Hensher.
\newblock Stated preference analysis of travel choices: the state of practice.
\newblock \emph{Transportation}, 21\penalty0 (2):\penalty0 107--133, 1994.

\bibitem[Hofmann(2013)]{hofmann2013fast}
Katja Hofmann.
\newblock Fast and reliable online learning to rank for information retrieval.
\newblock In \emph{SIGIR Forum}, volume~47, page 140, 2013.

\bibitem[Jang et~al.(2017)Jang, Kim, Suh, and Oh]{SueIcml+17}
Minje Jang, Sunghyun Kim, Changho Suh, and Sewoong Oh.
\newblock Optimal sample complexity of m-wise data for top-k ranking.
\newblock In \emph{Advances in Neural Information Processing Systems}, pages
  1685--1695, 2017.

\bibitem[Katariya et~al.(2016)Katariya, Kveton, Szepesvari, and Wen]{DCM}
Sumeet Katariya, Branislav Kveton, Csaba Szepesvari, and Zheng Wen.
\newblock Dcm bandits: Learning to rank with multiple clicks.
\newblock In \emph{International Conference on Machine Learning}, pages
  1215--1224, 2016.

\bibitem[Kaufmann et~al.(2016)Kaufmann, Capp{\'e}, and
  Garivier]{Kaufmann+16_OnComplexity}
Emilie Kaufmann, Olivier Capp{\'e}, and Aur{\'e}lien Garivier.
\newblock On the complexity of best-arm identification in multi-armed bandit
  models.
\newblock \emph{The Journal of Machine Learning Research}, 17\penalty0
  (1):\penalty0 1--42, 2016.

\bibitem[Khetan and Oh(2016)]{KhetanOh16}
Ashish Khetan and Sewoong Oh.
\newblock Data-driven rank breaking for efficient rank aggregation.
\newblock \emph{Journal of Machine Learning Research}, 17\penalty0
  (193):\penalty0 1--54, 2016.

\bibitem[Komiyama et~al.(2015)Komiyama, Honda, Kashima, and
  Nakagawa]{Komiyama+15}
Junpei Komiyama, Junya Honda, Hisashi Kashima, and Hiroshi Nakagawa.
\newblock Regret lower bound and optimal algorithm in dueling bandit problem.
\newblock In \emph{COLT}, pages 1141--1154, 2015.

\bibitem[Kveton et~al.(2015)Kveton, Wen, Ashkan, and
  Szepesvari]{kveton2015tight}
Branislav Kveton, Zheng Wen, Azin Ashkan, and Csaba Szepesvari.
\newblock Tight regret bounds for stochastic combinatorial semi-bandits.
\newblock In \emph{Artificial Intelligence and Statistics}, pages 535--543,
  2015.

\bibitem[Lai and Robbins(1985)]{lai1985asymptotically}
Tze~Leung Lai and Herbert Robbins.
\newblock Asymptotically efficient adaptive allocation rules.
\newblock \emph{Advances in applied mathematics}, 6\penalty0 (1):\penalty0
  4--22, 1985.

\bibitem[Popescu et~al.(2016)Popescu, Dragomir, Slu{{s}}anschi, and
  St{{a}}n{{a}}{{s}}il{{a}}]{klub16}
Pantelimon~G Popescu, Silvestru Dragomir, Emil~I Slu{{s}}anschi, and Octavian~N
  St{{a}}n{{a}}{{s}}il{{a}}.
\newblock Bounds for {K}ullback-{L}eibler divergence.
\newblock \emph{Electronic Journal of Differential Equations}, 2016, 2016.

\bibitem[Radlinski et~al.(2008)Radlinski, Kurup, and
  Joachims]{radlinski2008does}
Filip Radlinski, Madhu Kurup, and Thorsten Joachims.
\newblock How does clickthrough data reflect retrieval quality?
\newblock In \emph{Proceedings of the 17th ACM conference on Information and
  knowledge management}, pages 43--52. ACM, 2008.

\bibitem[Ren et~al.(2018)Ren, Liu, and Shroff]{Ren+18}
Wenbo Ren, Jia Liu, and Ness~B Shroff.
\newblock P{AC} ranking from pairwise and listwise queries: Lower bounds and
  upper bounds.
\newblock \emph{arXiv preprint arXiv:1806.02970}, 2018.

\bibitem[Saha and Gopalan(2018)]{SG18}
Aadirupa Saha and Aditya Gopalan.
\newblock Battle of bandits.
\newblock In \emph{Uncertainty in Artificial Intelligence}, 2018.

\bibitem[Saha and Gopalan(2019)]{SGwin18}
Aadirupa Saha and Aditya Gopalan.
\newblock {PAC Battling Bandits in the Plackett-Luce Model}.
\newblock In \emph{Algorithmic Learning Theory}, pages 700--737, 2019.

\bibitem[Soufiani et~al.(2014)Soufiani, Parkes, and Xia]{AzariRB+14}
Hossein~Azari Soufiani, David~C Parkes, and Lirong Xia.
\newblock Computing parametric ranking models via rank-breaking.
\newblock In \emph{ICML}, pages 360--368, 2014.

\bibitem[Sui et~al.(2017)Sui, Zhuang, Burdick, and Yue]{Sui+17}
Yanan Sui, Vincent Zhuang, Joel Burdick, and Yisong Yue.
\newblock Multi-dueling bandits with dependent arms.
\newblock In \emph{Conference on Uncertainty in Artificial Intelligence},
  UAI'17, 2017.

\bibitem[Sz{\"o}r{\'e}nyi et~al.(2015)Sz{\"o}r{\'e}nyi, Busa-Fekete, Paul, and
  H{\"u}llermeier]{Busa_pl}
Bal{\'a}zs Sz{\"o}r{\'e}nyi, R{\'o}bert Busa-Fekete, Adil Paul, and Eyke
  H{\"u}llermeier.
\newblock Online rank elicitation for plackett-luce: A dueling bandits
  approach.
\newblock In \emph{Advances in Neural Information Processing Systems}, pages
  604--612, 2015.

\bibitem[Urvoy et~al.(2013)Urvoy, Clerot, F{\'e}raud, and Naamane]{SAVAGE}
Tanguy Urvoy, Fabrice Clerot, Raphael F{\'e}raud, and Sami Naamane.
\newblock Generic exploration and k-armed voting bandits.
\newblock In \emph{International Conference on Machine Learning}, pages 91--99,
  2013.

\bibitem[Wu and Liu(2016)]{DTS}
Huasen Wu and Xin Liu.
\newblock Double {T}hompson sampling for dueling bandits.
\newblock In \emph{Advances in Neural Information Processing Systems}, pages
  649--657, 2016.

\bibitem[Yue and Joachims(2009)]{Yue+09}
Yisong Yue and Thorsten Joachims.
\newblock Interactively optimizing information retrieval systems as a dueling
  bandits problem.
\newblock In \emph{Proceedings of the 26th Annual International Conference on
  Machine Learning}, pages 1201--1208. ACM, 2009.

\bibitem[Yue and Joachims(2011)]{BTM}
Yisong Yue and Thorsten Joachims.
\newblock Beat the mean bandit.
\newblock In \emph{Proceedings of the 28th International Conference on Machine
  Learning (ICML-11)}, pages 241--248, 2011.

\bibitem[Yue et~al.(2012)Yue, Broder, Kleinberg, and Joachims]{Yue+12}
Yisong Yue, Josef Broder, Robert Kleinberg, and Thorsten Joachims.
\newblock The $k$-armed dueling bandits problem.
\newblock \emph{Journal of Computer and System Sciences}, 78\penalty0
  (5):\penalty0 1538--1556, 2012.

\bibitem[Zoghi et~al.(2013)Zoghi, Whiteson, Munos, and de~Rijke]{Zoghi+13}
Masrour Zoghi, Shimon Whiteson, Remi Munos, and Maarten de~Rijke.
\newblock Relative upper confidence bound for the $k$-armed dueling bandit
  problem.
\newblock \emph{arXiv preprint arXiv:1312.3393}, 2013.

\bibitem[Zoghi et~al.(2014)Zoghi, Whiteson, Munos, Rijke, et~al.]{Zoghi+14RUCB}
Masrour Zoghi, Shimon Whiteson, Remi Munos, Maarten~de Rijke, et~al.
\newblock Relative upper confidence bound for the $k$-armed dueling bandit
  problem.
\newblock In \emph{JMLR Workshop and Conference Proceedings}, number~32, pages
  10--18. JMLR, 2014.

\end{thebibliography}

\newpage

\appendix

\section*{\large{Supplementary for Combinatorial Bandits with Relative Feedback}}
\vspace*{0.5cm}

\section{Related Works}
\label{app:rel}
Over the last decade, online learning from pairwise preferences has seen a widespread resurgence in the form of the Dueling Bandit problem, from the points of view of both pure-exploration (PAC) settings  \citep{BTM,Busa_pl,Busa_mallows,Busa14survey}, and regret minimisation \citep{Yue+12,SAVAGE,Zoghi+14RUCB,Ailon+14,Komiyama+15,DTS}.
In contrast, bandit learning with combinatorial, subset-wise preferences, though a natural and practical generalisation, has not received a commensurate treatment. 

There have been a few attempts in the batch (i.e., non-adaptive) setting for parameter estimation in utility-based subset choice models, e.g. Plackett-Luce or Thurstonian models \citep{Hajek+14,SueIcml+15,KhetanOh16,SueIcml+17}. 
In the online setup, a recent work by \citet{Brost+16} considers an extension of the dueling bandits framework where multiple arms are chosen in each round, but they receive comparisons for each pair, and there are no regret guarantees stated for their algorithm.  
Another similar work is DCM-bandits \citep{DCM}, where a list of $k$ distinct items are offered at each round and the users choose one or more from it scanning the list from top to bottom. However due to this cascading nature of their feedback model, this is also not strictly a relative subset-wise preference model unlike ours, since the utility or attraction weight of an item is assumed to be independently drawn, and so their learning objective differs substantially. 

A related body of literature lies in dynamic assortment selection, where the goal is to offer a subset of items to customers in order to maximise expected revenue. 
A specific, bandit (online) counterpart of this problem has been studied in the recent work of Agrawal et al. \cite{Agrawal+16,Agrawal+17TS}, although it takes items' prices into account due to which their notion of the `best subset' is rather different from our `benchmark subset', and the two settings are incomparable in general. More specifically, in this setting,
\begin{enumerate}
\item Their assumption of a {\em no-purchase} option, say item-$0$, necessarily present in every set and having the {\em known} and {\em highest} MNL parameter value $\theta_0 = 1$, is crucial for their algorithm design as well as the regret analysis --- more specifically this helps them to estimate the MNL model parameters easily. We however do not make this assumption, due to which it is more challenging to estimate the MNL model parameters in our case. This is also precisely the reason why the algorithm of \citet{Agrawal+16} cannot be directly applied for solving our problem.
\item The regret objective boils down to the top-$k$ best arm identification problem when all item prices are same, say $r_i = 1, \forall i \in [n]$. So in a sense we actually solve a special case of the assortment selection objective -- the top $k$ item(s) -- but without assumptions on the no-purchase item with known highest parameter value.
\item \citet{Agrawal+16} show gap independent $\tilde O(\sqrt{n T})$ regret for their algorithm and this is later improved to  gap-dependent $O(n^2 \ln T)$ regret \citet{Agrawal+17Dyn}; however, the latter guarantee is suboptimal by a factor of $n$, whereas we show tightness of the regret performance of our proposed algorithms by proving matching lower bound guarantees.
\end{enumerate}

Some recent work addresses the probably approximately correct (PAC) version of the best arm(s) identification problem from subsetwise preferences \cite{ChenSoda+18,Ren+18}, which is qualitatively different than the optimisation objective considered here. 
%
The work which is perhaps closest in spirit to ours is that of \citet{SG18}, but they consider a much more elementary subset choice model based on pairwise preferences, unlike the standard MNL model rooted in choice theory. \citet{Sui+17} also address a similar  problem; however, a key difference lies in the feedback which consists of outcomes of one or more pairs from the played subset, as opposed to our winner or \tf\, which is often practical. 

Lastly, like the dueling bandit, our more general MNL regret problem can be viewed as a stochastic partial monitoring problem \citep{bartok2011minimax}, in which the reward or loss of a subset play is not directly observed; instead, only stochastic feedback depending on the subset's parameters is observed. Moreover, under one of the regret structures we consider (\objbest, Sec. \ref{sec:alg_wiwf}), playing the optimal subset (the single item with the highest value) yields no useful information.

\section{Properties of MNL model}
\label{app:pl_prp}

\begin{defn}[Independence of Irrelevant Alternatives (IIA) property]
\label{def:iia}
A choice model is said to possess the {\it Independence of Irrelevant Attributes (IIA)} property if the ratio of probabilities of choosing any two items, say $i_1$ and $i_2$ from within any choice set $S \ni {i_1,i_2}$ is independent of a third alternative $j$ present in $S$ \citep{IIA-relevance16}. More specifically,
$
\frac{Pr(i_1|S_1)}{Pr(i_2|S_1)} = \frac{Pr(i_1|S_2)}{Pr(i_2|S_2)} \text{ for any two distinct subsets } S_1,S_2 \subseteq [n]
$ 
that contain $i_1$ and $i_2$. One such example is the MNL choice model as follows from Defn. \ref{def:mnl_thet}. 
\end{defn}


IIA turns out to be very valuable in estimating the parameters of a PL model, with high confidence, via \emph{Rank-Breaking} -- the idea of extracting pairwise comparisons from (partial) rankings and applying estimators on the obtained pairs, treating each comparison independently, as described below. 

\begin{defn}[\textbf{Rank-Breaking} \cite{AzariRB+14,KhetanOh16}]
\label{def:rb}
This is a procedure of deriving pairwise comparisons from multiwise (subsetwise) preference information. Formally, given any set $S \subseteq [n]$, $m \le |S| < n$, if $\bsigma \in \bSigma_{S}^m$ denotes a possible \tf\, of $S$, \rb\, considers each item in $S$ to be beaten by its preceding items in $\bsigma$ in a pairwise sense and extracts out total $\sum_{i = 1}^{m}(k-i) = \frac{m(2k-m-1)}{2}$ such pairwise comparisons. For instance, given a full ranking of a set of $4$ elements $S = \{a,b,c,d\}$, say $b \succ a \succ c \succ d$, Rank-Breaking generates the set of $6$ pairwise comparisons: $\{(b\succ a), (b\succ c), (b\succ d), (a\succ c), (a\succ d), (c\succ d)\}$. Similarly, given the ranking of only $2$ most preferred items say $b \succ a$, it yields the $5$ pairwise comparisons $(b, a\succ c),(b,a\succ d)$ and $(b\succ a)$ etc. See Line $10$ of Algorithm \ref{alg:mm} for example.
\end{defn}

Owning to the IIA property of \mnl\ model, one can show the following guarantee on the empirical pairwise estimates $\hp_{ij}(T) = \frac{n_i(T)}{n_{ij}(T)}$ obtained via \rb\ on MNL based subsetwise preferences:

\begin{restatable}[\cite{SGwin18}]
{lem}{plsimulator}
\label{lem:pl_simulator}
Consider a \mnl \, model, and fix two distinct items $i,j \in [n]$. Let $S_1, \ldots, S_T$ be a sequence of (possibly random) subsets of $[n]$ of size at least $2$, where $T$ is a positive integer, and $i_1, \ldots, i_T$ a sequence of random items with each $i_t \in S_t$, $1 \leq t \leq T$, such that for each $1 \leq t \leq T$, (a) $S_t$ depends only on $S_1, \ldots, S_{t-1}$, and (b) $i_t$ is distributed as the Plackett-Luce winner of the subset $S_t$, given $S_1, i_1, \ldots, S_{t-1}, i_{t-1}$ and $S_t$, and (c) $\forall t: \{i,j\} \subseteq S_t$ with probability $1$. Let $n_i(T) = \sum_{t=1}^T \1(i_t = i)$ and $n_{ij}(T) = \sum_{t=1}^T \1(\{i_t \in \{i,j\}\})$. Then, for any positive integer $v$, and $\eta \in (0,1)$,
\[Pr\left( \frac{n_i(T)}{n_{ij}(T)} - \frac{\theta_i}{\theta_i + \theta_j} \ge \eta, \; n_{ij}(T) \geq v \right) \vee \, Pr\left( \frac{n_i(T)}{n_{ij}(T)} - \frac{\theta_i}{\theta_i + \theta_j} \le -\eta, \; n_{ij}(T) \geq v \right) \leq e^{-2v\eta^2}. \]
\end{restatable}

\begin{rem}
Above lemma is crucially used in proving the regret bounds of our proposed algorithms (Alg. \ref{alg:mm} and \ref{alg:kmm}), in particular see the derivation of Lem. \ref{lem:cdelta}.
\end{rem}


\section{Supplementary for Sec. \ref{sec:res_bi}}


\subsection{Algorithm Pseudocode for \objbest}
\label{app:alg_wiwf}

\begin{center}
\begin{algorithm}[H]
   \caption{\textbf{\algmm}}
   \label{alg:mm}
\begin{algorithmic}[1]
   \STATE {\bfseries init:} $\alpha > 0.5$, $\W \leftarrow [0]_{n \times n}$, $\cB_0 \leftarrow \emptyset$ 
   \FOR{$t = 1,2,3, \ldots, T$} 
	\STATE Set $\bN = \W + \W^\top$, and $\hat{\P} = \frac{\W}{\bN}$. Denote $N = [n_{ij}]_{n \times n}$ and $\hat P = [\hp_{ij}]_{n \times n}$.
	\STATE Define $u_{ij} = \hp_{ij} + \sqrt{\frac{\alpha\ln t}{n_{ij}}}, \, \forall i,j \in [n], i\neq j$, $u_{ii} = \frac{1}{2}, \, \forall i \in [n]$. $\textbf{U} = [u_{ij}]_{n \times n}$ 
	\STATE $\cC_t \leftarrow \{i \in [n] ~|~ u_{ij} > \frac{1}{2}, \, \forall j \in [n]\setminus\{i\}\}$;  $\cB_t \leftarrow \cC_t \cap \cB_{t-1}$ 
	\STATE \textbf{if} $|\cC_t| = 1$, \textbf{then} set $\cB_t \leftarrow \cC_t$, $S_t \leftarrow \cC_t$, and go to Line $9$
	\STATE \textbf{if $\cB_t \neq \emptyset$} \textbf{then} set $S_t \leftarrow \cB_t$, \textbf{else} select any item $a \in \cC_t$, and set $S_t \leftarrow \{a\}$ 
	\STATE $S_t \leftarrow S_t ~\cup \,$ \algbld$(\textbf{U},S_t,[n]\sm S_t, m)$
	\STATE Play $S_t$, and receive: $\bsigma_t \in \bSigma_{S_t}^m$ 
   {\STATE \label{line:mm_rb} $W(\sigma_t(k'),i) \leftarrow W(\sigma_t(k'),i) + 1 ~~ \forall i \in S_t \sm \sigma_t(1:k')$} for all $k' = 1,2, \ldots, \min(|S_t|-1,m)$
   \ENDFOR
\end{algorithmic}
\end{algorithm}
\end{center}

\begin{center}
\begin{algorithm}[H]
   \caption{\textbf{\algbld}\,$(\textbf{U},S,I,\ell)$}
   \label{alg:bld}
\begin{algorithmic}[1]
   \STATE {\bfseries input:} $\textbf{U}$: UCB matrix of $\hat \P$, S: Set to build, I: pool of items $I$, $\ell > 0$: Number of items to draw
    \STATE $\cC \leftarrow \{i \in I ~|~ u_{ij} > \frac{1}{2}, \, \forall j \in I\sm\{i\}\}$ 
    \WHILE{$|\cC| < \ell$} 
    \STATE $S \leftarrow S \cup \cC$; $I \leftarrow I \sm \cC$; $\cC \leftarrow \{i \in I ~|~ u_{ij} > \frac{1}{2}, \, \forall j \in I\sm\{i\}\}$; $\ell \leftarrow \ell - |\cC|$ 
    \ENDWHILE
	\FOR{$k' = 2, 3,\ldots,\ell$}
	\STATE $a \leftarrow \underset{c \in I \sm S}{\arg \max}\Big[ \min_{i \in S}u_{ci} \Big]$; $S \leftarrow S \cup\{a\}$ 
	\ENDFOR
	\STATE {\bfseries return:} $S$
\end{algorithmic}
\end{algorithm}
\end{center}
\vspace*{-10pt}

\subsection{Restating the change of  measure Lemma 1 of \citet{Kaufmann+16_OnComplexity}}
\label{app:gar}
\begin{restatable}[\cite{Garivier+16}]{lem}{garivier}
\label{lem:gar16}
Given any bandit instance $(A,\bmu)$, with $A$ being the arm set of MAB, and $\bmu = \{\mu_i, ~\forall i \in A \}$ being the set of reward distributions associated to $A$ with arm $1$ having the highest expected reward, for any suboptimal arm $a \in A\sm\{1\}$, consider an altered bandit instance $\bmu^a$ with $a$ being the (unique) optimal arm (the one with highest expected reward) for $\bmu^a$, and let $\bmu$ and $\bmu^a$ be mutually absolutely continuous for all $a \in A\sm\{1\}$. At any round $t$, let $A_t$ and $Z_t$ denote the arm played and the observation (reward) received, respectively. Let $\cF_t = \sigma(A_1,Z_1,\ldots,A_t,Z_t)$ be the sigma algebra generated by the trajectory of a sequential bandit algorithm upto round $t$. Then, for any $\cF_T$-measurable random variable $Z$ with values in $[0, 1]$ it satisfies:

{
$ 
\sum_{i \in A}\E_{\bmu}[N_i(T)]KL(\mu_i, \mu^a_i) \ge kl(\E_{\bmu}[Z], \E_{\bmu^a}[Z]),
$ 
}
where $N_i(T)$ denotes the number of pulls of arm $i \in [n]$ in $T$ rounds, KL is the Kullback-Leibler divergence between distributions, and $kl(p,q)$ is the Kullback-Leibler divergence between Bernoulli distributions with parameters $p$ and $q$.
\end{restatable}


\subsection{Proof of Thm. \ref{thm:lb_wiwf}}
\label{app:lb_wiwf}

\lbwiwf*

\begin{proof}
The foundation of the current lower bound analysis stands on the ground on constructing \mnl\, instances, and slightly modified versions of it such that no algorithm can achieve \nr\, property on these instances without incurring $\Omega(n\ln T)$ regret. We describe the our constructed problem instances below:

Consider an \mnl\, instance with the arm (item) set $A$ containing all subsets of sizes $1,2,\ldots$ upto $k$ of $[n]$: $A = \{S = (S(1), \ldots S(k')) \subseteq [n] \mid k' \in [k]\}$. Let MNL$(n,\btheta^1)$ be the true distribution associated to the bandit arms $[n]$, given by the MNL parameters $\btheta^1 = (\theta_1^1,\ldots,\theta_n^1)$, such that $\theta_1^1 > \theta_i^1, \, \forall i \in [n]\setminus\{1\}$ such that,
\begin{align*}
\textbf{True Instance: } \text{MNL}(n,\btheta^1): \theta_1^1 > \theta_2^1 = \ldots = \theta_n^1 = \theta ~~\text{(say)}.
\end{align*}

for some $\theta \in \R_+$. We moreover denote $\Lambda = (\theta_1^1 - \theta)$. Clearly, the \bi\, of MNL$(n,\btheta^1)$ is $a^* = 1$. Now for every suboptimal item $a \in [n]\setminus \{1\}$, consider the altered problem instance MNL$(n,\btheta^a)$ such that:
\begin{align*}
\textbf{Instance a: } \text{MNL}(n,\btheta^a): \theta_a^a = \theta_1^1 + \epsilon = \theta + (\Lambda + \epsilon); ~\theta_i^a = \theta_i^1, ~~\forall i \in [n]\sm \{a\}
\end{align*}

for some $\epsilon > 0$. Clearly, the \bi\, of MNL$(n,\btheta^a)$ is $a^* = a$. Note that, for problem instance MNL$(n,\btheta^a) \, a \in [n]$, the probability distribution associated to arm $S \in A$ is given by
\[
p^a_S \sim Categorical(p_1, p_2, \ldots, p_k), \text{ where } p_i = Pr(i|S) = \frac{\theta_i^a}{\sum_{j \in S}\theta_j^a}, ~~\forall i \in [k], \, \forall S \in A, \, \forall a \in [n],
\]
since recall that $Pr(i|S)$ is as defined in Defn. \ref{def:mnl_thet}. 
Now applying Lem. \ref{lem:gar16} we get, 

\begin{align}
\label{eq:FI_a}
\sum_{\{S \in A \sm \{a\} \mid a \in S\}}\E_{\btheta^1}[N_S(T)]KL(p^1_S, p^a_S) \ge {kl(\E_{\btheta^1}[Z], \E_{\btheta^a}[Z])}.
\end{align}

The above result holds from the straightforward observation that for any arm $S \in \cA$ with $a \notin S$, $p^1_S$ is same as $p^a_S$, hence $KL(p^1_S, p^a_S)=0$, $\forall S \in A, \,a \notin S$ or if $S = \{a\}$. 

For the notational convenience we will henceforth denote $S^a = \{S \in \cA \sm \{a\} \mid a \in S\}$. 
Now let us analyse the right hand side of \eqref{eq:FI_a}, for any set $S \in S^a$. We further denote $\Lambda' = \Lambda + \epsilon = (\theta_1^1 - \theta) + \epsilon$, $k' = |S| \in [k]$, $r = \1(1 \in S)$, $q = (k'-r)$, and $\theta_S^a = \sum_{i \in S}\theta_i^a$ for any $a \in [n]$.

Note that by construction of above problem instances we can further derive that for any $i \in S$:
\begin{align*}
p^1_S(i) = 
\begin{cases} 
\frac{r\theta^1_1}{\theta^1_S} = \frac{\theta+\Lambda}{\theta|S| + r\Lambda}, \text{ such that } i = 1,\\
\frac{\theta}{\theta^1_S} = \frac{\theta}{\theta|S| + r\Lambda}, \text{ otherwise. }
\end{cases}
\end{align*}

On the other hand, for problem \textbf{Instance-a}, we have that: 

\begin{align*}
p^a_S(i) = 
\begin{cases} 
\frac{r\theta^1_1}{\theta^1_S + \Lambda'} = \frac{\theta+\Lambda}{\theta|S| + \Lambda(1+r) + \epsilon}, \text{ such that } i = 1,\\
\frac{\theta^1_1 + \epsilon}{\theta^1_S + \Lambda'} = \frac{\theta + \Lambda + \epsilon}{\theta|S| + \Lambda(1+r) + \epsilon}, \text{ such that } i = a,\\
\frac{\theta}{\theta^1_S + \Lambda'} = \frac{\theta}{\theta|S| + \Lambda(1+r) + \epsilon}, \text{ otherwise. }
\end{cases}
\end{align*}

Now using the following upper bound on $KL(\p,\q) \le \sum_{x \in \X}\frac{p^2(x)}{q(x)} -1$, $\p$ and $\q$ be two probability mass functions on the discrete random variable $\X$ \cite{klub16}, we get:

\begin{align}
\label{eq:lb_wiwf_kl}
\nonumber KL(p^1_S, p^a_S) & \le \sum_{i \in S\sm\{a\}}\bigg(\frac{\theta_i^1}{\theta_S^1}\bigg)^2\bigg( \frac{\theta_S^a}{\theta_i^a}\bigg) + \bigg(\frac{\theta_a^1}{\theta_S^1}\bigg)^2\bigg( \frac{\theta_S^a}{\theta_a^a}\bigg) - 1\\
\nonumber & = \sum_{i \in S\sm\{a\}}\bigg(\frac{\theta_i^1}{\theta_S^1}\bigg)^2\bigg( \frac{\theta_S^1 + \Lambda'}{\theta_i^1}\bigg) + \bigg(\frac{\theta_a^1}{\theta_S^1}\bigg)^2\bigg( \frac{\theta_S^1 + \Lambda'}{\theta_a^1 + \Lambda'}\bigg) - 1\\
\nonumber & = \bigg(\frac{\theta_S^1 + \Lambda'}{(\theta_S^1)^2} \bigg) \bigg(\sum_{i \in [n]\sm\{a\}}\theta^1_i + \frac{(\theta_a^1)^2}{\theta_a^1 + \Lambda'}\bigg) - 1\\
\nonumber & = \bigg(\frac{\theta_S^1 + \Lambda'}{(\theta_S^1)^2} \bigg) \bigg( \frac{\theta_a^1\theta_S^1 + \Lambda'(\theta_S^1 - \theta_a^1)}{\theta_a^1 + \Lambda'}\bigg) - 1 ~~~\bigg[\text{replacing } \sum_{i \in [n]\sm\{a\}}\theta^1_i = (\theta_S^1 - \theta_a^1)\bigg]\\
\nonumber & = \frac{ \Lambda'^2(\theta_S^1 - \theta_a^1)}{(\theta_S^1)^2(\theta_a^1 + \Lambda')} \le \frac{ \Lambda'^2}{\theta_S^1(\theta_a^1 + \Lambda')} = \frac{ \Lambda'^2}{\theta_S^1(\theta_1^1 + \epsilon)}\\
& = \frac{ (\Lambda+\epsilon)^2}{(\theta|S|+r\Lambda)(\theta_1^1 + \epsilon)} \le \frac{ (\Lambda+\epsilon)^2}{\theta|S|(\theta_1^1 + \epsilon)}
\end{align}

Let us now analyze the left hand side of \eqref{eq:FI_a}, with $Z = \frac{N_{a}(T)}{T}$, where $N_a(T)$ simply denotes the number of times the singleton set containing item $\{a\}$ is played by $\cA$, for any suboptimal item $a \in [n]\sm\{1\}$. Thus we get,
\begin{align}
\label{eq:win_lb1}
{kl(\E_{\btheta^1}[Z], \E_{\btheta^a}[Z])} \ge \Bigg( 1 - \frac{\E_{\btheta^1}[N_1(T)]}{T} \Bigg)\ln \frac{T}{T - \E_{\btheta^a}[N_1(T)]} - \ln 2,
\end{align}

where the inequality follows from the fact that for all $(p, q) \in [0, 1]^2$, $kl(p,q) = p\ln \frac{1}{q} + (1-p)\ln\frac{1}{1-q} + (p \ln p + (1-p)\ln (1-p))$, and $p\ln \frac{1}{q} \ge 0$, $(p \ln p + (1-p)\ln (1-p)) \ge -\ln 2$. 

But now owing to the \nr\, property (see Defn. \ref{def:con}) of Algorithm $\cA$, we have $\E_{\btheta^1}[N_{a}(T)] = o(T^\alpha)$ and $T - \E_{\btheta^a}[N_{a}(T)] = \E_{\btheta^a}[\sum_{S \in A, S \neq \{a\}}N_{S}(T)] = o(T^\alpha)$, $0 < \alpha \le 1$. Thus from \eqref{eq:win_lb1}, we get

\begin{align*}
\lim_{T \to \infty}\frac{kl(\E_{\btheta^1}[Z], \E_{\btheta^a}[Z])}{\ln T} & \ge \lim_{T \to \infty}\frac{1}{\ln T}\bigg[\Bigg( 1 - \frac{\E_{\btheta^1}[N_a(T)]}{T} \Bigg)\ln \frac{T}{T - \E_{\btheta^a}[N_a(T)]} - \ln 2\bigg]\\
& = \lim_{T \to \infty}\frac{1}{\ln T}\bigg[\Bigg( 1 - \frac{o(T^\alpha)}{T} \Bigg)\ln \frac{T}{T^\alpha} - \ln 2\bigg] = (1-\alpha).
\end{align*}

Combining above with \eqref{eq:lb_wiwf_kl} we get:

\begin{align}
\label{eq:win_lb2}
\nonumber & \lim_{T \to \infty}\frac{1}{\ln T}\sum_{\{S \in S^a\}}\E_{\btheta^1}[N_S(T)]KL(p^1_S, p^a_S) \ge (1-\alpha)\\
\nonumber & \implies \lim_{T \to \infty}\frac{1}{\ln T}\sum_{\{S \in S^a\}}\E_{\btheta^1}[N_S(T)]\frac{ \Lambda'^2}{\theta|S|(\theta_1^1 + \epsilon)} \ge (1-\alpha)\\
& \implies \lim_{T \to \infty}\frac{1}{\ln T}\sum_{\{S \in S^a\}}\E_{\btheta^1}[N_S(T)]\frac{\Lambda'}{|S|} \ge (1-\alpha)\frac{\theta(\theta_1^1 + \epsilon)}{\Lambda'}
\end{align}

Now applying \eqref{eq:win_lb2} for each modified bandit \textbf{Instance-$\btheta^a$}, and summing over $(n-1)$ suboptimal items $a \in [n]\setminus \{1\}$ we get,

\begin{align}
\label{eq:win_lb2.5}
\lim_{T \to \infty}\frac{1}{\ln T} \sum_{a = 2}^{n}\sum_{\{S \in S^a\}}\E_{\btheta^1}[N_S(T)]\frac{\Lambda'}{|S|} \ge (1-\alpha)\theta(\theta_1^1 + \epsilon)\frac{(n-1)}{\Lambda'}
\end{align}

Now recall that regret of $\cA$ on the true instance MNL($n,\btheta^1$), is given by: $R_T^1(\cA) = \sum_{t=1}^T\Big(\sum_{i \in S_t}\frac{(\theta_1^1 - \theta_i^1)}{|S_t|}\Big)$.
But this can be equivalently written as:

\begin{align}
\label{eq:lb_wiwf_reg}
\nonumber \E_{\btheta^1}[R_T^1(\cA)] & = \E_{\btheta^1}\Big[\sum_{t=1}^T\sum_{i \in S_t}\frac{(\theta_1^1 - \theta_i^1)}{|S_t|}\Big]\\
\nonumber & = \E_{\btheta^1}\Big[\sum_{t=1}^T\sum_{S \in A}\1(S_t=S)\sum_{a = 2}^n\1(a \in S)\frac{(\theta_1^1 - \theta_a^1)}{|S_t|}\Big]\\
\nonumber & = \E_{\btheta^1}\bigg[ \sum_{a=2}^{n}\sum_{t = 1}^{T}\sum_{S \in A}\1(S_t = S)\1(a \in S)\frac{(\theta_1^1 - \theta_a^1)}{|S|} \bigg]\\
\nonumber & = \sum_{a=2}^{n}\sum_{t = 1}^{T}\E_{\btheta^1}\bigg[ \sum_{S \in A}\1(S_t = S)\1(a \in S)\frac{(\theta_1^1 - \theta_a^1)}{|S|} \bigg]\\
\nonumber & = \sum_{a=2}^{n}\sum_{S \in A}\E_{\btheta^1}\bigg[ \sum_{t = 1}^{T}\1(S_t = S)\1(a \in S)\frac{(\theta_1^1 - \theta_a^1)}{|S|} \bigg]\\
\nonumber & = \sum_{a=2}^{n}\sum_{S \in A}\bigg[ \E_{\btheta^1}[N_S(T)]\1(a \in S)\frac{(\theta_1^1 - \theta)}{|S|} \bigg]\\
& = \sum_{a=2}^{n}\sum_{\{S \in A \mid a \in S\}} \E_{\btheta^1}[N_S(T)]\frac{\Lambda}{|S|} 
\end{align}

Then combining \eqref{eq:lb_wiwf_reg} with \eqref{eq:win_lb2.5} we get and taking $\epsilon \to 0$:
\begin{align*}
\lim_{T \to \infty}\frac{1}{\ln T}\E_{\btheta^1}[R_T^1(\cA)] & \ge \lim_{T \to \infty}\frac{1}{\ln T} \sum_{a = 2}^{n}\sum_{\{S \in S^a\}}\E_{\btheta^1}[N_S(T)]\frac{\Lambda}{|S|} \\ 
& \ge \, (1-\alpha)\theta(\theta_1^1)\frac{(n-1)}{\Lambda} = (1-\alpha)\theta_1^1\frac{(n-1)}{(\frac{\theta_1^1}{\theta}-1)}.
\end{align*}
Finally, since $\alpha$ is a fixed constant in $(0,1]$, above construction shows the existence of a \mnl\, problem instance, precisely MNL($n,\btheta^1$), such that for large $T$, $\E_{\btheta^1}[R_T^{1}] = \Omega\Bigg( \dfrac{\theta_1^1}{\Big(\frac{\theta_1^1}{\theta}-1\Big)}{(n-1)}\ln T  \Bigg)$, which concludes the proof.
\end{proof}  

\subsection{An alternate version of the regret lower bound (Thm. \ref{thm:lb_witf}) with pairwise preference-based instance complexities}
\label{app:alt_wilb}

\begin{restatable}[Alternate version of Thm. \ref{thm:lb_witf} with pairwise preference based instance complexities]{thm}{lbwiwff}
\label{thm:lb_wiwf2}
For any \nr\, algorithm $\cA$ for \objbest\, with \wf, there exists a problem instance of \mnl\, model, such that the expected regret incurred by $\cA$ on it satisfies 
$
\underset{T \to \infty }{\lim \inf}\,\E_{\btheta}\Big[\frac{R_T^{1}(\cA)}{\ln T}\Big] \ge \frac{\theta_{a^*}}{4\Big(\underset{i \in [n]\sm\{a^*\}}{\min} p_{a^*,i} - 0.5 \Big)}(n-1), 
$ 
where $p_{ij} := Pr(i|\{i,j\}) = \frac{\theta_i}{\theta_i+\theta_j}\, \forall i,j \in [n]$, and $\E_{\btheta}[\cdot]$, $a^*$ are same as that of Thm. \ref{thm:lb_wiwf}. Thus the only difference lies in terms of the instance dependent complexity term (`gap') which is now expressed in terms of pairwise preference of the best item $a^*$ over the second best item: $\Big(\underset{i \in [n]\sm\{a^*\}}{\min} p_{a^*,i} - 0.5 \Big)$.
\end{restatable}


\begin{proof}
Firstly, is easy to note that $\underset{i \in [n]\sm\{a^*\}}{\arg\min} \Big( p_{a^*,i} - 0.5 \Big) = \underset{i \in [n]\sm\{a^*\}}{\arg\max}\theta_i =: b$ (say). 
The proof now follows from the fact that
\begin{align*}
p_{a^*b} - 0.5 = \frac{\theta_{a^*} - \theta_{b}}{2(\theta_{a^*} + \theta_{b})} \le \frac{\theta_{a^*} - \theta_{b}}{4\theta_{b}} ~~~(\text{since} \theta_b \le \theta^{a^*})
\end{align*}

Thus using the lower bound from Thm. \ref{thm:lb_witf}, one can further derive
\begin{align*}
\underset{T \to \infty }{\lim \inf}\,\frac{1}{\ln T}\E_{\btheta}\Big[R_T^{1}(\cA)\Big] \ge \frac{4\theta_{a^*}\theta_{b}}{4\Big(\theta_{a^*} - \theta_b\Big)}(n-1) \ge \frac{\theta_{a^*}}{4\Big(\underset{i \in [n]\sm\{a^*\}}{\min} p_{a^*,i} - 0.5 \Big)}(n-1),
\end{align*}
which proves the claim.
\end{proof}

\subsection{Proof of Thm. \ref{thm:lb_witf}}
\label{app:lb_witf}

\lbwitf*

\begin{proof}
The proof proceeds almost same as the proof of Thm. \ref{thm:lb_wiwf}, the only difference lies in the analysis of the KL-divergence terms with \tf. 

Consider the exact same \mnl\, instances, MNL$(n,\btheta^a)$ we constructed for Thm. \ref{thm:lb_wiwf}. It is now interesting to note that how \tf\, affects the KL-divergence analysis, precisely the KL-divergence shoots up by a factor of $m$ which in fact triggers an $\frac{1}{m}$ reduction in regret learning rate.
Note that for \tf\, for any problem instance MNL$(n,\btheta^a), \, a \in [n]$, each $k$-set $S \subseteq [n]$ (such that $|S| = k$) is associated to ${k \choose m} (m!)$ number of possible outcomes, each representing one possible ranking of set of $m$ items of $S$, say $S_m$. Also the probability of any permutation $\bsigma \in \bSigma_{S_m}$ is given by
$
p^a_S(\bsigma) = Pr_{\btheta^a}(\bsigma|S),
$
where $Pr_{\btheta^a}(\bsigma|S)$ is as defined for \tf\, (in Sec. \ref{sec:feed_mod}).
For ease of analysis let us first assume $1 \notin S$ and let $k' = |S|$ be the cardinality of $S$ and $m' = min(m,k)$. (Note if $m' \le m+1$ the corresponding \tf\, becomes a full ranking feedback on the entire $m'$ items). In this case we get

\begin{align*}
p^1_S(\bsigma) = \prod_{i = 1}^{m'}\frac{{\theta_{\sigma(i)}^1}}{\sum_{j = i}^{m'}\theta_{\sigma(j)}^1 + \sum_{j \in S \setminus \sigma(1:m')}\theta_{\sigma(j)}^1} = \frac{1}{k'(k'-1)(k'-2)\cdots(k'-m+1)}, ~~ \forall \sigma \in \Sigma_{S}^{m'}.
\end{align*}

On the other hand, for problem \textbf{Instance-a}, we have that: 

\begin{align*}
p^a_S(\bsigma) 
& = \prod_{i = 1}^{m'}\frac{{\theta_{\sigma(i)}^a}}{\sum_{j = i}^{m'}\theta_{\sigma(j)}^a + \sum_{j \in S \setminus \sigma(1:m')}\theta_{\sigma(j)}^a}\\
& = 
\begin{cases} 
\frac{x}{(x+k'-1)(x+k'-2)\cdots(x+k'-i)(k'-i)(k'-i-1)\cdots(k'-m'+1)}, \text{ such that } \sigma(i) = a,\\
\frac{1}{(x+k'-1)(x+k'-2)\cdots(x+k'-i)(k'-i)(k'-i-1)\cdots(k'-m'+1)}, \text{ such that } a \notin \sigma(1:m'),
\end{cases}
\end{align*}
where we denote by $x = 1 + \frac{\Lambda'}{\theta}$, where recall that we denote $\Lambda' = \Lambda + \epsilon$. Similarly we can derive the probability distribution associated to sets including item $1$.

The important thing now to note is that $KL(p^1_S, p^a_S) = 0$ for any set $S \not\owns a$. Hence while comparing the KL-divergence of instances $\btheta^1$ vs $\btheta^a$, we need to focus only on sets containing $a$. Applying \emph{Chain-Rule} of KL-divergence, we now get

\begin{align}
\label{eq:lb_witf_kl}
\nonumber KL(p^1_S, p^a_S) = KL(p^1_S(\sigma_1), p^a_S(\sigma_1)) & + KL(p^1_S(\sigma_2 \mid \sigma_1), p^a_S(\sigma_2 \mid \sigma_1)) + \cdots \\ 
& + KL(p^1_S(\sigma_m \mid \sigma(1:m-1)), p^a_S(\sigma_m \mid \sigma(1:m-1))),
\end{align}
where we abbreviate $\sigma(i)$ as $\sigma_i$ and following the usual convention the notation $KL( P(Y \mid X),Q(Y \mid X)): = \sum_{x}Pr\Big( X = x\Big)\big[ KL( P(Y \mid X = x),Q(Y \mid X = x))\big]$ denotes the conditional KL-divergence. 
Moreover it is easy to note that for any $\sigma \in \Sigma_{S}^m$ such that $\sigma(i) = a$, we have $KL(p^1_S(\sigma_{i+1} \mid \sigma(1:i)), p^a_S(\sigma_{i+1} \mid \sigma(1:i))) := 0$, for all $i \in [m]$.

Now as derived in \eqref{eq:lb_wiwf_kl} in the proof of Thm. \ref{thm:lb_wiwf}, we have 
\[
KL(p^1_S(\sigma_1), p^a_S(\sigma_1)) \le \frac{ (\Lambda+\epsilon)^2}{\theta|S|(\theta_1^1 + \epsilon)}.
\]

To bound the remaining terms of \eqref{eq:lb_witf_kl},  note that for all $i \in [m-1]$
\begin{align*}
KL&(p^1_S(\sigma_{i+1} \mid \sigma(1:i)),  p^a_S(\sigma_{i+1} \mid \sigma(1:i))) \\
& = \sum_{\sigma' \in \Sigma_S^i}Pr(\sigma')KL(p^1_S(\sigma_{i+1} \mid \sigma(1:i))=\sigma', p^a_S(\sigma_{i+1} \mid \sigma(1:i))=\sigma')\\
& = 	\sum_{\sigma' \in \Sigma_S^i \mid a \notin \sigma'}\Bigg[\prod_{j = 1}^{i}\Bigg(\dfrac{\theta^1_{\sigma'_j}}{\theta_S^1 - \sum_{j'=1}^{j-1}\theta_{\sigma'_{j'}}}\Bigg)\Bigg]\dfrac{ \Lambda'^2}{(|S|-i)\theta(\theta_1^1 + \epsilon)} \\
& = \prod_{j = 1}^{i}(|S|-j)\frac{\theta^i}{\prod_{j = 1}^{i}(\theta(|S|-i+1) + \Lambda')}\frac{(\Lambda')^2}{(|S|-i)\theta(\theta_1^1 + \epsilon)}  = \frac{\cancel{\theta}}{(\theta|S|+ \Lambda')}\frac{\Lambda'^2}{\cancel{\theta}(\theta_1^1+ \epsilon)} \\
& = \frac{\Lambda'^2}{(\theta|S|+ \Lambda')(\theta_1^1+ \epsilon)},
\end{align*}

where for simplicity we assumed $1 \notin S$. It is easy to note that the similar analysis would lead to the same upper bound for sets $S$ containing $1$ as well. Thus applying above in \eqref{eq:lb_witf_kl} we get:

\begin{align}
\label{eq:lb_witf_kl2}
\nonumber KL(p^1_S, p^a_S) & = KL(p^1_S(\sigma_1) + \cdots + KL(p^1_S(\sigma_m \mid \sigma(1:m-1)), p^a_S(\sigma_m \mid \sigma(1:m-1))) \\
& \le \frac{m\Lambda'^2}{|S|\theta(\theta_1^1+ \epsilon)}.
\end{align}

Eqn. \eqref{eq:lb_witf_kl2} gives the main result to derive Thm. \ref{thm:lb_witf} as it shows an $m$-factor blow up in the KL-divergence terms owning to \tf. The rest of the proof follows exactly the same argument used in \ref{thm:lb_wiwf}. We add the steps below for convenience.
Firstly, considering $Z = \frac{N_{a}(T)}{T}$, in this case as well, one can show that: 

\begin{align*}
\lim_{T \to \infty}\frac{kl(\E_{\btheta^1}[Z], \E_{\btheta^a}[Z])}{\ln T} & \ge \lim_{T \to \infty}\frac{1}{\ln T}\bigg[\Bigg( 1 - \frac{\E_{\btheta^1}[N_a(T)]}{T} \Bigg)\ln \frac{T}{T - \E_{\btheta^a}[N_a(T)]} - \ln 2\bigg]\\
& = \lim_{T \to \infty}\frac{1}{\ln T}\bigg[\Bigg( 1 - \frac{o(T^\alpha)}{T} \Bigg)\ln \frac{T}{T^\alpha} - \ln 2\bigg] = (1-\alpha).
\end{align*}

Now combining above with \eqref{eq:lb_witf_kl2} we get:

\begin{align}
\label{eq:win_lb2k}
\nonumber & \lim_{T \to \infty}\frac{1}{\ln T}\sum_{\{S \in S^a\}}\E_{\btheta^1}[N_S(T)]KL(p^1_S, p^a_S) \ge (1-\alpha)\\
\nonumber & \implies \lim_{T \to \infty}\frac{1}{\ln T}\sum_{\{S \in S^a\}}\E_{\btheta^1}[N_S(T)]\frac{ m\Lambda'^2}{\theta|S|(\theta_1^1 + \epsilon)} \ge (1-\alpha)\\
& \implies \lim_{T \to \infty}\frac{1}{\ln T}\sum_{\{S \in S^a\}}\E_{\btheta^1}[N_S(T)]\frac{\Lambda'}{|S|} \ge (1-\alpha)\frac{\theta(\theta_1^1 + \epsilon)}{m\Lambda'}
\end{align}

Applying \eqref{eq:win_lb2k} for each modified bandit \textbf{Instance-$\btheta^a$}, and summing over $(n-1)$ suboptimal items $a \in [n]\setminus \{1\}$ we get,

\begin{align}
\label{eq:win_lb2.5k}
\lim_{T \to \infty}\frac{1}{\ln T} \sum_{a = 2}^{n}\sum_{\{S \in S^a\}}\E_{\btheta^1}[N_S(T)]\frac{\Lambda'}{|S|} \ge (1-\alpha)\theta(\theta_1^1 + \epsilon)\frac{(n-1)}{m\Lambda'}
\end{align}

Further recall that we derived earlier that
$
\nonumber \E_{\btheta^1}[R_T^1(\cA)] =  \sum_{a=2}^{n}\sum_{\{S \in A \mid a \in S\}} \E_{\btheta^1}[N_S(T)]\frac{\Lambda}{|S|}, 
$ 
using which combined with \eqref{eq:win_lb2.5k}, and taking $\epsilon \to 0$ we get:
\begin{align*}
\lim_{T \to \infty}\frac{1}{\ln T}\E_{\btheta^1}[R_T^1(\cA)] & \ge \lim_{T \to \infty}\frac{1}{\ln T} \sum_{a = 2}^{n}\sum_{\{S \in S^a\}}\E_{\btheta^1}[N_S(T)]\frac{\Lambda}{|S|} \\ 
& \ge \, (1-\alpha)\theta(\theta_1^1)\frac{(n-1)}{m\Lambda} = (1-\alpha)\theta_1^1\frac{(n-1)}{m(\frac{\theta_1^1}{\theta}-1)}.
\end{align*}
Now since $\alpha$ is a fixed constant in $(0,1]$, we thus prove the existence of a \mnl\, problem instance $\big($precisely MNL($n,\btheta^1$)$\big)$, such that for large $T$, $\E_{\btheta^1}[R_T^{1}] = \Omega\Bigg( \dfrac{\theta_1^1}{\Big(\frac{\theta_1^1}{\theta}-1\Big)}\frac{(n-1)}{m}\ln T  \Bigg)$, which concludes the proof.


\end{proof}  

\subsection{Proof of Thm. \ref{thm:whp_reg_mm}}
\label{app:whp_reg_mm}

\whpmm*

\begin{proof}
For the notational convenience we will assume $\theta_1 > \theta_2 \ldots \ge \theta_n$, so $a^* = 1$. We also use $\hat{p}_{ij}(t), n_{ij}(t)$ and $u_{ij}(t)$ to denote the values of the respective quantities at time iteration $t$, for any $t \in [T]$, just to be precise 
\[
u_{ij}(t) = \hat{p}_{ij}(t) + \sqrt{\frac{\alpha\ln t}{n_{ij}(t)}}, \, \forall i,j \in [n], i\neq j,
\]
and $u_{ii}(t) = \frac{1}{2}$ for all $i \in [n]$. We also find it convenient to denote 
\[
c_{ij}(t) = \sqrt{\frac{\alpha\ln t}{n_{ij}(t)}},
\text{ and } l_{ij}(t) = 1 - u_{ij}(t).
\]

We also denote $T_0(\delta) = 2f(\delta) + 2D\ln 2D$, where $D := \sum_{i < j}D_{ij}$.
We start with the following crucial lemma that analyzes the confidence bounds $[l_{ij}(t), u_{ij}(t)]$ on the pairwise probability estimates $\hat{p}_{ij}$ for each pair $(i,j)$, $i\neq j$.

\begin{lem}
\label{lem:cdelta}
Suppose $\P: = [p_{ij}]$ be the pairwise probability matrix associated to the underlying \mnl\, model, i.e. $p_{ij} = Pr(i|\{i,j\}) = \frac{\theta_i}{\theta_i + \theta_j}$. Then for any $\alpha > \frac{1}{2}$, $\delta \in (0,1)$,
\[
Pr\Big( \forall t > f(\delta), \forall i,j, ~~p_{ij} \in [l_{ij}(t), u_{ij}(t)] \Big) > (1-\delta)
\]
where $f(\delta) = \Big[ \frac{2\alpha n^2}{(2\alpha-1)\delta} \Big]^{\frac{1}{2\alpha-1}}$.
\end{lem}

\begin{proof}
The proof of this lemma is adapted from a similar result (Lemma 1) of \cite{Zoghi+14RUCB}.
Suppose $\cG_{ij}(t)$ denotes the event that at time $t$, $p_{ij} \in [l_{ij}(t), u_{ij}(t)], \, \forall i,j \in [n]$. $\cG^c_{ij}(t)$ denotes its complement.

{\bf Case 1: $(i = j)$ } Note that for any such that pair $(i,i)$, $\cG_{ii}(t)$ always holds true for any $t \in [T]$ and $i \in [n]$, as $p_{ii} = u_{ii} = l_{ii} = \frac{1}{2}$.

{\bf Case 2: $ (i \neq j)$ }
Recall from the definition of $u_{ij}(t)$ that $\cG_{ij}(t)$ equivalently implies at round $t$, $|\hp_{ij}(t) - p_{ij}| \le \sqrt{\frac{\alpha\ln (t)}{n_{ij}(t)}}, \, \forall i,j \in [n]$. 
Moreover, for any $t$ and $i,j$, $\cG_{ij}(t)$ holds if and only if $\cG_{ij}(t)$ as $|\hp_{ji}(t) - p_{ji}| = |(1-\hp_{ij}(t)) - (1-p_{ij})| = |\hp_{ij}(t) - p_{ij}| $.
Thus we will restrict our focus only to pairs $i < j$ for the rest of the proof.

Let $\tau_{ij}(n)$ the time step $t \in [T]$ when the pair $(i,j)$ was updated for the $n^{th}$ time. 
Clearly for any $n \in \N$, $\tau_{ij}(n+1) \ge \tau_{ij}(n)$ and $\tau_{ij}(n+k) > \tau_{ij}(n)$. For convenience of notation we use $F = f(\delta)$.
It is now straightforward to note that we want to find $F$ such that:

\begin{align}
\label{eq:conf_1}
\nonumber Pr\Big(\forall t > F, \forall i,j  & \text{ such that } i<j, ~\cG_{ij}(t)\Big) > (1-\delta) \text{ or equivalently,}\\
Pr \Big( \exists t > F & \text{ and atleast a pair } (i<j), \text{ with } ~\cG^c_{ij}(t) \Big) < \delta.
\end{align}

Further decomposing  the right hand side of above we get:

\begin{align*}
& \hspace{-10pt} Pr \Big( \exists t > F,i<j, \text{ such that } ~\cG^c_{ij}(t) \Big)\\
&\hspace{-10pt} \le \sum_{i < j} \Bigg [ Pr \Bigg( \exists n \ge 0,
\tau_{ij}(n) > F, |p_{ij} - \hp_{ij}(\tau_{ij}(n)) | > \sqrt{\frac{\alpha\ln (\tau_{ij}(n))}{n_{ij}(\tau_{ij}(n))}} \Bigg) \Bigg ]\\
&\hspace{-10pt} \le \sum_{i < j} \Bigg [ Pr \Bigg( \exists n \le F, \tau_{ij}(n) > F, ~|p_{ij} - \hp_{ij}(n) | > \sqrt{\frac{\alpha\ln (\tau_{ij}(n))}{n}} \Bigg)
 + Pr \Bigg( \exists n > F, ~|p_{ij} - \hp_{ij}(n) | > \sqrt{\frac{\alpha\ln (\tau_{ij}(n))}{n}} \Bigg) \Bigg ],
\end{align*}
where $\hp(n) = \frac{w_{ij}(\tau_{ij}(n))}{w_{ij}(\tau_{ij}(n)) + w_{ij}(\tau_{ji}(n))}$ is the frequentist estimate of $p_{ij}$ after $n$ comparisons between arm $i$ and $j$. 
Now the above inequality can be further upper bounded as:
\begin{align*}
& Pr \Big(\exists t > F,i<j, \text{ such that } ~\cG^c_{ij}(t) \Big)\\
& \le \sum_{i < j} \Bigg [ Pr \Bigg( \exists n \le F, \tau_{ij}(n) > F, ~|p_{ij} - \hp_{ij}(n) | > \sqrt{\frac{\alpha\ln (F)}{n}} \Bigg)
 + Pr \Bigg( \exists n > F, ~|p_{ij} - \hp_{ij}(n) | > \sqrt{\frac{\alpha\ln (n)}{n}} \Bigg) \Bigg ],
\end{align*}
since in the second term $\tau_{ij}(n)>F$, and for the third term $n < \tau_{ij}(n)$ since at a particular time iteration, any pair $(i,j)$, can be updated at most once, implying $n \le \tau_{ij}(n)$. 
Using Lem. \ref{lem:pl_simulator} we now get:
\begin{align*}
& Pr \Big( \exists t > F,i<j, \text{ such that } ~\cG^c_{ij}(t) \Big)\\
& \le \sum_{i < j} \Bigg [ \sum_{n=1}^{F}2e^{-2n\frac{\alpha\ln F}{n}} + \sum_{n=F+1}^{\infty}2e^{-2n\frac{\alpha\ln n}{n}} \Bigg]\\
& = \frac{n(n-1)}{2}\Bigg [ 2\sum_{n = 1}^{F}\frac{1}{F^{2\alpha}} + \sum_{n = F+1}^{\infty}\frac{2}{n^{2\alpha}} \Bigg]\\
& \le \frac{n^2}{F^{2\alpha-1}} + n^2\int_{F}^{\infty}\frac{dx}{x^{2\alpha}} 
\le \frac{n^2}{F^{2\alpha-1}} - \frac{n^2}{(1-2\alpha)F^{2\alpha-1}}
 = \frac{(2\alpha)n^2}{(2\alpha-1)F^{2\alpha-1}}. 
\end{align*}
Now from \eqref{eq:conf_1}, we want to find $F$ such that 
\[
\frac{2\alpha n^2}{(2\alpha-1)F^{2\alpha-1}} \le \delta,
\]
which suffices by setting $F = \Big[ \frac{2\alpha n^2}{(2\alpha-1)\delta} \Big]^{\frac{1}{2\alpha-1}}$, and recall that we assumed $f(\delta) = F$, for any given $\delta \in [0,1]$, which concludes the claim.
\end{proof}


Lem. \ref{lem:cdelta} ensures the termination of the \rnd\, phase.
We now proceed to analyse \prog\, phase which shows that the set $\cB_t$ captures the \bi\, $a^* = 1$ `soon after' $f(\delta)$ within a constant number of rounds $T_0(\delta)$ which is independent of $T$ (see Lem. \ref{lem:t_hat_delta}). Once the $1$ is captured in $\cB_t$, the algorithm goes into \sat\, phase where the suboptimal items can not stay too long in the set of potential \bi s $\cC_t$, and thus the regret bound follows (Lem. \ref{lem:mm_reg_postT0}). More formally, the rest of the proof follows based on the following main observations: 
\\
\noindent \textbf{In \prog:}

\begin{itemize}

\item \textbf{Observation $1$:} At any iteration, the set $\cB_t$ is either singleton or an empty set.

\item \textbf{Observation $2$:} 
For any $\delta \in (0,1)$, suppose $T_0(\delta) := \min_{t > f(\delta)}\cB_t = \{1\}$, then for any $t > T_0(\delta)$, $\cB_t = \{1\}$. 

\item \textbf{Observation $3$:} $T_0(\delta)$ is not far from $f(\delta)$ (see Lem. \ref{lem:t_hat_delta} which holds due to Lem. \ref{lem:n_1i} and \ref{lem:n_ij})

\end{itemize}

\noindent \textbf{In \sat:}

\begin{itemize}

\item \textbf{Observation $4$:} After $T_0(\delta)$, $\cB_t = \{1\}$ thereafter, and thus it is always played in $S_t$, i.e. $1 \in S_t$ for all $t > T_0(\delta)$. Now the suboptimal items start getting frequently compared to item $1$ every time they are played alongside with $1$, and thus they can not stay too long in the set of `good' items $\cC_t$ and eventually $\cC_t=\{1\}$, when the algorithm \algmm\, plays the optimal set $S_t=\{1\}$ only, and thus the regret bound follows. (see Lem. \ref{lem:mm_n1i_postT0} and \ref{lem:mm_reg_postT0})
\end{itemize}

Observation $1$ is straightforward to follow from Alg. \ref{alg:mm}. Observation $2$ follows from Lem. \ref{lem:cdelta}, as for any $t > f(\delta)$, $1\in \cC_t$ always, since $u_{1i} \ge p_{1i} > \frac{1}{2}, ~\forall i \in [n]\setminus\{1\}$.
We next recall the notations before proceeding to the next results: Let $\Delta_{i}  = P(1 \succ i) - \frac{1}{2} = \frac{\theta_1-\theta_i}{2(\theta_1+\theta_i)}$, for all $i \in [n]\setminus\{1\}$.
For any pair $(i,j)$ such that $1 \notin \{i,j\}$, we define $D_{ij} = \frac{4\alpha}{\min\{\Delta_i^2,\Delta_j^2\}}$. 
For any $i \in [n] \setminus \{1\}$, $D_{1i} = \frac{4\alpha}{\Delta_i^2}$.

\begin{defn}[Unsaturated Pairs]
\label{def:unsat_pair}
At any time $t \in [T]$, and any pair of two distinct items $i,j \in [n]$, we term the pair $(i,j)$ to be \emph{unsaturated} at time $t$ if $n_{ij}(t) \le D_{ij}\ln t$. Otherwise, we call the pair \emph{saturated} at $t$. 
\end{defn}

\begin{lem}
\label{lem:topm_comp}
For any set $S \subseteq [n]$ such that $|S| \ge m+1$, and given a \tf\, $\sigma \in \Sigma_{S}^{m}$ (for any $m \in [k-1]$), applying pairwise \rb\, on $S$ according to $\sigma$, updates each element $i \in S$ for atleast $m$ distinct pairs.
\end{lem}

\begin{proof}
For any item $i \in S$, one can make the following two case analyses:

\textbf{Case 1: ($i \in \sigma(1:m)$)}. If the item $i$ occurs in one of the top-$m$ position, it is clearly compared with rest of the $|S|-1 \ge m$ elements of $S$, as it is beaten by the preceding items in $\sigma$ and wins over the rest.

\textbf{Case 2: ($i \notin \sigma(1:m)$)}. In this case $i$ gets updated for $m$ many pairs since it is considered to be beaten by all items in $\sigma(1:m)$ in a pairwise duel.

The claim follows combining Case $1$ and $2$ above.
\end{proof}

\begin{lem}
\label{lem:n_1i}
Assuming $\forall t > f(\delta)$, and $\forall i,j \in [n] ~~p_{ij} \in [l_{ij}(t), u_{ij}(t)] 
$, for some $\delta \in (0,1)$: At any iteration $t>f(\delta)$, if $\exists$ a suboptimal item $i \in [n]\sm\{1\}$, such that $i \in \cC_t$, then the pair $(1,i)$ is unsaturated at $t$.
\end{lem}

\begin{proof}
Firstly note that for any $t > f(\delta)$, $1\in \cC_t$ always, since $u_{1i} \ge p_{1i} > \frac{1}{2}, ~\forall i \in [n]\setminus\{1\}$.

Now suppose $(1,i)$ is indeed saturated at time $t$, i.e. $n_{1i}(t) > D_{1i}\ln t$, then this implies:

\begin{align*}
u_{i1}(t) = \hat{p}_{i1}(t) + c_{i1}(t) \le {p}_{i1}(t) + 2c_{i1}(t) = {p}_{i1}(t) + \Delta_i = \frac{1}{2},
\end{align*}
which implies $i \notin \cC_t$, at $t$. Thus $(1,i)$ must be unsaturated at $t$.
\end{proof}

\begin{lem}
\label{lem:n_ij}
Assuming $\forall t > f(\delta)$, and $\forall i,j \in [n] ~~p_{ij} \in [l_{ij}(t), u_{ij}(t)] 
$, for some $\delta \in (0,1)$: At any iteration $t>f(\delta)$, for any set $S \not\owns 1$ if $\exists$ a suboptimal item $a \in [n]\sm\{1\}$, such that $a = \underset{c \in I\sm S}{\arg \max}\Big[ \min_{i \in S}u_{ci}(t) \Big]$, then $\exists$ atleast one suboptimal item $i \in S$ such that the pair $(i,a)$ is unsaturated at $t$.
\end{lem}

\begin{proof}
We start by noting that for any $i \in \cC_t\sm\{1\}$ if $u_{ji}(t) > u_{1i}(t)$, then $n_{ij} \le D_{ij}\ln t$, i.e. the pair $(i,j)$ must be unsaturated at round $t$. Suppose not and $n_{ij}(t) > D_{ij}\ln t$. Then we have that

\begin{align*}
u_{ij}(t) - l_{ij}(t) = 2c_{ij}(t)  
\le \sqrt{\min\{\Delta_i^2,\Delta_j^2\}}
= \min\{ \Delta_{i}, \Delta_j \}.
\end{align*}

But on the other hand, since $u_{ji}(t) > u_{1i}(t)$, this implies:

\begin{align*}
u_{ij}(t) - l_{ij}(t) = u_{ij}(t) + u_{ji}(t) - 1 > \frac{1}{2} + u_{1i}(t) - 1
> \frac{1}{2} + p_{1i}(t) - 1 = \Delta_i \ge \min\{\Delta_i, \Delta_j\},
\end{align*}
where the first inequality is because $i \in \cC_t$, hence $u_{ij}(t) > \frac{1}{2}$ and $u_{ji}(t) > u_{1i}(t)$. This leads to a contradiction implying that $(i,j)$ has to be unsaturated at $t$.

The proof now follows noting that, by definition of $a$, $\min_{i \in S}u_{ai}(t) > \min_{i \in S}u_{1i}(t) \implies \exists$ atleast one item $i \in S$ such that $u_{ai}(t) > u_{1i}$. But following above chain of argument that leads to a contradiction unless the pair $(i,a)$ is unsaturated at round $t$.
\end{proof}


Combining Lem. \ref{lem:n_1i} and \ref{lem:n_ij} we can conclude that it does not take too long to reach to a time $T_0(\delta) > f(\delta)$, such that $\cC_{T_0(\delta)} = \{1\}$ and thus $\cB_{t} = \{1\}$ for all $t > T_0(\delta)$.

\begin{lem}
\label{lem:t_hat_delta}
Assume $\forall t > f(\delta)$, and $\forall i,j \in [n] ~~p_{ij} \in [l_{ij}(t), u_{ij}(t)]$, for some $\delta \in (0,1)$. Then if we define $T_0(\delta)$ such that:
$
T_0(\delta) = \min \{ t > f(\delta) \mid \cC_t = \{1\} \},
$ 
it can be upper bounded as $T_0(\delta) \le 2f(\delta) + 2D\ln 2D$, where $D := \sum_{i < j}D_{ij}$.
\end{lem}

\begin{proof}
The first observation for this is to note that: For any $t > f(\delta)$, $1 \in \cC_t$ since $u_{1i} \ge p_{1i} > \frac{1}{2}, ~\forall i \in [n]\setminus\{1\}$. So, until $T_0(\delta)$, for all $t \in \{f(\delta),f(\delta)+1, \ldots T_{0}(\delta)-1\}$, $|\cC_t| \ge 2$.

Secondly, for any $t \in \{f(\delta),f(\delta)+1, \ldots T_{0}(\delta)-1\}$, there exists atleast $\min(m,|\cC_t|-1)$ unsaturated pairs in $S_t$ which gets updated.  This holds from the following two case analyses:

\textbf{Case 1: ($1 \in S_t$)}. This is the easy case since for any item $i \in \cC_t\sm \{1\}$, we know that $(1,i)$ is unsaturated from Lem. \ref{lem:n_1i}, and item $1$ has to be updated for atleast $\min(m,|\cC_t|-1)$ many unsaturated pairs as follows from Lem. \ref{lem:topm_comp}.

\textbf{Case 2: ($1 \notin S_t$)}. From Lem. \ref{lem:n_ij} we know that for any item $i \in S_t \cap \cC_t\sm \{1\}$ has to be unsaturated with atleast another item $j \in S_t $. Since \algmm\, makes sure $|S_t| \ge m+1$, again owing to Lem. \ref{lem:topm_comp}, any item $i \in S_t \cap \cC_t \sm \{1\}$ gets compared for atleast $m$ pairs out of which atleast one pair has to be unsaturated which proves the claim.

Moreover, as argued above, at any round $t \in \{f(\delta),f(\delta)+1, \ldots T_{0}(\delta)-1\}$, since $|\cC_t| \ge 2$, any such round $t$ updates atleast $\min(m,|\cC_t|-1) \ge 1$ unsaturated pair.

Thirdly, at any time $t$, if all pairs $(i,j), \, i \neq j, \, i,j \in [n]$ are saturated, then $\cC_t = \{1\}$.

So to bound $T_0(\delta)$, all we need to figure out is the worst possible number of iterations \algmm\, would take to saturate all possible unsaturated pairs, precisely $\sum_{i < j}D_{ij}\ln t$ many pairwise updates. But as we argued before, since any round $t > f(\delta)$ updates atleast one unsaturated pair, we find that
\[
T_0(\delta) = \min\{ t > f(\delta) \mid t > f(\delta) + \sum_{i < j}D_{ij}\ln t\}
\]
Now it is easy to see that the above inequality $(t > f(\delta) + \sum_{i < j}D_{ij}\ln t)$ certainly satisfies for $t = 2f(\delta) + 2D\ln 2D$, where $D := \sum_{i < j}D_{ij}$ as:

\begin{align*}
f(\delta) + D \ln t &= f(\delta) + D\ln (2f(\delta) + 2D \ln 2D)\\
& \le f(\delta) + D \ln (2D\ln 2D) + D\frac{2f(\delta)}{2D\ln 2D}\\
& \le f(\delta) + D\ln (2D)^2 + f(\delta), ~~[\text{ since, }  \ln 2D>1]\\
& = 2f(\delta) + 2D\ln 2D = t.
\end{align*}

Since $T_0(\delta)$ is the minimum time index at which $t > f(\delta) + D\ln t$ is satisfied, clearly $T_0(\delta) \le 2f(\delta) + 2D\ln 2D$.
\end{proof}

Finally we are ready to prove Thm. \ref{thm:whp_reg_mm} based on the the following two claims:

\begin{lem}
\label{lem:mm_n1i_postT0}
Assume $\forall t > f(\delta)$, and $\forall i,j \in [n] ~~p_{ij} \in [l_{ij}(t), u_{ij}(t)]$, for some $\delta \in (0,1)$.
For any time step $t > T_0(\delta)$, $1 \in S_t$ always. Moreover for any $|S_t|>1$, item $1$ gets compared with atleast $m$ suboptimal items $a \in [n]\sm \{1\}$.
\end{lem}

\begin{proof}
For any $t > f(\delta)$, $1 \in \cC_t$ since $u_{1i} \ge p_{1i} > \frac{1}{2}, ~\forall i \in [n]\setminus\{1\}$. Moreover as $T_0(\delta)$ ensures $\cC_{T_0(\delta)} = \{1\}$, at this round, the algorithm set $\cB_{T_0(\delta)} = \{1\}$. For the subsequent rounds $t > T_0(\delta)$, thus the algorithm continues setting $\cB_t = \cB_{t-1} \cap \cC_t = \{1\}$.

Moreover note that for any $t > T_0(\delta)$, unless $|\cC_t| = 1$, the algorithm always plays a set $S_t$ such that $|S_t| =  m+1$, and in which item $1$ always resides. Then by Lem. \ref{lem:topm_comp} we can conclude that item $1$ is compared with atleast $m$ distinct items at any round after pairwise \rb\, update.
\end{proof}

\begin{lem}
\label{lem:mm_reg_postT0}
For any $\delta \in (0,1)$, with probability atleast $(1-\delta)$, the total cumulative regret of \algmm\, is upper bounded as:
\[
R_T^{1} \le \bigg(2\Big[ \frac{2\alpha n^2}{(2\alpha-1)\delta} \Big]^{\frac{1}{2\alpha-1}} + 2D\ln 2D\bigg)\hat \Delta_{\max} + \frac{\ln T}{m+1}\sum_{i = 2}^{n}\hat \Delta_i \big(\1(m = 1)D_{1i} + \1(m > 1)D_{\max} \big)
\]
where recall that $\forall i \in [n] \sm \{a^*\}$, $\hat \Delta_i = (\theta_{1} - \theta_i)$, $\Delta_i = p_{1i} = \frac{\theta_{1}-\theta_i}{2(\theta_{1}+\theta_i)}$, $\hat \Delta_{\max} = \max_{i \in [n] \sm \{1\}}\hat \Delta_i$  $D_{1i} = \frac{4\alpha}{\Delta_i^2}$, $D := \sum_{i < j}D_{ij}$, $D_{\max} = \max_{i \in [n] \sm \{1\}}D_{1i}$.
\end{lem}

\begin{proof}
Given Lem. \ref{lem:mm_n1i_postT0} in place, the crucial observation now is to note that for any $t > T_0(\delta)$, \algmm, always explores as long as there exists any suboptimal item $i \in [n]\sm \{1\}$ such that the pair $(1,i)$ is unsaturated and thus $i \in \cC_t$. 
In other words, our set building rule (\algbld) always picks items from $\cC_t$ first before picking anything from $[n]\sm \cC_t$. However, any suboptimal item $i\in [n]\sm \{1\}$ can belong to $\cC_t$ only if the pair $(1,i)$ is unsaturated, as follows from Lem. \ref{lem:n_1i}. 

Thus for any time $t$, if the pair $(1,i)$ is already saturated (i.e. $n_{1i}(t) > D_{1i}(t)\ln t$), then $i \notin S_t$ unless item $1$ is saturated with every suboptimal item in $[n]\sm\{1\}$. But then $\cC_t = \{1\}$ by Lem. \ref{lem:n_1i} and the algorithm would go on playing $S_t = \{1\}$ until some pair $(1,i)$ gets unsaturated again. This argument holds true even for $t = T$.

Now lets try to analyse what is the maximum number of time an item $i \in [n]\sm\{1\}$ can show up at any round post \sat\ (i.e. for any $t > T_0(\delta)$). But since post \sat, for any $t > T_0(\delta)$, $1 \in S_t$ always, the quantity $n_{1i}(T) - n_{1i}(T_0(\delta))$ is same as above. We hence analyse $n_{1i}(T) - n_{1i}(T_0(\delta))$ with the following two cases:

\textbf{Case-$1$ $(m = 1)$:} This case is easy to analyse since at any round $t$, $|S_t| = 2$ and since $1\in S_t$, so $1$ gets compared with exactly one other suboptimal element $i \in [n]\sm \{1\}$ at any $t$ such that $i \in \cC_t$. So clearly $n_{1i}(T) - n_{1i}(T_0(\delta)) \le D_{1i}\ln T$ as by Lem. \ref{lem:n_1i} after $1$ is compared to $i$ for $D_{1i}\ln T$ times $i \notin \cC_t$ henceforth.

\textbf{Case-$2$ $(m > 1)$:} In this case there are two possible ways $i$ can show up in $S_t$: $(i)$. If its unsaturated with $1$ for which it can show up for at most $D_{1i}\ln T$ times as argued i Case-$1$, and $(ii)$. When $i \notin \cC_t$ but it shows up as a place holder for onle of the $m+1$ slot of $S_t$ as long as some other element $j \in [n]\sm\{1\}$, $i \neq j$ is unsaturated with $1$ and $j \in \cC_t$. But in the worst case once all item $i \in [n]\sm\{1\}$ has appeared in $S_t$ for $D_{\max}\ln T (> D_{1i}\ln T)$ times by Lem. $1$ we have $u_{i1} < \frac{1}{2} \, \forall i \in [n]\sm\{1\}$ and then $\cC_t$ has to be the singleton $\{1\}$ thereafter. So it has to be that $n_{1i}(T) - n_{1i}(T_0(\delta)) \le D_{\max}\ln T$.


Finally note that all our above results holds good under the assumption that $\forall t > f(\delta)$, and $\forall i,j \in [n] ~~p_{ij} \in [l_{ij}(t), u_{ij}(t)]$, for some $\delta \in (0,1)$, which itself holds good with probability atleast $(1-\delta)$. Thus we have the maximum regret incurred by \algmm\, in $T$ rounds is 

\begin{align*}
R_T & \le T_0(\delta)\hat \Delta_{\max} + \frac{1}{m+1}\sum_{i = 2}^{n} \big( n_{1i}(T)-n_{1i}(T_0(\delta)) \big) \big(\1(m = 1)D_{1i} + \1(m > 1)D_{\max} \big)\\
&= \bigg(2\Big[ \frac{2\alpha n^2}{(2\alpha-1)\delta} \Big]^{\frac{1}{2\alpha-1}} + 2D\ln 2D\bigg)\hat \Delta_{\max} + \frac{\ln T}{m+1}\sum_{i = 2}^{n}\hat \Delta_i \big(\1(m = 1)D_{1i} + \1(m > 1)D_{\max} \big), 
\end{align*}

with probability atleast $(1-\delta)$, were first term in the right hand side of the inequality holds since the maximum possible per trial regret that could be incurred by \algmm\, in initial $T_0(\delta)$ rounds is $\hat \Delta_{\max}$. The proof now follows further upper bounding $T_0$ using Lem. \ref{lem:t_hat_delta}. 
\end{proof}

This also concludes the proof of Thm. \ref{thm:whp_reg_mm} using the exact value of $f(\delta)$ as derived in Lem. \ref{lem:cdelta}.

\end{proof}


\subsection{Proof of Theorem \ref{thm:exp_reg_mm}}

\expmm*

\begin{proof}
Recall from the statement of Thm. \ref{thm:whp_reg_mm} that the only term in $R_T^1$ that depends on $\delta$ is $2f(\delta)$, where recall that $T_0(\delta) = 2f(\delta) + 2D\ln 2D$.
Then by integrating $f(\delta)$ for $\delta$ from $0$ to $1$ as follows:

\[
\int_{0}^{1} f(\delta)d\delta = \int_{0}^{1}\Big[ \frac{2\alpha n^2}{(2\alpha-1)\delta} \Big]^{\frac{1}{2\alpha-1}}d\delta = \Big[ \frac{2\alpha n^2}{(2\alpha-1)} \Big]^{\frac{1}{2\alpha-1}}\int_{0}^{1}\Bigg(\frac{1}{\delta}\Bigg)^{\frac{1}{2\alpha-1}}d\delta = \Big[ \frac{2\alpha n^2}{(2\alpha-1)} \Big]^{\frac{1}{2\alpha-1}}\frac{2\alpha-1}{2\alpha-2}
\]

Thus expected regret $\E_{\delta}[R_T]$ can be upper bounded as:
\[
\E_{\delta}[R_T^1] \le \Bigg( 2\Big[ \frac{2\alpha n^2}{(2\alpha-1)} \Big]^{\frac{1}{2\alpha-1}}\frac{2\alpha-1}{\alpha-1} + 2D\ln 2D\Bigg)\hat\Delta_{\max} + \frac{\ln T}{m+1}\sum_{i = 2}^{n}\hat \Delta_i \big(\1(m = 1)D_{1i} + \1(m > 1)D_{\max} \big).
\]
\end{proof}


\section{Supplementary for Section \ref{sec:res_tk}}


\subsection{Proof of Thm. \ref{thm:lb_tktf}}
\label{app:lb_tktf}

\lbtktf*

\begin{proof}
The main idea lies in constructing `hard enough' problem instances for which any \nr\, algorithm has to incur $\Omega\bigg( \frac{n}{k\Delta_{(k)}} \ln T\bigg)$ regret. 

We choose our \emph{true problem instance} with MNL parameters $\btheta^1 = (\theta_1^1,\ldots,\theta_n^1)$, such that: 

\begin{align*}
\textbf{True Instance: } \text{MNL}(n,\btheta^1): & \theta_1^1 = \theta_2^1 = \ldots = \theta_{k-1}^1 = \theta + 2\epsilon; \\
& \theta_n^1 = \theta + \epsilon;\,  \theta_{k+1}^1 = \theta_{k+2}^1 = \ldots \theta_{n-1}^1 = \theta.
\end{align*}

for some $\theta \in \R_+$ and $\epsilon > 0$. Clearly, the \bk\, (recall the definition from Def. \ref{def:mnl_thet}, Sec. \ref{sec:prelims}) of MNL$(n,\btheta^1)$ is $\sS[1] = [k-1]\cup \{n\}$. Now for every $n-k$ suboptimal items $a \notin \sS[1]$, consider the altered problem instance MNL$(n,\btheta^a)$ such that:

\begin{align*}
\textbf{Instance a: } \text{MNL}(n,\btheta^a): \theta_a^a = \theta + 2\epsilon; ~\theta_i^a = \theta_i^1, ~~\forall i \in [n]\sm \{a\}
\end{align*}

And now the \bk\, of MNL$(n,\btheta^a)$ is $\sS[a] = [k-1]\cup\{a\}$.
Same as the case for proof of Thm. \ref{thm:lb_wiwf} or Thm. \ref{thm:lb_witf}, we now again use the results of \cite{Garivier+16} (Lem. \ref{lem:gar16}) for proving the lower bound. Precisely, the main trick lies in  analyzing the KL-divergence terms for the above problem instances. For ease of analysis we first assume analyse the case with just the \wf. Borrowing same notations from Thm. \ref{thm:lb_witf}, and denoting $x = |S\cap\sS[1]| - r$, $r = \1(n \in S)$, $y = k - (x+r)$, for any set $S \in S^a$, we now get that for any $i \in S$:

\begin{align*}
p^1_S(i) = 
\begin{cases} 
\frac{\theta+2\epsilon}{\theta^1_S} = \frac{\theta+2\epsilon}{x(\theta+2\epsilon)+r(\theta+\epsilon)+y\theta} = \frac{\theta+2\epsilon}{k\theta + \epsilon(2x+r)}, \text{ such that } i \in \sS[1]\cap S,\\
\frac{\theta+\epsilon}{\theta^1_S} = \frac{\theta+\epsilon}{x(\theta+2\epsilon)+r(\theta+\epsilon)+y\theta} = \frac{\theta+\epsilon}{k\theta + \epsilon(2x+r)}, \text{ such that } i = n,\\
\frac{\theta}{\theta^1_S} = \frac{\theta}{x(\theta+2\epsilon)+r(\theta+\epsilon)+y\theta} = \frac{\theta+\epsilon}{k\theta + \epsilon(2x+r)}, \text{ otherwise. }
\end{cases}
\end{align*}

On the other hand, for problem \textbf{Instance-a}, we have that: 

\begin{align*}
p^a_S(i) = 
\begin{cases} 
\frac{\theta+2\epsilon}{\theta^a_S} = \frac{\theta+2\epsilon}{(x+1)(\theta+2\epsilon)+r(\theta+\epsilon)+(y-1)\theta} = \frac{\theta+2\epsilon}{k\theta + \epsilon(2(x+1)+r)}, \text{ such that } i \in (S\cap\sS[1]) \cup \{a\},\\
\frac{\theta+\epsilon}{\theta^a_S} = \frac{\theta+\epsilon}{(x+1)(\theta+2\epsilon)+r(\theta+\epsilon)+(y-1)\theta} = \frac{\theta+\epsilon}{k\theta + \epsilon(2(x+1)+r)}, \text{ such that } i = n,\\
\frac{\theta}{\theta^a_S} = \frac{\theta}{(x+1)(\theta+2\epsilon)+r(\theta+\epsilon)+(y-1)\theta} = \frac{\theta+\epsilon}{k\theta + \epsilon(2(x+1)+r)}, \text{ otherwise. }
\end{cases}
\end{align*}

For ease of notation we denote $\theta_S = \theta_S^1$.
Now using the following upper bound on $KL(\p,\q) \le \sum_{x \in \X}\frac{p^2(x)}{q(x)} -1$, $\p$ and $\q$ be two probability mass functions on the discrete random variable $\X$ \cite{klub16}, we get for any $S \in S^a$:

\begin{align}
\label{eq:lb_tkwf_kl}
\nonumber & KL(p^1_S, p^a_S) \le \sum_{i \in S}\bigg(\frac{\theta_i^1}{\theta_S^1}\bigg)^2\bigg( \frac{\theta_S^a}{\theta_i^a}\bigg) - 1\\
\nonumber & = \hspace*{-15pt}\sum_{i \in S \cap \sS[1]}\bigg(\frac{\theta+2\epsilon}{\theta_S}\bigg)^2\bigg( \frac{\theta_S+2\epsilon}{\theta + 2\epsilon}\bigg) \\
\nonumber & + \hspace*{-15pt}\sum_{i \in S\cap(\{a,n\}\cup \sS[1])^c }\hspace*{-6pt}\bigg(\frac{\theta}{\theta_S}\bigg)^2\bigg( \frac{\theta_S+2\epsilon}{\theta}\bigg) + \bigg(\frac{\theta+\epsilon}{\theta_S}\bigg)^2\bigg( \frac{\theta_S+2\epsilon}{\theta+\epsilon}\bigg) + \bigg(\frac{\theta}{\theta_S}\bigg)^2\bigg( \frac{\theta_S+2\epsilon}{\theta+2\epsilon}\bigg) - 1\\
\nonumber & = \frac{\theta+2\epsilon}{\theta_S^2}\bigg[ \theta_S + \frac{\theta^2}{\theta + 2\epsilon} - \theta \bigg] = \frac{2\epsilon}{\theta_S}\bigg[ 1-\frac{\theta}{\theta+2\epsilon}\bigg] - \frac{4\epsilon^2\theta}{\theta_S^2(\theta+2\epsilon)}\\
& \le \frac{(2\epsilon)^2[\theta_S - \theta]}{\theta_S^2(\theta+2\epsilon)} \le \frac{(2\epsilon)^2}{\theta_S(\theta+2\epsilon)} = \frac{4\epsilon^2}{[k\theta + (2x+r)\epsilon](\theta+2\epsilon)} \le \frac{4\epsilon^2}{k\theta(\theta+2\epsilon)}
\end{align}


Now coming back to the \tk\, applying chain rule of KL-divergence (similar to the analysis of Eqn. \eqref{eq:lb_witf_kl}), we can write

\begin{align*}
KL(p^1_S, p^a_S) = KL(p^1_S(\sigma_1) + \cdots + KL(p^1_S(\sigma_k \mid \sigma(1:k-1)), p^a_S(\sigma_k \mid \sigma(1:k-1))).
\end{align*}
for any ranking $\sigma \in \Sigma_{S}^{k}$.
And following the same argument that of \eqref{eq:lb_witf_kl2}, we further get 

\begin{align}
\label{eq:lb_tktf_kl2}
KL(p^1_S, p^a_S) \le \frac{4k\epsilon^2}{k\theta(\theta+2\epsilon)}.
\end{align}

The rest of the proof follows exactly the same argument used in \ref{thm:lb_witf}. We add the steps below for convenience.
As before, considering $Z = \frac{N_{\sS[1]}(T)}{T}$, for large $T$, in this case we get: 

\begin{align}
\label{eq:lb_tktf_rhs}
\lim_{T \to \infty}\frac{kl(\E_{\btheta^1}[Z], \E_{\btheta^a}[Z])}{\ln T} \ge (1-\alpha), 
\end{align}

which follows from an exact similar analysis shown in the proof of Thm. \ref{thm:lb_wiwf} along with the facts that: 

\begin{align*}
& \E_{\btheta^1}[N_{\sS[1]}(T)] = 1 - o(T^\alpha) ~~\text{ since } \cA \text{ is assumed to be \nr\, and}\\
& \E_{\btheta^a}[N_{\sS[1]}(T)] = o\big({T^\alpha}\big). 
\end{align*}

Then using the results of Eqn. \eqref{eq:lb_tktf_kl2} and \eqref{eq:lb_tktf_rhs} in Lem. \ref{lem:gar16}, we further get:

\begin{align}
\label{eq:win_lb2tk}
\nonumber & \lim_{T \to \infty}\frac{1}{\ln T}\sum_{\{S \in S^a\}}\E_{\btheta^1}[N_S(T)]KL(p^1_S, p^a_S) \ge (1-\alpha)\\
\nonumber & \implies \lim_{T \to \infty}\frac{1}{\ln T}\sum_{\{S \in S^a\}}\E_{\btheta^1}[N_S(T)]\frac{4k\epsilon^2}{k\theta(\theta+2\epsilon)} \ge (1-\alpha)\\
& \implies \lim_{T \to \infty}\frac{1}{\ln T}\sum_{\{S \in S^a\}}\E_{\btheta^1}[N_S(T)]\frac{\epsilon}{k} \ge (1-\alpha)\frac{\theta(\theta + 2\epsilon)}{4k\epsilon}
\end{align}

Now applying \eqref{eq:win_lb2tk} for each $n-k$ modified bandit \textbf{Instance-$\btheta^a$} (i.e. for each $a \in [n]\sm\sS[1]$), we get:

\begin{align}
\label{eq:win_lb2.5tk}
\lim_{T \to \infty}\frac{1}{\ln T}\sum_{a = k+1}^{n}\sum_{\{S \in S^a\}}\E_{\btheta^1}[N_S(T)]\frac{\epsilon}{k} \ge (1-\alpha)\theta(\theta + 2\epsilon)\frac{(n-k)}{4k\epsilon}
\end{align}

Further recall from Eqn. \eqref{eq:reg_k} the expected regret of $\cA$ on problem instance MNL($n,\btheta^1$) is given by:
$ 
\nonumber \E_{\btheta^1}[R_T^k(\cA)] = \sum_{t = 1}^{T} r_t^k = \sum_{t=1}^T\left(\frac{\sum_{i \in [k] }\theta_i^1 - \sum_{i \in S_t }\theta_i^1}{k}\right) 
$ 
which can be rewritten as: 

\begin{align}
\label{eq:lb_tktf_reg}
\nonumber \E_{\btheta^1}[R_T^k(\cA)] & = \E_{\btheta^1}\bigg[\sum_{t=1}^T\left(\frac{\sum_{i \in [k] }\theta_i^1 - \sum_{i \in S_t }\theta_i^1}{k}\right)\bigg]\\
\nonumber & \ge \E_{\btheta^1}\bigg[ \sum_{t=1}^T\left(\frac{\sum_{i \in [k] }\theta_k^1 - \sum_{i \in S_t }\theta_i}{k}\right)\bigg]\\
\nonumber & = \E_{\btheta^1}\bigg[ \sum_{t=1}^T\left(\sum_{i \in [S_t]}\frac{ \theta_k^1 - \theta_i}{k}\right)\bigg]\\
\nonumber & = \E_{\btheta^1}\Big[\sum_{t=1}^T\sum_{S \in A}\1(S_t=S)\sum_{a = k+1}^n\1(a \in S)\frac{(\theta_k^1 - \theta_a^1)}{|S_t|}\Big]\\
\nonumber & = \E_{\btheta^1}\bigg[ \sum_{a=k+1}^{n}\sum_{t = 1}^{T}\sum_{S \in A}\1(S_t = S)\1(a \in S)\frac{((\theta+\epsilon) - \theta)}{k} \bigg] ~~~(\text{since } \theta_k^1 = \theta_n^1 = \theta+\epsilon)\\
\nonumber & = \sum_{a=k+1}^{n}\sum_{t = 1}^{T}\E_{\btheta^1}\bigg[ \sum_{S \in A}\1(S_t = S)\1(a \in S)\frac{\epsilon}{k} \bigg]\\
\nonumber & = \sum_{a=k+1}^{n}\sum_{S \in A}\E_{\btheta^1}\bigg[ \sum_{t = 1}^{T}\1(S_t = S)\1(a \in S)\frac{\epsilon}{k} \bigg]\\
\nonumber & = \sum_{a=k+1}^{n}\sum_{S \in A}\bigg[ \E_{\btheta^1}[N_S(T)]\1(a \in S)\frac{\epsilon}{k} \bigg]\\
& = \sum_{a=k+1}^{n}\sum_{\{S \in A \mid a \in S\}} \E_{\btheta^1}[N_S(T)]\frac{\epsilon}{k} 
\end{align}

Using above combined with \eqref{eq:win_lb2.5tk} we get:
\begin{align*}
\lim_{T \to \infty}\frac{1}{\ln T}\E_{\btheta^1}[R_T^k(\cA)] & \ge \lim_{T \to \infty}\frac{1}{\ln T}\sum_{a=k+1}^{n}\sum_{\{S \in A \mid a \in S\}} \E_{\btheta^1}[N_S(T)]\frac{\epsilon}{k} \\
& \ge \lim_{T \to \infty}\frac{1}{\ln T}\sum_{a = k+1}^{n}\sum_{\{S \in S^a\}}\E_{\btheta^1}[N_S(T)]\frac{\epsilon}{k} \, \ge (1-\alpha)\theta(\theta + 2\epsilon)\frac{(n-k)}{4k\epsilon}.
\end{align*}

Now since $\alpha$ is a fixed constant in $(0,1]$, we thus prove the existence of a \mnl\, problem instance $\big($precisely MNL($n,\btheta^1$)$\big)$, such that for large $T$, $\E_{\btheta^1}[R_T^{1}] = \Omega\Bigg( \frac{\theta_1\theta_{k+1}}{\Delta_{(k)}}\frac{(n-k)}{k}\ln T  \Bigg)$ (noting that for instance MNL$(n,\btheta^1), \Delta_{(k)} = \epsilon$), which concludes the proof 


\end{proof}


\subsection{Algorithm pseudocode: \objk}
\label{app:alg_tktf}
\begin{center}
\begin{algorithm}[H]
   \caption{\textbf{\algkmm}}
   \label{alg:kmm}
\begin{algorithmic}[1]
   \STATE {\bfseries init:} $\alpha > 0.5$, $\W \leftarrow [0]_{n \times n}$, $\cB_0 \leftarrow [\emptyset]_k$ 
   \FOR{$t = 1,2,3, \ldots, T$} 
	\STATE Set $I \leftarrow [n]$, $\bN = \W + \W^\top$, and $\hat{\P} = \frac{\W}{\bN}$. $N = [n_{ij}]_{n \times n}$ and $\hat P = [\hp_{ij}]_{n \times n}$.
	\STATE Define $u_{ij} = \hp_{ij} + \sqrt{\frac{\alpha\ln t}{n_{ij}}}, \, \forall i,j \in [n], i\neq j$, $u_{ii} = \frac{1}{2}, \, \forall i \in [n]$. $\textbf{U} = [u_{ij}]_{n \times n}$ 
	\FOR{$h = 1,2,\ldots, k-1$}
	\STATE $\cC_t^h \leftarrow \{i \in I ~|~ u_{ij} > \frac{1}{2}, \, \forall j \in I\setminus\{i\}\}$;  $\cB_t(h) \leftarrow \cC_t^h \cap \cB_{t-1}(h)$ 
	\STATE \textbf{if} $\cB_{t}(h) \ne \emptyset$, \textbf{then} set $I \leftarrow I \setminus \cB_t(h)$ and $S_t \leftarrow S_t \cup \cB_t(h)$
	\IF{$\cB_t(h) = \emptyset$}
	\STATE \textbf{if} $\cC_t^h = \emptyset$ \textbf{then} $\cC_t^h \leftarrow I$; \textbf{elseif} $|\cC_t^h| > 0$ set $\cB_t(h) \leftarrow$ \algbld$(\textbf{U},S_t,I, 1)$ 
	\STATE  $S_t \leftarrow S_t ~\cup \,$ \algbld$(\textbf{U},S_t,I, k-|S_t|)$. Exit and Goto Line $15$)
	\ENDIF
	\ENDFOR
	
    \STATE $\cC_t^k \leftarrow \{i \in I ~|~ u_{ij} > \frac{1}{2}, \, \forall j \in I\setminus\{i\}\}$;  $\cB_t(k) \leftarrow \cC_t \cap \cB_{t-1}(k)$ 
    \STATE \textbf{if} $|\cC_t^k| = 1$, \textbf{then} $\cB_t(k) \leftarrow \cC_t^k$, and set $S_t \leftarrow S_t \cup \cC_t^k$; \textbf{else} $S_t \leftarrow S_t ~\cup \,$ \algbld$(\textbf{U},S_t, I , 1)$
	\STATE Play $S_t$, and receive: $\bsigma_t \in \bSigma_{S_t}^k$ 
   {\STATE $W(\sigma_t(k'),i) \leftarrow W(\sigma_t(k'),i) + 1 ~~ \forall i \in S_t \sm \sigma_t(1:k')$} for all $k' = 1,2, \ldots, k)$
   \ENDFOR
\end{algorithmic}
\end{algorithm}
\end{center}
\vspace*{-15pt}

\subsection{Proof of Thm. \ref{thm:whp_reg_kmm}}
\label{app:whp_reg_kmm}

\whpkmm*

\begin{proof}
For ease of analysis we assume $\theta_1 \ge \theta_2 \ge \ldots \theta_k > \theta_{k+1} \ge \ldots \ge \theta_n$ and hence $\sS = [k]$.

We use the same notations as introduced in the proof of Thm. \ref{thm:whp_reg_mm}. Note that Lem. \ref{lem:cdelta} holds in this case as well. So that concludes the \rnd\ phase. 

\noindent \textbf{Analysis of the \prog\ phase:}
We next proceed to analyse the \prog \ phase from round $f(\delta)+1$ to $T_0(\delta)$, where $T_0(\delta)$ is defined to be such that 
\[
T_0(\delta) = \arg\min_{t > f(\delta)}\cB_{t} = [k]
\]
The goal of this phase is to show that the length of interval $[f(\delta)+1,T_0(\delta)]$ is \emph{`small'}, {precisely} $T_0(\delta) \le 2f(\delta) + 2\bar D^{(k)}\ln \big( 2\bar D^{(k)} \big)$, where $\bar D^{(k)} := \sum_{r = 1}^{k}D^{(r)}$, and $D^{(r)} := \sum_{i\in [n]}\sum_{j \in W(r)}D_{ij}^r$ (see Lem. \ref{lem:t_hat_deltak}). Note here $\bar D^{(k)}$ is a problem dependent constant, independent of $T$.

\noindent \textbf{Notations:} We first define few notations for ease of analysis: 
Let us first define the set of items $W(i) = \{\ j \in [k] \mid \theta_j > \theta_i \}$ be the set of items in the \bk\ strictly better than $i$, and   
$Z(i) = \{ j \in [k] \mid \theta_j < \theta_i\}$ be the set of items in $[k]$ worse than item $i$, for some $i \in [k]$. 

For any $g \in [k]$, 

$D^{g}_{ij} =
\begin{cases} \frac{4\alpha}{\min((p_{gi}-\frac{1}{2})^2,(p_{gj}-\frac{1}{2})^2)} \text{ if } i,j \in Z(g)\\
\frac{4\alpha}{(p_{gi} - p_{ji})^2} \text{ if } i \in W(g) \text{ and } j \in Z(g)\\
D^{g}_{gj} \text{ if } \theta_i  = \theta_g
\end{cases}$, $D_{ij}^g = D_{ji}^g$ 

for any pair $(i,j) \in [n]\times[n]$, $i \neq g$, $j \neq g$.
and $D^{g}_{gi} = \frac{4\alpha}{(p_{gi}-\frac{1}{2})^2}$, where $p_{ij} = \frac{1}{2} + \frac{\theta_i - \theta_j}{2(\theta_i + \theta_j)} $ for all $i,j \in [n]$.

Towards analysing the \prog\ phase we first make the following key observations:

\begin{itemize}
\item \textbf{Observation $1$:} At any round $t \in [T]$, $\cB_t(i)$ is either singleton or an empty set, for all $i \in [k]$, which follows by the construction of $\cB_t(i)$.


\item \textbf{Observation $2$:} 
For any item $i \in [k]$ in the \bk, at any round $t > f(\delta)$ if $j \in \cB_t(|W(j)|+1:k-|Z(j)|)$ for all $j \in W(i)$, and $i \in \cB_t$ such that $\cB_t(x) = i$ for some $x \in \{|W(i)|+1 \ldots k-|Z(i)\}$ and for any $x' \in \{|W(i)|+1, \ldots, (x-1)\}$, $\theta_{\cB_t(x')} = \theta_i$, then $\cB_{t'}(x) = \{i\}\, \forall t' > t$--- in other words $i$ will continue to reside in slot $\cB_{t'}(x)$ for any $t' > t$.

We next define another notation $T_0^i(\delta)$ for any $i \in [k]$ such that
$
T_0^i(\delta) = \arg\min_{t > f(\delta)} \{ \forall j \in W(i),\, j \in \cB_t(|W(j)|+1:k-|Z(j)|), i = \cB_t(x), \text{and }  \forall x' \in \{|W(i)|+1, \ldots, (x-1)\}, \theta_{\cB_t(x')} = \theta_i \}
$. Clearly $\max_{i \in [k]}T_0^i(\delta) = T_0(\delta)$ as defined above.

\item \textbf{Observation $3$:} $T_0(\delta)$ is not far from $f(\delta)$---we prove this in a stepwise manner, to explain it in an intuitive level assume $\theta_1 > \ldots > \theta_k$, Then we first show that $T_{0}^1(\delta)$ is bounded. Once item $1$ is secured in its slot $\cB_{T_{0}^1(\delta)}(1)$, we proceed to bound $T_{0}^2(\delta)$, and so on till $T_{0}^k(\delta) = T_0(\delta)$ (see Lem. \ref{lem:t_hat_deltak} for the formal details which holds due to Lem. \ref{lem:n_1ik} and \ref{lem:n_ijk}). 

\end{itemize}

We find it convenient to define one more definition before proving Lem. \ref{lem:t_hat_deltak}:

\begin{defn}[$g$-Unsaturated Pairs]
\label{def:gunsat_pair}
At any time $t \in [T]$, for any item $g \in [k]$ and any pair of two distinct pair of items $i,j \in [n]$, we call the pair $(i,j)$ to be \emph{$g$-unsaturated} at time $t$ if $n_{ij}(t) \le D_{ij}^g\ln t$. Otherwise, we call the pair \emph{$g$-saturated} at $t$. 
\end{defn}

\begin{lem}
\label{lem:n_1ik}
Assuming $\forall t > f(\delta)$, and $\forall i,j \in [n] ~~p_{ij} \in [l_{ij}(t), u_{ij}(t)]$, for some $\delta \in (0,1)$: At any iteration $t>f(\delta)$, for any $g \in [k]$, if $\exists$ an item $i \in Z(g)$, i.e. $\theta_g > \theta_i$ and both $i,j \in \cC_t^h$ for some $h \in [k]$, then the pair $(g,i)$ is $g$-unsaturated at $t$.
\end{lem}

\begin{proof}
By assumption $\forall i,j \in [n] ~~p_{ij} \le u_{ij}(t)$. Now if $\exists h \in [k]$ such that a pair $(g,i)$ such that both $i,g \in \cC_t^h$, then it has to be the case that $u_{gi}>\frac{1}{2}$ and $u_{ig}>\frac{1}{2}$.

But then suppose $(g,i)$ was indeed $g$-saturated at time $t$, i.e. $n_{1i}(t) > D_{gi}^g\ln t$, this implies:

\begin{align*}
u_{ig}(t) = \hat{p}_{ig}(t) + c_{ig}(t) \le {p}_{ig}(t) + 2c_{ig}(t) = {p}_{ig}(t) + \Delta_i^g = \frac{1}{2},
\end{align*}
which implies there cannot exist $i \notin \cC_t^h$ if $g \in \cC_t^h$ for any $h$, which leads to a contradiction. Hence the pair $(g,i)$ must be $g$-unsaturated at $t$.
\end{proof}

\begin{lem}
\label{lem:n_ijk}
Assuming $\forall t > f(\delta)$, and $\forall i,j \in [n] ~~p_{ij} \in [l_{ij}(t), u_{ij}(t)] 
$, for some $\delta \in (0,1)$. Consider any $g \in [k]$. At any iteration $t>f(\delta)$, for any set $S_t \not\owns g$, $0 \le |S_t| < k$ if $\exists$ an item $a \in Z(g)$, i.e. $\theta_g > \theta_a$, such that $a = \underset{c \in I \sm S_t}{\arg \max}\Big[ \min_{i \in S_t}u_{ci}(t) \Big]$, then $\exists$ atleast one  item $b \in S$ such that the pair $(a,b)$ is unsaturated at $t$.
\end{lem}

\begin{proof}
Firstly the important observation to make is at any round $t$, and in any of its sub-phase $h \in [k]$, our set building rule ensures that $u_{ji} > \frac{1}{2}$ for $j \in S_t$ and $i \notin S_t$. 

Moreover since $a = \underset{c \in I \sm S_t}{\arg \max}\Big[ \min_{i \in S_t}u_{ci}(t) \Big]$ and $a \notin S_t$, there must exist an item $b$ in $S$ such that $u_{ab}(t) > u_{gb}(t)$ as otherwise $g$ would have been picked instead of $a$. But following the argument above we also know that $u_{ba}(t) > \frac{1}{2}$. Now $b$ can fall into the following three categories:

\textbf{Case-1} $\big[b \in Z(g)\big]$: We first note that:
\[
u_{ba}(t) - l_{ba}(t) = u_{ba}(t) + u_{ab}(t) - 1 > \frac{1}{2} + u_{gb}(t) - 1 > \frac{1}{2} + p_{gb}(t) - 1 = p_{gb} - \frac{1}{2},
\]
but on the other hand if the pair $(a,b)$ is indeed $g$-saturated at $t$, i.e. $n_{ba}(t) > D_{ba}^g\ln t$. Then we have that

\begin{align*}
u_{ba}(t) - l_{ba}(t) = 2c_{ba}(t)  
\le \sqrt{\min\Bigg(\Big(p_{gb} - \frac{1}{2}\Big)^2,\Big(p_{ga} - \frac{1}{2}\Big)^2\Bigg)}
\le \Big(p_{gb} - \frac{1}{2}\Big).
\end{align*}

\textbf{Case-2} $\big[b \in W(g)\big]$: 

In this case suppose if the pair $(a,b)$ is indeed $g$-saturated at $t$, i.e. $n_{ba}(t) > D_{ba}^g\ln t$ we have

\begin{align*}
 2c_{ab}(t)  
\le \sqrt{\Big( p_{gb} - p_{ab} \Big)^2}
\le \Big(p_{gb} - p_{ab}\Big).
\end{align*}

It is important to note that the right hand side of the above inequality is positive since for this case $\theta_b > \theta_g > \theta_a$. But this implies 
\[
u_{ab}(t) \le  p_{ab}(t) + 2c_{ab}(t) = p_{gb} + 2c_{ab}(t) - (p_{gb} - p_{ab}) \le p_{gb} < u_{gb}(t)
\]
which leads to a contradiction again.

\textbf{Case-3} $\big[b: \theta_b = \theta_g\big]$: The analysis in this case  goes similar to \textbf{Case-2} above which finally leads to the contradiction that $u_{ab} < p_{gb} = \frac{1}{2} < u_{gb}(t)$.

Hence combining the above three cases, it follows that the unless the pair $(b,a)$ is $g$-unsaturated at round $t$, $a \in Z(g)$ can not show up prior to $g$.
\end{proof}


\noindent \textbf{Assumption:} Recall we assumed $\theta_1 \ge \theta_2 \ge \ldots \ge \theta_k$. For ease of explanation (without loss of generality by relabelling the items) we also assume that at any time $t$, for any pair of items $(i,j)$ such that $i,j \in [k]$, $\theta_i = \theta_j$, $i < j$, and if it happens to be the case that both $i,j \in \cB_t$, with $\cB_t(x) = i$, $\cB_t(y) = j$ then $x < y$.

\begin{lem}
\label{lem:t_hat_deltak}
Assume $\forall t > f(\delta)$, and $\forall i,j \in [n] ~~p_{ij} \in [l_{ij}(t), u_{ij}(t)]$, for some $\delta \in (0,1)$. Then if we define $T_0(\delta)$ such that:
$
T_0(\delta) = \min \{ t > f(\delta) \mid \cB_t = [k] \},
$ 
it can be upper bounded as $T_0(\delta) \le 2f(\delta) + 2\bar D^{(k)}\ln \big( 2\bar D^{(k)} \big)$, where $\bar D^{(k)} := \sum_{r = 1}^{k}D^{(r)}$, and $D^{(r)} := \sum_{i\in [n]}\sum_{j \in W(r)}D_{ij}^r$ (recall the rest of the notations as defined above).
\end{lem}

\begin{proof}
Combining Lem. \ref{lem:n_1ik} and \ref{lem:n_ijk} we first aim to bound the term $T_{0}^1(\delta)$ such that the first time after $T_0(\delta)$, when $\cB_t(1) = 1$ and post which it follows that $\cB_{t} = \{1\}$ for all $t > T_0(\delta)$.

\noindent \textbf{Bounding} $T_{0}^1(\delta):$ Note that for any $t > f(\delta)$, $Z(1) = n-1$ always. So the only way $1$  can miss the $\cB_t(1)$ slot is if $\exists i \neq 1$ and $\theta_1 > \theta_i$ (due to the relabelling \textbf{Assumption} above) which occupies $\cB_t(1)$. All we need to figure out is the worst possible number of rounds \algkmm\, would take to $1$-saturate all the pairs, precisely $\sum_{i < j}D_{ij}^1\ln t$ many pairwise updates should be done within $t$ rounds. Now using a similar chain of argument given in Lem. \ref{lem:t_hat_delta} (along with Lem. \ref{lem:n_1ik} and \ref{lem:n_ijk}), since any round $t > f(\delta)$ updates atleast one $1$-unsaturated pair, we find that
\[
T_0^1(\delta) = \min\{ t > f(\delta) \mid t > f(\delta) + D^{(1)}\ln t\},
\]
where $D^{(1)} := \sum_{i\in [n]}\sum_{j \in W(1)}D_{ij}^1$


\noindent \textbf{Bounding} $T_{0}^2(\delta)$: Note that once $1 \in \cB_t(1)$, for any $t > T_0^1(\delta)$, $1 \in S_t$. And also either $\theta_2 = \theta_1$ or $\theta_2 < \theta_1$. But in either case $2 \in \cC_t^2$, as $u_{2j} > \frac{1}{2}$ for all $j \in [n]\sm\{1\}$. Then the only way the one can stop $2$ occupying the slot $\cB_t(2)$ is if there exists some other item $i \neq 2$ and $\theta_2 > \theta_i$ (due to the relabelling \textbf{Assumption} above) which occupies $\cB_t(2)$. But then the algorithm picks $\cB_t(2)$ in $S_t$, i.e. $i \in S_t$ alongside $1$ and it get compared with $1$ at each round it is picked. Moreover the last element of $S_t$ is always picked by the $\algbld$ subroutine, so following the three case analyses of Lem. \ref{lem:n_ijk}, the maximum number of rounds till which $2$ can miss the slot $\cB_t(2)$ is 
\[
T_0^2(\delta) = \min\{ t \mid t > T_0^1(\delta) + D^{(2)}\ln t\},
\]
where $D^{(2)} := \sum_{i\in [n]}\sum_{j \in W(2)}D_{ij}^2$.

Following the same argument we can state a general result that for any $r \in [k]\sm\{1\}$:

\noindent \textbf{Bounding} $T_{0}^r(\delta)$ where $D^{(r)} := \sum_{i\in [n]}\sum_{j \in W(r)}D_{ij}^r$: 
\[
T_0^r(\delta) = \min\{ t \mid t > T_0^{r-1}(\delta) + D^{(r)}\ln t\}
\]

Then combining above for $r = 1,\ldots k$ we get $T_0(\delta) = T_{0}^k(\delta)$ should be such that

\[
T_0^k(\delta) = \min\{ t \mid t > f(\delta) + \sum_{r = 1}^{k}D^{(r)}\ln t\}
\]

And now following the exact same analysis of Lem. \ref{lem:t_hat_delta}, it is easy to see that the above inequality satisfies for $t = 2f(\delta) + 2\sum_{r = 1}^{k}D^{(r)}\ln \Big(\sum_{r = 1}^{k}D^{(r)}\Big)$. 
So that bounds $T_0(\delta) \le 2f(\delta) + 2\sum_{r = 1}^{k}D^{(r)}\ln \Big(\sum_{r = 1}^{k}D^{(r)}\Big)$. 

\end{proof}

\noindent \textit{Analysis for \sat\, phase ($t > T_0(\delta)$):}  

\begin{itemize}
\item \textbf{Observation $4$:} After $T_0(\delta)$, now  $\cB_t = \sS = [k]$, with $\cB_t(k) = \theta_k$, and thereafter i.e. $\cB_t = [k]$ for all $t > T_0(\delta)$. 
Note that in this phase the algorithm always plays either $S_t = [k]$ or it plays $S_t = [k-1] \cup \{b\},\, b \notin [k]$. So for any $t > T_0(\delta)$, $S_t \cap \cB_t = [k-1]$.
Then at any time $t$ is a suboptimal item $b \in [n]\sm[k]$ comes in $S_t$ then it gets compared to all items in $[k-1]$ (owing to \rb). But can not happen for too long and after a time $\cC_t\cap\{b\} = \emptyset$, when the algorithm \algkmm\, will not play $b$ any more. This holds true for any $b \in [n]\sm[k]$, for which the algorithm will left with no other choice for $S_t$ other than $S_t = [k]$ when it incurs no regret. See Lem. \ref{lem:kmm_reg_postT0} and \ref{lem:kmm_reg_full} for the formal claims.

\end{itemize}

\begin{lem}
\label{lem:kmm_reg_postT0}
Assume $\cB_t = \sS = [k]$ for all $t > T_0(\delta)$. Then the total cumulative regret of \algkmm\, post $T_0(\delta)$ is upper bounded by:
\[
R_T^{k}(T_0(\delta):T) := \sum_{t = T_0(\delta)}^{T}r_t^k  \le \frac{4\alpha\ln T}{k}\bigg(\sum_{b = k+1}^{n}\frac{(\theta_k-\theta_b)}{{\hat D}^2} \bigg), 
\]
where ${\hat D} = \min_{g \in [k-1]}(p_{kg}-p_{bg})$.
\end{lem}

\begin{proof}
Note that when $\cB_t = \sS = [k]$, at any such round \algkmm\ plays the set $S_t$ such that the first $(k-1)$ items of $\cB_t$ are always included in $S_t$, i.e. $|S_t \cap \cB_t(1:k-1)| = k-1$. The $k^{th}$ element of $\cB_t$ only gets replaced by a suboptimal element  $b \in [n]\sm[k]$ only if $\exists g \in \cB_t(1:k-1)$ such that the pair $(g,b)$ is \emph{unsaturated}, in a sense that $u_{bg} > u_{kg}$, and hence $b$ got picked by the algorithm instead of $k$ (in Line $14$).

But is that possible for long? Precisely, we now show that any such suboptimal item $b \in [n]\sm[k]$ can not get selected by the algorithm for more than $\frac{4\alpha\ln T}{{\hat D}^2}$ times. 
This is since for any $S_t = [k-1] \cup \{b\}$ (recall that $\cB_{t}(1:k-1) = [k-1]$), once played for $\frac{4\alpha\ln T}{{\hat D}^2}$ times (say this happens at time $t = \tau$), we know that $\forall g \in [k-1]$ such that number of times the pair $(g,b)$ gets updated is exactly $\frac{4\alpha\ln T}{{\hat D}^2}$ too due to \rb \ on \tk. But this implies for any $g \in [k-1]$,

\begin{align}
\label{eq:bg}
u_{bg}(\tau) \le p_{bg} + 2c_{bg}(\tau) = p_{kg} + 2c_{bg}(\tau) - (p_{kg} - p_{bg}) \le p_{kg} \le u_{kg}(\tau)
\end{align}
where the first and last inequality follows by definition of $u_{bg}$ and Lem. \ref{lem:cdelta}, the second last inequality follows due to the fact that since  $n_{gb}(\tau) \ge \frac{4\alpha\ln T}{{\hat D}^2}$
\[
2c_{gb}(\tau) \le 2\sqrt{\frac{\alpha\ln T}{\frac{4\alpha\ln T}{{\hat D}^2}}} = {\hat D} \le  (p_{kg} - p_{bg}),
\]
since by definition ${\hat D} = \min_{g \in [k-1]}(p_{kg}-p_{bg})$.
Then Eqn. \ref{eq:bg} leads to a contradiction showing $u_{bg}(\tau) < u_{kg}(\tau) \implies$ $b$ can not replace $k$ at any round $t > \tau$.

The rest of the analysis simply follows from the fact that since any $b \in [n]\sm[k]$  can appear for only $\bigg(\frac{4\alpha\ln T}{{\hat D}^2} \bigg)$ times and it replaces the item $\cB_t(k) = k$, hence the cost incurred for $b$ is $\frac{(\theta_k-\theta_b)}{k}$ (by Eqn. \ref{eq:reg_k}). Thus the total regret incurred in saturation phase is $\frac{4\alpha\ln T}{k}\bigg(\sum_{b = k+1}^{n}\frac{(\theta_k-\theta_b)}{{\hat D}^2} \bigg)$.
\end{proof}

\begin{lem}
\label{lem:kmm_reg_full}
For any $\delta \in (0,1)$, with probability at least $(1-\delta)$, the total cumulative regret of \algkmm\, can be upper bounded by:
\[
R_T^{k} \le \big( 2f(\delta) + 2\bar D^{(k)}\ln \big( 2\bar D^{(k)} \big) \Delta'_{\max} + \frac{4\alpha\ln T}{k}\bigg(\sum_{b = k+1}^{n}\frac{(\theta_k-\theta_b)}{{\hat D}^2} \bigg)
\]
where $\Delta'_{\max} = \frac{(\sum_{i =1}^{k}\theta_i- \sum_{i=n-k+1}^{n}\theta_i)}{k}$, and $f(\delta)$, $\bar D^{(k)}$ and $\hat D$ is as defined in Lem. \ref{lem:cdelta}, \ref{lem:t_hat_deltak}, \ref{lem:kmm_reg_postT0}.
\end{lem}

\begin{proof}
This can be proved just by combining the claims of Lem. \ref{lem:t_hat_deltak} and \ref{lem:kmm_reg_postT0}. Note from Lem. \ref{lem:t_hat_deltak} that till the \prog\ phase $T_0(\delta)$, the algorithm can play any arbitrary sets $S_t$ for which the maximum regret incurred can be $\Delta'_{\max} = \frac{(\sum_{i =1}^{k}\theta_i- \sum_{i=n-k+1}^{n}\theta_i)}{k}$. Thereafter the algorithm enters into \sat\ phase at which the maximum regret in can incur is $\frac{4\alpha\ln T}{k}\bigg(\sum_{b = k+1}^{n}\frac{(\theta_k-\theta_b)}{{\hat D}^2} \bigg)$ as follows from Lem. \ref{lem:kmm_reg_full}, which concludes the proof.
\end{proof}

The entire analysis above thus concludes the proof of Thm. \ref{thm:whp_reg_kmm}.
\end{proof}

\section*{Proof of Thm. \ref{thm:exp_reg_kmm}}

\expkmm*

\begin{proof}
The proof essentially follows same as the proof of Thm. \ref{thm:exp_reg_mm} by integrating the $\delta$ dependent term $\Big[ \frac{2\alpha n^2}{(2\alpha-1)\delta} \Big]^{\frac{1}{2\alpha-1}}$ in $R_T^k$ (see Thm. \ref{thm:whp_reg_kmm}) from $\delta = 0$ to $\infty$. 
\end{proof}


\section{Experiment Details}
\label{app:expts}

\noindent
We report numerical results of the proposed algorithms run on the following \mnl\, models:

\textbf{\mnl\, Environments.}
1. {\it g1}, 2. {\it g4}, 3. {\it arith}, 4. {\it geo} all with $n = 16$ and two larger models 5. {\it arith-big}, and 6. {\it geo-big} each with $n=50$ items. 
Their individual score parameters are as follows: 
\textbf{1. g1:} $\theta_1 = 0.8$, $\theta_i = 0.2, \, \forall i \in [16]\sm\{1\}$
\textbf{2. g4:} $\theta_1 = 1$, $\theta_i = 0.7, \, \forall i \in \{2,\ldots 6\}$, $\theta_i = 0.5, \, \forall i \in \{7,\ldots 11\}$, and $\theta_i = 0.01$ otherwise.
\textbf{3. arith:} $\theta_1 = 1$ and $ \theta_{i} - \theta_{i+1} = 0.06, \, \forall i \in [15]$.
\textbf{4. geo:} $\theta_1=1$, and $\frac{\theta_{i+1}}{\theta_{i}} = 0.8, ~\forall i \in [15]$. 
\textbf{5. har:} $\theta_1 = 1$ and $\theta_i = 1 - \frac{1}{i}, ~\forall i \in \{2,3,\ldots,16\}$.
\textbf{6. arithb:} $\theta_1 = 1$ and $ \theta_{i} - \theta_{i+1} = 0.02, \, \forall i \in [49]$.
\textbf{7. geob:} $\theta_1=1$, and $\frac{\theta_{i+1}}{\theta_{i}} = 0.9, ~\forall i \in [49]$.

\end{document}